\documentclass{article}

\usepackage{graphicx} 
\usepackage{subfigure} 

\usepackage{natbib}
\usepackage{multibib} %
\newcites{sup}{Supplementary References}

\usepackage{algorithm}
\usepackage{algorithmic}

\usepackage{hyperref}

\usepackage[accepted]{icml2013}

\usepackage{amssymb,amsmath,amsthm,color}
\usepackage{algorithmic,algorithm}
\usepackage{etex}
\usepackage{etoolbox}
\usepackage{ifthen}

\usepackage{noitemsep}
\noitemsep

\icmltitlerunning{Optimization with First-Order Surrogate Functions}

\input{proofsatend.tex}
\def\x{{\mathbf x}}
\def\z{{\mathbf z}}
\def\1{{\mathbf 1}}

\def\HH{{\mathbf H}}

\def\hattn{{{\hat t}_n}}

\def\barz{{ \bf \bar z}}
\def\barg{{ \bar g}}

\def\hati{{\hat{\imath}}}

\def\w{{\mathbf w}}

\def\E{{\mathbb E}}

\def\PPP{{\mathbb P}}
\def\S{{\mathcal S}}
\def\H{{\mathcal H}}

\def\Real{{\mathbb R}}

\def\argmin{\operatornamewithlimits{arg\,min}}
\def\liminf{\operatornamewithlimits{lim\,inf}}

\def\st{~~\text{s.t.}~~}
\def\defin{\triangleq}

\newcommand{\proofstep}[1]{\noindent{\bfseries #1:}~\newline}

\long\def\symbolfootnote[#1]#2{\begingroup\def\thefootnote{\fnsymbol{footnote}}\footnote[#1]{#2}\endgroup}

\newtheorem{lemma}{Lemma}[section]
\newtheorem{proposition}{Proposition}[section]

\newtheorem{definition}{Definition}[section]

\def\mybullet{~\hspace*{0.5cm}$\bullet$\hspace*{0.12cm}}
\def\begincondeq{\ifthenelse{\isundefined{\supplemental}}{
\begin{displaymath}
}{
\begin{equation}
}}
\def\endcondeq{\ifthenelse{\isundefined{\supplemental}}{
\end{displaymath}
}{
\end{equation}
}}
\def\condvspace{\ifthenelse{\isundefined{\supplemental}}{\vspace*{-0.1cm}}{}}
\def\condvspacesmall{\ifthenelse{\isundefined{\supplemental}}{\vspace*{-0.05cm}}{}}
\newcommand{\myvspace}[1]{\ifthenelse{\isundefined{\supplemental}}{\vspace*{-#1cm}}{}}

\fixstatement{theorem}
\fixstatement{lemma}
\fixstatement{proposition}

\begin{document} 

\twocolumn[
\icmltitle{Optimization with First-Order Surrogate Functions}

\icmlauthor{Julien Mairal}{julien.mairal@inria.fr}
\icmladdress{INRIA LEAR Project-Team, Grenoble, France}

\icmlkeywords{nonconvex and convex optimization, boosting, proximal gradient, incremental gradient, block coordinate descent, conditional gradient}

\vskip 0.3in
]

\begin{abstract} 
In this paper, we study optimization methods consisting of iteratively
minimizing surrogates of an objective function. By proposing several
algorithmic variants and simple convergence analyses, we make two
main contributions.  First, we provide a unified viewpoint for several
first-order optimization techniques such as accelerated proximal gradient,
block coordinate descent, or Frank-Wolfe algorithms.  Second, we introduce a
new incremental scheme that experimentally matches or outperforms
state-of-the-art solvers for large-scale optimization problems typically
arising in machine learning.

\end{abstract} 

\section{Introduction}
The principle of iteratively minimizing a majorizing surrogate of an objective
function is often called \emph{majorization-minimization} \cite{lange2}. Each
iteration drives the objective function downhill, thus giving the hope of finding a
local optimum.  A large number of existing procedures can be interpreted from
this point of view. This is for instance the case of gradient-based or proximal
methods \citep[see][]{nesterov,beck,wright}, EM algorithms \citep[see][]{neal},
DC programming \citep[][]{horst}, boosting~\citep{collins,pietra}, and some
variational Bayes techniques \citep[][]{wainwright2,seeger}. The concept of
``surrogate'' has also been used successfully in the signal processing
literature about sparse optimization \citep{daubechies,gasso} and matrix factorization
\citep{lee2,mairal7}.

In this paper, we are interested in generalizing the
ma\-jo\-ri\-za\-tion-minimization principle. Our goal is both to discover new
algorithms, and to draw connections with existing methods. We focus our study
on ``first-order surrogate functions'', which consist of approximating a
possibly non-smooth objective function up to a smooth error. We present several
schemes exploiting such surrogates, and analyze their convergence properties:
asymptotic stationary point conditions for non-convex problems, and convergence
rates for convex ones.  More precisely, we successively study:

\ifthenelse{\isundefined{\supplemental}}{
\vspace*{-0.1cm}
\hspace*{-0.1cm}{\mybullet}a generic majorization-minimization approach;\\
}{\noindent\hspace*{-0.1cm}{\mybullet}a generic majorization-minimization approach;\\}
{\mybullet}a randomized block coordinate descent algorithm \citep[see][]{tseng,shalev,nesterov6,richtarik};\\
{\mybullet}an accelerated variant for convex problems inspired
by~\citet{nesterov4,beck};\\
{\mybullet}a generalization of the ``Frank-Wolfe'' conditional gradient
method~\citep[see][]{zhang3,harchaoui3,hazan,zhang4};\\
{\mybullet}a new incremental scheme, which we call MISO.\footnote{\emph{Minimization by Incremental Surrogate Optimization}.} 

We present in this work a
unified view for analyzing a large family of algorithms with simple convergence
proofs and strong guarantees. In particular, all the above optimization methods
except Frank-Wolfe have linear convergence rates for minimizing strongly convex
objective functions. This is remarkable for MISO, the new
incremental scheme derived from our framework; to the best of our knowledge,
only two recent incremental algorithms share such a property: the
\emph{stochastic average gradient} method (SAG) of \citet{leroux}, and the
\emph{stochastic dual coordinate ascent} method (SDCA) of \citet{shalev2}.
Our scheme MISO is inspired in part by these two works, but yields different
update rules than SAG or SDCA. 

After we present and analyze the different optimization schemes, we conclude
the paper with numerical experiments focusing on the scheme MISO. We show that
in most cases MISO matches or outperforms cutting-edge solvers for large-scale
$\ell_2$- and $\ell_1$-regularized logistic regression
\citep{bradley2,beck,leroux,fan2,bottou5}.

\section{Basic Optimization Scheme}\label{sec:generic}
Given a convex subset $\Theta$ of $\Real^p$ and a continuous function~$f:\Real^p \to \Real$, we are interested in solving
\begin{displaymath}
   \min_{\theta \in \Theta} f(\theta),
\end{displaymath}
where we assume, to simplify, that $f$ is bounded below. Our goal is 
to study the majorization-minimization scheme presented in
Algorithm~\ref{alg:generic_batch} and its variants. This procedure relies on the concept of
   surrogate functions, which are minimized instead of~$f$ at every
   iteration.\footnote{Note that this concept differs from the
   machine learning terminology, where a ``surrogate'' often denotes a fixed convex
  upper bound of the nonconvex ($0\!-\!1$)-loss.}
  \myvspace{0.15}
\begin{algorithm}[hbtp]
    \caption{Basic Scheme}\label{alg:generic_batch}
    \begin{algorithmic}[1]
    \INPUT $\theta_0 \in \Theta$; $N$ (number of iterations).
    \FOR{ $n=1,\ldots,N$}
    \STATE Compute a surrogate function $g_n$ of $f$ near $\theta_{n-1}$;
    \STATE Update solution:
       $\theta_n \in \argmin_{\theta \in \Theta} g_n(\theta).$ 
    \ENDFOR
    \OUTPUT $\theta_{N}$ (final estimate);
    \end{algorithmic}
\end{algorithm}

\myvspace{0.15}
For this approach to be successful, we intuitively need surrogates that 
approximate well the objective~$f$ and that are easy to minimize.
In this paper, we focus on ``first-order surrogate
functions'' defined below, which will be shown to have ``good'' theoretical
properties.

\begin{definition}[\bfseries First-Order Surrogate]~\label{def:surrogate_batch}\newline
A function $g: \Real^p \to \Real$ is a first-order surrogate of $f$
near~$\kappa$ in~$\Theta$ when the following conditions are satisfied:
\begin{itemize}
\myvspace{0.2}
\setlength\itemindent{15pt}
   \item {\bfseries Majorization}: we have $g(\theta') \geq f(\theta')$ for all
    $\theta'$ in $\argmin_{\theta \in \Theta} g(\theta)$.
    When the more general condition $g \geq f$ holds, we say that $g$ is a \textbf{majorant} function;
   \item {\bfseries Smoothness}: the approximation error $h\defin g-f$ is
   differentiable, and its gradient is $L$-Lipschitz continuous. Moreover, we
   have $h(\kappa)=0$ and $\nabla h(\kappa)=0$. 
\end{itemize}
\myvspace{0.2}
We denote by~$\S_{L}(f,\kappa)$ the set of such surrogates, and by~$\S_{L,\rho}(f,\kappa)$ the subset of $\rho$-strongly convex surrogates.
\end{definition}

First-order surrogates have a few simple properties, which form the building
block of our analyses:
\begin{lemma}[\bfseries Basic Properties - Key Lemma]~\label{lemma:basic}\newline
Let~$g$ be in $\S_{L}(f,\kappa)$ for some $\kappa$ in $\Theta$. Define $h\defin g-f$
 and let $\theta'$ be in $\argmin_{\theta \in
\Theta} g(\theta)$. Then, for all $\theta$ in~$\Theta$,
\begin{itemize}
\setlength\itemindent{15pt}
   \item $|h(\theta)| \leq \frac{L}{2}\|\theta-\kappa\|_2^2$;
   \item $f(\theta') \leq f(\theta) + \frac{L}{2}\|\theta-\kappa\|_2^2$.
\end{itemize}
Assume that $g$ is in $\S_{L,\rho}(f,\kappa)$, then, for all $\theta$ in $\Theta$, 
\begin{itemize}
\setlength\itemindent{15pt}
   \item $f(\theta') + \frac{\rho}{2}\|\theta'-\theta\|_2^2 \leq f(\theta) + \frac{L}{2}\|\theta-\kappa\|_2^2$.
\end{itemize}
\end{lemma}
\proofatend
The first inequality is a direct applications of Lemma~\ref{lemma:upperlipschitz} applied to the function~$h$ at the point $\kappa$ when noticing that $h(\kappa)=0$ and $\nabla h(\kappa)=0$. Then, for all $\theta$ in $\Theta$, we have 
$$f(\theta') \leq g(\theta') \leq g(\theta) = f(\theta)+h(\theta),$$
and we obtain the second inequality from the first one. When $g$ is $\rho$-strongly convex, we can in addition exploit the second-order growth property of $g$ presented in Lemma~\ref{lemma:second}, and obtain
$$f(\theta') + \frac{\rho}{2}\|\theta'-\theta\|_2^2 \leq g(\theta') + \frac{\rho}{2}\|\theta-\theta'\|_2^2 \leq g(\theta)  = f(\theta)+h(\theta),$$
and the third inequality follows from the second one.
\endproofatend
The proof of this lemma is relatively simple but for space limitation reasons, all proofs
in this paper are provided as supplemental material. With Lemma~\ref{lemma:basic} in hand, we now study the
properties of Algorithm~\ref{alg:generic_batch}.  

\subsection{Convergence Analysis}
For general non-convex
problems, proving convergence to a global (or local) minimum is out of reach,
and classical analyses study instead asymptotic stationary point
conditions~\citep[see, e.g.,][]{bertsekas}.
To do so, we make the mild assumption that for all $\theta,\theta'$ in $\Theta$, the
directional derivative $\nabla f(\theta,\theta'-\theta)$ of $f$ at~$\theta$ in the direction
$\theta'-\theta$ exists.
A classical necessary first-order condition~\citep[see][]{borwein}
for~$\theta$ to be a local minimum of $f$ is to have $\nabla
f(\theta,\theta'-\theta)$ non-negative for all~$\theta'$ in~$\Theta$. 
This naturally leads us to consider the following asymptotic condition to assess
the quality of a sequence $(\theta_n)_{n \geq 0}$ for non-convex problems:
\begin{definition}[\bfseries Asymptotic Stationary Point]
A sequence $(\theta_n)_{n \geq 0}$ satisfies an asymptotic stationary point condition if
\begin{displaymath}
  \liminf_{n \to +\infty} \inf_{\theta \in \Theta} \frac{\nabla
  f(\theta_{n},\theta-\theta_n)}{ \|\theta-\theta_n\|_2} \geq 0.
\end{displaymath}
In particular, if $f$ is differentiable on $\Real^p$
and $\Theta= \Real^p$, this condition implies $\lim_{n \to +\infty} \|\nabla
f(\theta_n)\|_2 = 0$.
\end{definition}

Building upon this definition, we now give a first convergence result about Algorithm~\ref{alg:generic_batch}.
\begin{proposition}[\bfseries Non-Convex Analysis]~\label{prop:conv1}\newline
Assume that the surrogates $g_n$ from Algorithm~\ref{alg:generic_batch} are in $\S_L(f,\theta_{n-1})$ and
are majorant or strongly convex.  Then, $\!(f(\theta_n))_{n \geq 0}$ monotonically decreases and~$(\theta_n)_{n\geq 0}$ satisfies an asymptotic stationary point condition.
\end{proposition}
\proofatend
The fact that $(f(\theta_n))_{n \geq 0}$ is non-increasing and convergent because bounded below is clear:
 \begin{displaymath}
    f(\theta_n) \leq g_n(\theta_n) \leq g_n(\theta_{n-1}) = f(\theta_{n-1}),
 \end{displaymath}
 where the first inequality and the last equality come from
 Definition~\ref{def:surrogate_batch}. The second inequality comes from the
 definition of $\theta_n$.
 Denote by $f^\star$ the limit of the sequence $(f(\theta_n))_{n \geq 0}$ and by $h_n\defin g_n-f$ the approximation error
 functions. The latter are differentiable and their gradient are $L$-Lipschitz continuous
 according to the definitions of the surrogate functions. Then,
 \begin{displaymath}
     f(\theta_n) + h_n(\theta_n) = g_n(\theta_n) \leq f(\theta_{n-1}),
 \end{displaymath}
 and thus, by summing over $n$,
 \begin{displaymath}
    \sum_{n=1}^{\infty} h_n(\theta_n) \leq f(\theta_0) - f^\star,
 \end{displaymath}
 and the non-negative sequence $(h_n(\theta_n))_{n \geq 0}$ necessarily converges to zero.

We have then two possibilities (according to the assumptions made in the proposition):
\begin{itemize}
\setlength\itemindent{25pt}
   \item if the $g_n$'s are majorant surrogates, plugging $\theta' = \theta_{n}-\frac{1}{L}\nabla h_n(\theta_{n})$ in Lemma~\ref{lemma:upperlipschitz} yields
   \begin{displaymath}
   h_n(\theta') \leq h_n(\theta_n) - \frac{1}{2L}\|\nabla h_n(\theta_n)\|_2^2,
   \end{displaymath}
   and therefore,
   \begin{displaymath}
   \|\nabla h_n(\theta_n)\|_2^2 \leq 2L (h_n(\theta_n)-h_n(\theta')) \leq 2L h_n(\theta_n) \underset{n \to +\infty}{\longrightarrow} 0,
   \end{displaymath}
   where we use the fact that $h_n(\theta') \geq 0$ because $g_n$ is majorant.
   \item otherwise, the functions $g_n$ are $\rho$-strongly convex and we can exploit some inequalities of Lemma~\ref{lemma:basic}. Notably,
     $$ \frac{\rho}{2}\|\theta_n-\theta_{n-1}\|_2^2  \leq f(\theta_{n-1})-f(\theta_n).$$
     Summing this inequality over $n$ yields that $\|\theta_n-\theta_{n-1}\|_2^2$ necessarily converges to zero, and
     $$
     \|\nabla h_n(\theta_n)\|_2 =    \|\nabla h_n(\theta_n)-\nabla h_n(\theta_{n-1})\|_2 \leq L\|\theta_n-\theta_{n-1}\|_2 \underset{n \to +\infty}{\longrightarrow} 0,
     $$
     since $\nabla h_n(\theta_{n-1})=0$ according to Definition~\ref{def:surrogate_batch}.
     \end{itemize}
     We can now compute directional derivatives of $f$ at a point $\theta_n$ and a direction $\theta-\theta_n$, where $\theta$ is in~$\Theta$:
     \begin{displaymath}
     \nabla f(\theta_{n},\theta-\theta_n) = \nabla g_n(\theta_{n},\theta-\theta_n) - \nabla h_n(\theta_n)^\top (\theta-\theta_n).
     \end{displaymath}
     Note that $\theta_n$ minimizes $g_n$ on $\Theta$ and therefore $\nabla g_n(\theta_{n},\theta-\theta_n) \geq 0$. Therefore, 
     \begin{displaymath}
     \nabla f(\theta_{n},\theta-\theta_n)  \geq - \|\nabla h_n(\theta_n)\|_2 \|\theta-\theta_n\|_2,
     \end{displaymath}
     using Cauchy-Schwarz inequality. Then,
     \begin{displaymath}
     \liminf_{n \to +\infty} \inf_{\theta \in \Theta} \frac{\nabla f(\theta_{n},\theta-\theta_n)}{ \|\theta-\theta_n\|_2} \geq - \lim_{n\to+\infty} \|\nabla h_n(\theta_n)\|_2 = 0.
     \end{displaymath}
\endproofatend

Convergence results for non-convex problems are by nature weak. This is not the case
when~$f$ is convex. In the next proposition, we obtain convergence rates by following a proof technique from~\citet{nesterov} originally designed for proximal gradient methods.
\begin{proposition}[\bfseries Convex Analysis for $\S_L(f,\kappa)$]\label{prop:conv2}
Assume that $f$ is convex and that for some $R >0$,
\begin{equation}
\|\theta-\theta^\star\|_2 \leq R ~~~\text{for all}~ \theta \in \Theta \st f(\theta) \leq f(\theta_0),\label{eq:bounded}
\end{equation}
where $\theta^\star$ is a minimizer of $f$ on $\Theta$. 
When the surrogate~$g_n$ in Algorithm~\ref{alg:generic_batch} are in $\S_{L}(f,\theta_{n-1})$, we have
     $$f(\theta_n)-f^\star \leq \frac{2LR^2}{n+2}~~~\text{for all}~n \geq 1,$$
     where $f^\star \defin f(\theta^\star)$. 
Assume now that $f$ is $\mu$-strongly convex. Regardless of condition~(\ref{eq:bounded}), we have
\begin{displaymath}
f(\theta_n) - f^\star  \leq \beta^n( f(\theta_0)-f^\star)~~\text{for all}~n \geq 1,
\end{displaymath}
where $\beta \defin \frac{L}{\mu}$ if $\mu >2L$ or $\beta \defin \left(1-\frac{\mu}{4L}\right)$ otherwise.
\end{proposition}
\proofatend
We separately prove the two parts of the proposition.\\
\proofstep{Non-strongly convex case}
Let us define $h_n \defin g_n - f$ the approximation error function at iteration $n$. From Lemma~\ref{lemma:basic} (with $g\!=\!g_n, \kappa\!=\!\theta_{n-1}, \theta'\!=\!\theta_n$), we have
\begin{displaymath}
f(\theta_n) \leq \min_{\theta \in \Theta} \left[ f(\theta) + \frac{L}{2} \|\theta-\theta_{n-1}\|_2^2\right].
\end{displaymath}
Then, following a similar proof technique as~\citet[][Theorem 4]{nesterov}, we have
\begin{equation}
\begin{split}
f(\theta_n) & \leq \min_{\alpha \in [0,1]}  f(\alpha\theta^\star+(1-\alpha)\theta_{n-1}) + \frac{L\alpha^2}{2} \|\theta^\star-\theta_{n-1}\|_2^2, \\
 & \leq \min_{\alpha \in [0,1]}  \alpha f(\theta^\star) +(1-\alpha)f(\theta_{n-1}) + \frac{L\alpha^2}{2} \|\theta^\star-\theta_{n-1}\|_2^2,
\end{split} \label{eq:tmp_rate2}
\end{equation}
where the minimization over $\Theta$ in the previous equation is replaced by a minimization on the line segment $\alpha\theta^\star+(1-\alpha)\theta_{n-1} : \alpha \in [0,1]$. Then, because the sequence $(f(\theta_n))_{n \geq 0}$ is monotonically decreasing we can use the bounded level set assumption, which yields
   \begin{displaymath}
   f(\theta_n) - f^\star \leq \min_{\alpha \in [0,1]}  (1-\alpha)(f(\theta_{n-1}) - f^\star) + \frac{LR^2\alpha^2}{2}.
   \end{displaymath}
   \begin{itemize}
\setlength\itemindent{25pt}
   \item if $f(\theta_{n-1}) - f^\star \geq LR^2$, then we have the optimal value $\alpha^\star = 1$ and $f(\theta_n) - f^\star \leq \frac{LR^2}{2}$;
   \item otherwise $\alpha^\star =\frac{f(\theta_{n-1}) - f^\star}{LR^2}$. Denoting by $r_{n} \defin f(\theta_n) - f^\star$, we have
   $$r_n \leq r_{n-1}\left(1-\frac{r_{n-1}}{2LR^2}\right).$$
   Thus, $r_n^{-1} \geq r_{n-1}^{-1} \left(1-\frac{r_{n-1}}{2LR^2}\right)^{-1} \geq r_{n-1}^{-1} + \frac{1}{2LR^2}$, where the second inequality comes from the convexity inequality $(1-x)^{-1} \geq 1+x$ for $x \in (0,1)$.
   \end{itemize}
   Then, we have seen that if $r_0 \geq LR^2$, then $r_1 \leq \frac{LR^2}{2}$ and thus $r_n^{-1} \geq  r_1^{-1} + \frac{n-1}{2LR^2} \geq \frac{n+3}{2LR^2}$. Otherwise, we have $r_n^{-1} \geq r_0^{-1} + \frac{n}{2LR^2} \geq \frac{n+2}{2LR^2}$, which is sufficient to conclude.

\proofstep{$\mu$-strongly convex case}
Let us now assume that $f$ is $\mu$-strongly convex, and drop the bounded level
sets assumption.  The proof again follows~\citet{nesterov} for computing the
convergence rate of proximal gradient methods. We start from~(\ref{eq:tmp_rate2}).
   We can then use the second-order growth property of $f$ (Lemma~\ref{lemma:second}) which states that $f(\theta_{n-1}) \geq f^\star + \frac{\mu}{2}\|\theta_{n-1}-\theta^\star\|_2^2$ and obtain
   \begin{displaymath}
   f(\theta_n) - f^\star \leq \left(\min_{\alpha \in [0,1]}  1-\alpha + \frac{L\alpha^2}{\mu}\right)(f(\theta_{n-1})-f^\star).
   \end{displaymath}
   At this point, it is easy to show that 
   \begin{itemize}
\setlength\itemindent{25pt}
   \item if $\mu \geq 2L$, then the previous binomial is minimized for $\alpha^\star = 1$ and 
   \begin{displaymath}
   f(\theta_n) - f^\star \leq \frac{L}{\mu}(f(\theta_{n-1})-f^\star);
   \end{displaymath}
   \item if $\mu \leq 2L$, then we have $\alpha^\star = \frac{\mu}{2L}$ and
   \begin{displaymath}
   f(\theta_n) - f^\star \leq \left(1-\frac{\mu}{4L}\right)(f(\theta_{n-1})-f^\star),
   \end{displaymath}
\end{itemize}
which is sufficient to conclude.

\endproofatend
The result of Proposition~\ref{prop:conv2} is interesting in the sense that it
provides sharp theoretical results without making strong assumption on 
the surrogate functions.  The next proposition shows that
slightly better rates can be obtained when the surrogates are strongly convex.
\begin{proposition}[\bfseries Convex Analysis for $\S_{L,\rho}(f,\kappa)$]\label{prop:conva}
Assume that $f$ is convex and let $\theta^\star$ be a minimizer of~$f$ on~$\Theta$.
When the surrogates~$g_n$ of Algorithm~\ref{alg:generic_batch} are in $\S_{L,\rho}(f,\theta_{n-1})$ with $\rho \geq L$, we have for all $n \geq 1$,
     \begincondeq
     f(\theta_n) - f^\star  \leq  \frac{L\|\theta_{0}-\theta^\star\|_2^2}{2n}.\label{eq:rate1}
     \endcondeq
When $f$ is $\mu$-strongly convex, we have for all $n\geq 1$,
\begin{displaymath}
\left\{ \begin{array}{lcr}
     \|\theta_n-\theta^\star\|_2^2 & \leq &  \left(\frac{L}{\rho+\mu}\right)^{n}\|\theta_0-\theta^\star\|_2^2 \\
     f(\theta_n) - f^\star  & \leq & \left(\frac{L}{\rho+\mu}\right)^{n-1}\frac{L\|\theta_0-\theta^\star\|_2^2}{2}
 \end{array}\right..
\end{displaymath} 
\end{proposition}
\proofatend
We separately prove the two parts of the proposition.

\proofstep{Non-strongly convex case}
From Lemma~\ref{lemma:basic} (with $g\!=\!g_n$, $\kappa\!=\!\theta_{n-1}$, $\theta'\!=\!\theta_n$, $\theta\!=\!\theta^\star$), we have 
\begin{equation}
f(\theta_n) - f(\theta^\star) \leq \frac{L}{2}\|\theta_{n-1}-\theta^\star\|_2^2 - \frac{\rho}{2}\|\theta_n-\theta^\star\|_2^2\leq \frac{L}{2}\|\theta_{n-1}-\theta^\star\|_2^2 - \frac{L}{2}\|\theta_n-\theta^\star\|_2^2.\label{eq:rate1_proof}
\end{equation}
By summing this inequality, we have
$$ n(f(\theta_n) - f(\theta^\star)) \leq \sum_{k=1}^n (f(\theta_k) - f(\theta^\star))\leq \frac{L}{2}(\|\theta_{0}-\theta^\star\|_2^2-\|\theta_{n}-\theta^\star\|_2^2) \leq \frac{L\|\theta_{0}-\theta^\star\|_2^2}{2},$$
where the first inequality comes from the fact that $f(\theta_k) \geq f(\theta_n)$ for all $k \leq n$. This is sufficient to prove~(\ref{eq:rate1}). Note that finding telescopic sums to prove convergence rates is a classical technique~\citep[see][]{beck}.

\proofstep{$\mu$-strongly convex case}
Let us now prove the second part of the proposition and assume that $f$ is $\mu$-strongly convex. The strong convexity implies the second-order growth property of Lemma~\ref{lemma:second}: $f(\theta_n)-f^\star \geq \frac{\mu}{2}\|\theta_n-\theta^\star\|_2^2$ for all $n$. Combined with~(\ref{eq:rate1_proof}), this yields
$$
\frac{\mu+\rho}{2}\|\theta_n-\theta^\star\|_2^2 \leq \frac{L}{2}\|\theta_{n-1}-\theta^\star\|_2^2,
$$
and thus
$$
f(\theta_n) - f(\theta^\star) \leq \frac{L}{2}\|\theta_{n-1}-\theta^\star\|_2^2 \leq \left(\frac{L}{\rho+\mu}\right)^{n-1}\frac{L\|\theta_0-\theta^\star\|_2^2}{2}.
$$

\endproofatend
Note that the condition $\rho \! \geq \! L$ is relatively strong;
it can indeed be shown that~$f$ is necessarily 
$(\rho\!-\!L)$-strongly convex if $\rho \!>\! L$, and convex if $\rho\!=\!L$.
The fact that making stronger assumptions yields better convergence rates suggests that
going beyond first-order surrogates could provide even sharper results. This is
confirmed in the next proposition:
\begin{proposition}[\bfseries Second-Order Surrogates]~\label{prop:conv5}\newline
Make similar assumptions as in Proposition~\ref{prop:conv2}, and also assume
that the error functions $h_n\!\defin\! g_n\!-\!f$ are twice differentiable, that
their Hessians $\nabla^2h_n$ are $M$-Lipschitz, and that $\nabla^2h_n(\theta_{n-1})=0$ for all $n$.  Then,$$f(\theta_n)-f^\star \leq
\frac{9MR^3}{2(n+3)^2}~~~\text{for all}~n \geq 1.$$
 If $f$ is $\mu$-strongly convex, the convergence rate is
 superlinear with order $3/2$.
 \end{proposition}
 \proofatend
We separately prove the two parts of the proposition.

\proofstep{Non-strongly convex case}
Following a similar scheme as in Proposition~\ref{prop:conv2} and using Lemma~\ref{lemma:uppersecond} on the approximation error functions $h_n$ instead of Lemma~\ref{lemma:basic}, we have
\begin{equation}
f(\theta_n) \leq \min_{\alpha \in [0,1]}  f(\alpha\theta^\star+(1-\alpha)\theta_{n-1}) + \frac{M\alpha^3}{6} \|\theta^\star-\theta_{n-1}\|_2^3,\label{eq:tmp_rate3}
   \end{equation}
   Then, again following the proof of Proposition~\ref{prop:conv2},
   \begin{displaymath}
   f(\theta_n) - f^\star \leq \min_{\alpha \in [0,1]}  (1-\alpha)(f(\theta_{n-1}) - f^\star) + \frac{\alpha^3 M R^3}{6}.
   \end{displaymath}
   Denoting by $r_{n} \defin f(\theta_n) - f^\star$ and by $\alpha^\star$ the solution of this optimization problem, we have
   \begin{itemize}
\setlength\itemindent{25pt}
   \item if $r_{n-1} \geq MR^3/2$, then $\alpha^\star=1$ and $r_{n} \leq MR^3/6$;
   \item otherwise, $\alpha^\star = \sqrt{2 r_{n-1}/(MR^3)}$, and
   $$ r_n \leq r_{n-1}\left(1- \sqrt{\frac{8r_{n-1}}{9MR^3}} \right).$$
   It follows that $ r_n^{-1/2} \geq r_{n-1}^{-1/2}\left(1- \sqrt{\frac{8r_{n-1}}{9MR^3}} \right)^{-1/2} \geq r_{n-1}^{-1/2} + \sqrt{\frac{2}{9MR^3}}$, where the last inequality comes from the convexity inequality $(1-x)^{-1/2} \geq 1 + x/2$. 
   \end{itemize}
   Then, we have seen that if $r_0 \geq MR^3/2$, then $r_1 \leq MR^3/6$ and thus $r_n^{-1/2} \geq r_1^{-1/2} + (n-1)\sqrt{\frac{2}{9MR^3}} \geq\sqrt{\frac{2}{9MR^3}}(n-1+3\sqrt{3}) \geq \sqrt{\frac{2}{9MR^3}}(n+4)$.
   Otherwise, we have $r_n^{-1/2} \geq r_0^{-1/2} + n\sqrt{\frac{2}{9MR^3}} \geq \sqrt{\frac{2}{9MR^3}}(n+3)$.
   This is sufficient to obtain the first part of the proposition.

\vspace*{0.1cm}
\proofstep{$\mu$-strongly convex case}
Let us now assume that $f$ is $\mu$-strongly convex, and drop the bounded level
sets assumption. Starting again from~(\ref{eq:tmp_rate3}),
   \begin{displaymath}
   f(\theta_n) \leq \min_{\alpha \in [0,1]}  f(\alpha\theta^\star+(1-\alpha)\theta_{n-1}) + \frac{M\alpha^3}{6} \|\theta^\star-\theta_{n-1}\|_2^3,
   \end{displaymath}
   and using the second-order growth property of $f$, we have
   \begin{displaymath}
   r_n \leq \min_{\alpha \in [0,1]}  (1-\alpha)r_{n-1} + \frac{\gamma \alpha^3}{3} r_{n-1}^{3/2}.
   \end{displaymath}
   \begin{itemize}
\setlength\itemindent{25pt}
   \item when $\gamma \sqrt{r_{n-1}} \geq 1$, we have $\alpha^\star= \frac{1}{\sqrt{\gamma \sqrt{r_{n-1}}}}$ and the desired inequality follows;
   \item otherwise, $\alpha^\star=1$ and $r_n \leq \frac{3}{2}r_{n-1}^{3/2} \leq \frac{r_{n-1}}{3}$.
   \end{itemize}
\endproofatend
Consistently with this proposition, similar rates were obtained
by~\citet{nesterov7} for the Newton method with cubic regularization, which 
involve second-order surrogates.
In the next section, we focus again on first-order surrogates, and
present simple mechanisms to build them. The proofs of the different claims are provided in the
supplemental material.
\subsection{Examples of Surrogate Functions}
\label{subsec:surrogates}
\paragraph{Lipschitz Gradient Surrogates.}~\\
When $f$ is differentiable and $\nabla f$ is $L$-Lipschitz, $f$ admits the following majorant surrogate in $\S_{2L,L}(f,\kappa)$:
\begin{displaymath}
    g: \theta \mapsto f(\kappa) + \nabla f(\kappa)^\top (\theta-\kappa) + \frac{L}{2}\|\theta-\kappa\|_2^2.
\end{displaymath}
In addition, when $f$ is convex, $g$ is in $\S_{L,L}(f,\kappa)$, 
and when $f$ is $\mu$-strongly convex, $g$ is in $\S_{L-\mu,L}(f,\kappa)$.
Note also that minimizing $g$ amounts to performing a classical 
classical gradient descent step $\theta' \leftarrow \kappa - \frac{1}{L}\nabla
f(\kappa)$.
 
\paragraph{Proximal Gradient Surrogates.}~\\
Assume that $f$ splits into $f = f_1 + f_2$, where $f_1$ is differentiable with a $L$-Lipschitz gradient. Then,~$f$ admits the following majorant surrogate in $\S_{2L}(f,\kappa)$:
\begin{displaymath}
    g: \theta \mapsto f_1(\kappa) + \nabla f_1(\kappa)^\top (\theta-\kappa) + \frac{L}{2}\|\theta-\kappa\|_2^2 + f_2(\theta).
\end{displaymath}
The approximation error $g-f$ is indeed the same as in the previous paragraph and thus:
\begin{itemize}
   \item when $f_1$ is convex, $g$ is in $\S_{L}(f,\kappa)$. If $f_2$ is also convex, $g$ is in $\S_{L,L}(f,\kappa)$.
   \item when $f_1$ is $\mu$-strongly convex, $g$ is in $\S_{L-\mu}(f,\kappa)$. If $f_2$ is also convex, $g$ is in $\S_{L-\mu,L}(f,\kappa)$.
\end{itemize}
Minimizing~$g$ amounts to performing a proximal gradient step~\citep[see][]{nesterov,beck}. 
 
\paragraph{DC Programming Surrogates.}~\\
Assume that $f = f_1 + f_2$, where $f_2$ is concave and differentiable with a $L_2$-Lipschitz gradient. Then, the following function $g$ is a majorant surrogate in $\S_{L_2}(f,\kappa)$:
\begin{displaymath}
   g: \theta \mapsto f_1(\theta) + f_2(\kappa) + \nabla f_2(\kappa)^\top (\theta-\kappa).
\end{displaymath}
Such a surrogate forms the root of DC- (difference of convex functions)-programming~\citep[see][]{horst}. It is also indirectly used in reweighted-$\ell_1$ algorithms~\citep{candes4} for minimizing on $\Real_+^p$ a cost function of the form $\theta \mapsto f_1(\theta) + \lambda \sum_{i=1}^p \log( \theta_i + \varepsilon)$.

\paragraph{Variational Surrogates.}~\\
Let $f$ be a real-valued function defined on $\Real^{p_1} \times \Real^{p_2}$. Let $\Theta_1 \subseteq \Real^{p_1}$ and $\Theta_2 \subseteq \Real^{p_2}$ be two convex sets. 
Define~$\tilde{f}$ as $\tilde{f}(\theta_1) \defin \min_{\theta_2 \in \Theta_2} f(\theta_1,\theta_2)$ and
assume that 
\begin{itemize}
\item $\theta_1 \mapsto f(\theta_1,\theta_2)$ is differentiable for all~$\theta_2$ in $\Theta_2$;
\item $\theta_2 \mapsto \nabla_{1} f(\theta_1,\theta_2)$ is $L$-Lipschitz for all $\theta_1$ in $\Real^{p_1}$;\footnote{The notation $\nabla_1$ denotes the gradient w.r.t. $\theta_1$.}
\item $\theta_1 \mapsto \nabla_{1} f(\theta_1,\theta_2)$ is $L'$-Lipschitz for all $\theta_2$ in $\Theta_2$;
\item $\theta_2 \mapsto f(\theta_1,\theta_2)$ is $\mu$-strongly convex for all $\theta_1$ in~$\Real^{p_1}$. 
\end{itemize}
Let us fix $\kappa_1$ in $\Theta_1$. Then, the following function is a majorant surrogate in $\S_{2L''}(\tilde{f},\kappa)$ for some $L'' > 0$:
\condvspacesmall
\begin{displaymath}
g: \theta_1 \mapsto f(\theta_1,\kappa_2^\star) ~\text{with}~~ \kappa_2^\star \defin \argmin_{\theta_2 \in \Theta_2} \tilde{f}(\kappa_1,\theta_2).
\condvspacesmall
\end{displaymath}
When $f$ is jointly convex in $\theta_1$ and~$\theta_2$, $\tilde{f}$
is itself convex and we can choose $L''=L'$.
Algorithm~\ref{alg:generic_batch} becomes a block-coordinate descent
procedure with two blocks.

\paragraph{Saddle Point Surrogates.}~\\
Let us make the same assumptions as in the previous paragraph but with the following differences:
\begin{itemize}
\item $\theta_2 \!\mapsto\! f(\theta_1,\theta_2)$ is $\mu$-strongly concave for all $\theta_1$ in~$\Real^{p_1}$; 
\item $\theta_1  \!\mapsto\! f(\theta_1,\theta_2)$ is convex for all $\theta_2$ in $\Theta_2$;
\item $\tilde{f}(\theta_1) \defin \max_{\theta_2 \in \Theta_2} f(\theta_1,\theta_2).$
\end{itemize}
Then, $\tilde{f}$ is convex and the function below is a majorant surrogate in $\S_{2L''}(\tilde{f},\kappa_1)$:
\begin{displaymath}
g: \theta_1 \mapsto f(\theta_1,\kappa_2^\star) + \frac{L''}{2}\|\theta_1-\kappa_1\|_2^2,
\end{displaymath}
where $L''\defin \max(2L^2 /\mu,L')$.
When $\theta_1 \mapsto f(\theta_1,\theta_2)$ is affine, we can instead choose $L''\defin L^2 /\mu$.



\paragraph{Jensen Surrogates.}~\\
Jensen's inequality provides a natural mechanism to obtain surrogates for convex
functions. Following the presentation of \citet{lange2}, we consider a convex function
$f: \Real \mapsto \Real$, a vector $\x$ in~$\Real^p$, and define $\tilde{f}:
\Real^p \to \Real$ as $\tilde{f}(\theta) \defin f(\x^\top \theta)$ for
all $\theta$. Let $\w$ be a weight vector in $\Real_+^p$ such that $\|\w\|_1=1$ and $\w_i \neq 0$ whenever $\x_i \!\neq\! 0$. Then, we define for any $\kappa$ in~$\Real^p$
\begin{displaymath}
   g: \theta \mapsto \sum_{i=1}^p \w_i f \left(\frac{\x_i}{\w_i}( \theta_i-\kappa_i) + \x^\top \kappa\right),
\end{displaymath}
When $f$ is differentiable with an $L$-Lipschitz gradient, and $\w_i \defin |\x_i|^\nu / \|\x\|_\nu^\nu$, then $g$ is in $\S_{L'}(\tilde{f},\kappa)$ with
\begin{itemize}
   \item $L' = L\|\x\|_\infty^2\|\x\|_0$ for $\nu=0$;
   \item $L' = L\|\x\|_\infty\|\x\|_1$ for $\nu=1$;
   \item $L' = L\|\x\|_2^2$ for $\nu=2$.
\end{itemize}
As far as we know, the convergence rates we provide when using such surrogates
are new. We also note that Jensen surrogates have been successfully
used in machine learning. For instance, \citet{pietra} interpret boosting
procedures under this point of view through the concept of \emph{auxiliary
functions}. 

\paragraph{Quadratic Surrogates.}~\\
When $f$ is twice differentiable and admits a matrix~$\HH$ such that $\HH- \nabla^2 f$ is always positive definite, the following 
function is a first-order majorant surrogate:
\begin{displaymath}
    g: \theta \mapsto f(\kappa) + \nabla f(\kappa)^\top (\theta-\kappa) + \frac{1}{2}(\theta-\kappa)^\top\HH(\theta-\kappa).
\end{displaymath}
The Lipschitz constant of $\nabla(g-f)$ is the largest eigenvalue of
$\HH- \nabla^2 f(\theta)$ over $\Theta$.
Such surrogates appear frequently in the statistics and machine learning literature~\citep[][]{bohning,khan}.

We have shown that there are many rules to build first-order surrogates.
Choosing one instead of another mainly depends on how easy it is to build the surrogate (do we need to estimate an a priori unknown Lipschitz constant?), and 
on how cheaply it can be minimized.

\section{Block Coordinate Scheme}\label{sec:bcd}
In this section, we introduce a block coordinate descent extension of
Algorithm~\ref{alg:generic_batch}
under the assumptions that
\condvspace
\begin{itemize}
\item $\Theta$ is separable---that is, it can be written as a Cartesian product $\Theta=\Theta^1 \times \Theta^2 \times \ldots \times \Theta^k$;
\item the surrogates~$g_n$ are 
separable into $k$ components: 
\begin{displaymath}
\condvspacesmall
g_n(\theta) = \sum_{i=1}^k
g_{n}^i(\theta^i)~~\text{for}~~\theta=(\theta^1,\ldots,\theta^k) \in \Theta.
\end{displaymath}
\end{itemize}
We present a randomized procedure in Algorithm~\ref{alg:generic_bcd}
following \citet{tseng,shalev,nesterov6,richtarik}. 

\myvspace{0.10}
\begin{algorithm}[hbtp]
    \caption{Block Coordinate Descent Scheme}\label{alg:generic_bcd}
    \begin{algorithmic}[1]
    \INPUT $\theta_0=(\theta_0^1,\ldots,\theta_0^k) \in \Theta=(\Theta^1 \times \ldots \times \Theta^k)$; $N$.
    \FOR{ $n=1,\ldots,N$}
    \STATE Choose a separable surrogate $g_n$ of $f$ near~$\theta_{n-1}$;
    \STATE Randomly pick up one block $\hati_n$ and update $\theta_{n}^{\hati_n}$:
    \begin{displaymath}
\condvspacesmall
       \theta_n^{\hati_n} \in \argmin_{\theta^{\hati_n} \in \Theta^{\hati_n}} g_n^{\hati_n}(\theta^{\hati_n}). 
    \end{displaymath}
    \ENDFOR
    \OUTPUT $\theta_{N} = (\theta_N^1,\ldots,\theta_N^k)$ (final estimate);
    \end{algorithmic}
\end{algorithm}

\myvspace{0.10}
As before, we first study the convergence for non-convex problems.
The next proposition shows that  
similar guarantees as for Algorithm~\ref{alg:generic_batch} can be obtained.
\begin{proposition}[\bfseries Non-Convex Analysis]~\label{prop:conv14}\newline
Assume that the functions $g_n$ are majorant surrogates in
$\S_L(f,\theta_{n-1})$.  Assume also that
$\theta_0$ is the minimizer of a majorant surrogate
function in $\S_L(f,\theta_{-1})$ for some $\theta_{-1}$ in $\Theta$. Then, the conclusions of Proposition~\ref{prop:conv1} hold with probability one.
\end{proposition}
\proofatend
We proceed in several steps and adapt the convergence proof of
Proposition~\ref{prop:conv1} to our new setting.

\proofstep{Definition of an approximate surrogate $\barg_n$}
We define recursively the sequence of functions
$(\barg_n)_{n \geq 0}$ as follows:
\begin{displaymath}
   \barg_n \defin \barg_{n-1} + g_n^{\hati_n} - \barg_{n-1}^{\hati_n},
\end{displaymath}
where the surrogate $g_n$ and the index $\hati_n$ are chosen in the algorithm. We
also define $\barg_{-1}$ as a majorant separable surrogate function such that
$\theta_0 \in \argmin_{\theta \in \Theta} \barg_{-1}(\theta)$ (we have assumed in the proposition that such a
surrogate function exists).  Then, it is easy to see that $\barg_n$ is
constructed in such a way that $\theta_n$ is a minimizer of $\barg_n$ over
$\Theta$ for all $n \geq 0$ and that $\barg_n \geq f$.

\proofstep{Almost sure convergence of $(f(\theta_n))_{n \geq 0}$ and consequences}
We have $f(\theta_n) \leq g_n(\theta_n) = \sum_{i=1}^k g_n^k(\theta_n^k) \leq\sum_{i=1}^k g_n^k(\theta_{n-1}^k) =
g_{n}(\theta_{n-1})=f(\theta_{n-1})$ since we have $g_n^{\hati_n}(\theta_n^{\hati_n}) \leq g_n^{\hati_n}(\theta_{n-1}^{\hati_n})$ and $g_n^{i}(\theta_n^{i}) = g_n^{i}(\theta_{n-1}^{i})$ for $i \neq \hati_n$.
Thus, $(f(\theta_n))_{n \geq 0}$ is monotonically decreasing and converges almost surely.
We also have
\begin{displaymath}
\begin{split}
  \E[\barg_{n}(\theta_n) - \barg_{n-1}(\theta_{n-1})] & =    \E[\barg_{n}(\theta_n) - \barg_{n}(\theta_{n-1})] + \E[\barg_n(\theta_{n-1}) - \barg_{n-1}(\theta_{n-1})]  \\
  & =   \E[\barg_{n}(\theta_n) - \barg_{n}(\theta_{n-1})] + \E[g_n^{\hati_n}(\theta_{n-1}^{\hati_n}) - \barg_{n-1}^{\hati_n}(\theta_{n-1}^{\hati_n})] \\
  & =   \E[\barg_{n}(\theta_n) - \barg_{n}(\theta_{n-1})] + \E[\E[g_n^{\hati_n}(\theta_{n-1}^{\hati_n}) - \barg_{n-1}^{\hati_n}(\theta_{n-1}^{\hati_n}) | \theta_{n-1}]] \\
  & =   \E[\barg_{n}(\theta_n) - \barg_{n}(\theta_{n-1})] + \frac{1}{k}\E[g_n(\theta_{n-1}) - \barg_{n-1}(\theta_{n-1})] \\
  & =   \E[\barg_{n}(\theta_n) - \barg_{n}(\theta_{n-1})] + \frac{1}{k}\E[f(\theta_{n-1}) - \barg_{n-1}(\theta_{n-1})].
\end{split}
\end{displaymath}
Note that both terms $\barg_{n}(\theta_n) - \barg_{n}(\theta_{n-1})$ and $f(\theta_{n-1}) - \barg_{n-1}(\theta_{n-1})$ are non-positive with probability one and thus the sequence $(\E[\barg_n(\theta_n)])_{ n \geq 0}$ is non-increasing, bounded below and convergent.
The term $\E[\barg_{n}(\theta_n) - \barg_{n-1}(\theta_{n-1})]$ is therefore the summand of a converging sum, and so are
$\E[\barg_{n}(\theta_n) - \barg_{n}(\theta_{n-1})]$ and $\E[f(\theta_{n-1}) - \barg_{n-1}(\theta_{n-1})]$.
Then, we have by using Beppo-Levi theorem
\begin{displaymath}
   \sum_{n=0}^{+\infty} \E[\barg_n(\theta_n)-f(\theta_n)] = \E\left[ \sum_{n=0}^{+\infty} \barg_{n}(\theta_{n}) - f(\theta_{n})\right] < + \infty.
\end{displaymath}
Thus, the term $\barg_{n}(\theta_{n}) - f(\theta_{n})$ converges almost surely to $0$.

\proofstep{Asymptotic stationary point conditions}
Let us denote by ${\bar h}_n \defin \barg_n-f$ which is differentiable with
$L$-Lipschitz continuous gradient.
Then, for all $\theta$ in $\Theta$,
\begin{displaymath}
\nabla f(\theta_n,\theta-\theta_n) = \nabla \barg_n(\theta_n,\theta-\theta_n) - \nabla {\bar h}_n(\theta_n)^\top(\theta-\theta_n). 
\end{displaymath}
We have $\nabla \barg_n(\theta_n,\theta-\theta_n) \geq 0$ since $\theta_n$ is a minimizer of $\barg_n$, and $\|\nabla {\bar h}_n(\theta_n)\|_2^2 \leq 2L {\bar h}_n(\theta_n)$, following a similar argument as in the proof of Proposition~\ref{prop:conv1}.
Since we have shown that ${\bar h}_n(\theta_n)$ almost surely converges to zero, we conclude using the Cauchy-Schwarz inequality as in the proof of Proposition~\ref{prop:conv1}.

\endproofatend

Under convexity assumptions on $f$, the next two propositions give us expected
convergence rates.
\begin{proposition}[\bfseries Convex Analysis for $\S_L(f,\kappa)$]\label{prop:conv15}
Make the same assumptions as in Proposition~\ref{prop:conv2} and
define $\delta\defin \frac{1}{k}$. When the surrogate functions $g_n$ in Algorithm~\ref{alg:generic_bcd}
are majorant and in $\S_{L}(f,\theta_{n-1})$, the sequence $(f(\theta_n))_{n \geq 0}$ almost surely converges to $f^\star$ and
\begin{displaymath}
\E[f(\theta_n)-f^\star] \leq  \frac{2LR^2}{2+\delta(n-n_0)}~~\text{for all}~n \geq n_0, 
\end{displaymath}
where $n_0\defin \left\lceil \log\left(\frac{2(f(\theta_0)-f^\star)}{LR^2}-1\right) / \log\left(\frac{1}{1-\delta}\right) \right\rceil$ if $f(\theta_0)-f^\star > LR^2$ and $n_0\defin0$ otherwise.
Assume now that $f$ is $\mu$-strongly convex. Then, we have instead an expected linear convergence rate
$$\E[f(\theta_n)-f^\star] \leq ((1-\delta) + \delta\beta)^n(f(\theta_{0})-f^\star),$$
where $\beta \defin \frac{L}{\mu}$ if $\mu >2L$ or $\beta \defin \left(1-\frac{\mu}{4L}\right)$ otherwise.
\end{proposition}
\proofatend
The fact that $(f(\theta_n))_{n \geq 0}$ almost surely converges follows the beginning of
Proposition~\ref{prop:conv14}. To show the convergence rates of
$(\E[f(\theta_n)])_{ n \geq 0}$, we adapt the proof of Proposition~\ref{prop:conv2} to out
stochastic block setting.  Let us denote by $\theta_n^\star$ a minimizer of the
surrogate function $g_n$ over $\Theta$.
Since the indices~$\hati_n$ are picked up uniformly at random, we have the following conditional probabilities
$$\PPP( \theta_n^i = \theta_n^{\star i} | \theta_{n-1})=\delta ~~\text{and}~~  \PPP( \theta_n^i = \theta_{n-1}^{i} | \theta_{n-1})=1-\delta.$$
We can then obtain the following inequalities for all $\theta$ in $\Theta$
\begin{equation}
\begin{split}
\E[f(\theta_n)| \theta_{n-1}] & \leq \E[g_n(\theta_n)| \theta_{n-1}] = \sum_{i=1}^k  \E[g_n^i(\theta_n^i)| \theta_{n-1}] 
                               = \sum_{i=1}^k (1-\delta) g_n^i(\theta_{n-1}^i) + \delta g_n^i(\theta_n^{\star i}) \\
                              & = (1-\delta) g_n(\theta_{n-1})  + \delta g_n(\theta_n^{\star}) \\
                              & \leq (1-\delta)f(\theta_{n-1}) + \delta g_n(\theta) \\
                              & \leq (1-\delta)f(\theta_{n-1}) + \delta \left( f(\theta)  + \frac{L}{2}\|\theta-\theta_{n-1}\|_2^2 \right),
\end{split} \label{eq:bcd:tmp}
\end{equation}
where we have used the conditional probabilities computed above and the fact that $|g_n(\theta)-f(\theta)| \leq \frac{L}{2}\|\theta-\theta_{n-1}\|_2^2$ according to Lemma~\ref{lemma:basic}.
Let us now follow the proof of Proposition~\ref{prop:conv2}:
\begin{displaymath}
\E[f(\theta_n)| \theta_{n-1}] \leq (1-\delta)f(\theta_{n-1}) + \delta \left( \min_{\alpha \in [0,1]}  f(\alpha\theta^\star + (1-\alpha)\theta_{n-1})  + \frac{L\alpha^2}{2}\|\theta^\star-\theta_{n-1}\|_2^2 \right).
\end{displaymath}
We can now proceed by considering two different cases.

\proofstep{Case 1: without strong convexity}
To simplify the notation, we now introduce the quantities $r_n \defin f(\theta_n)-f^\star$ and following again the proof of Proposition~\ref{prop:conv2}, we have
$$ \E[r_n|\theta_{n-1}] \leq (1-\delta) r_{n-1} + \delta \left( \min_{\alpha \in [0,1]} (1-\alpha){r_{n-1}} + \frac{LR^2\alpha^2}{2}\right).$$
The term in parenthesis on the right is a concave function of $r_{n-1}$ as a pointwise infimum of concave functions (in fact, pointwise infimum of linear functions).
By taking the expectation and using Jensen inequality, we thus have
$$ \E[r_n] \leq (1-\delta) \E[r_{n-1}] + \delta \left( \min_{\alpha \in [0,1]} (1-\alpha){\E[r_{n-1}]} + \frac{LR^2\alpha^2}{2}\right).$$

By following again the proof of Proposition~\ref{prop:conv2}, we have
\begin{displaymath}
  \E[ r_n] \leq (1-\delta) \E[r_{n-1}] + \delta \left\{ \begin{array}{ll}
      \frac{LR^2}{2} & \text{if}~~ \E[r_{n-1}] > LR^2 \\
      \E[r_{n-1}]\left(1-\frac{\E[r_{n-1}]}{2LR^2}\right) & \text{otherwise}.
      \end{array} \right.
\end{displaymath}
We also notice that the inequality $\E[r_n] \leq (1-\delta)\E[r_{n-1}] +
\delta\frac{LR^2}{2}$ is always true.  This yields after simple calculations $\E[r_n] \leq (1-\delta)^n r_0
+ (1- (1-\delta)^n)\frac{LR^2}{2}$ for all $n \geq 1$. 
We also remark that the definition of $n_0$ in the proposition implies that $(1-\delta)^{n_0} r_0 +
(1-(1-\delta)^{n_0})\frac{LR^2}{2} \leq LR^2$ after some short calculations.
Thus, we have for all $n > n_0$ 
$\E[r_n] \leq \E[r_{n-1}]\left(1 - \frac{\delta\E[r_{n-1}]}{2LR^2}\right)$ and thus
$\E[r_n]^{-1} \geq \E[r_{n_0}]^{-1} + \frac{(n-n_0)\delta}{2LR^2} \geq \frac{2 + (n-n_0)\delta}{2LR^2}$, following similar
derivations as in Proposition~\ref{prop:conv2}.
This is sufficient to conclude.

\proofstep{Case 2: under strong convexity assumptions}
We proceed similarly as in case 1, but upper-bound instead $\|\theta^\star-\theta_{n-1}\|_2^2$ by $2r_{n-1}/\mu$.
This leads us to a similar relation as in the proof of Proposition~\ref{prop:conv2}:
$$
\E[r_n] \leq (1-\delta) \E[r_{n-1}] + \delta \left(\min_{\alpha \in [0,1]}  1-\alpha + \frac{L\alpha^2}{\mu}\right) \E[r_{n-1}],
$$
and following again the proof of Proposition~\ref{prop:conv2}, we have
$$\E[r_n] \leq ((1-\delta) + \delta\beta)\E[r_{n-1}],$$
yielding the desired convergence rate.
\endproofatend

\begin{proposition}[\bfseries Convex Analysis for $\S_{L,\rho}(f,\kappa)$]\label{prop:conv9}
Assume that $f$ is convex.
Define $\delta\defin \frac{1}{k}$. Choose majorant surrogates $g_n$ in $\S_{L,\rho}(f,\theta_{n-1})$ with $\rho \geq
L$, then $(f(\theta_n))_{ n \geq 0}$ almost surely converges to $f^\star$ and we have
\condvspacesmall
   \begin{displaymath} 
     \E[f(\theta_n) - f^\star]  \leq  \frac{C_0}{(1-\delta) + \delta n} ~~\text{for all}~~n \geq 1,
\condvspacesmall
   \end{displaymath}
    with $C_0\! \defin \! (1\! -\!\delta) (f(\theta_0)\!-\!f^\star)\!+\!\frac{(1\!-\!\delta)\rho\!+\!\delta L}{2}\|\theta_{0}\!-\!\theta^\star\|_2^2$.
Assume now that $f$ is $\mu$-strongly convex, then we have an expected linear convergence rate
\condvspacesmall
\begin{displaymath}
\left\{ \begin{array}{lcr}
    \frac{L}{2}\E[ \|\theta^\star-\theta_n\|_2^2] & \leq & C_0\left( (1-\delta) + \delta \frac{L}{\rho+\mu}\right)^n \\
    \E[ f(\theta_n)-f^\star] & \leq & \frac{C_0}{\delta}\left( (1-\delta) + \delta \frac{L}{\rho+\mu}\right)^{n-1}
    \end{array} \right. .
\end{displaymath}
\condvspacesmall
\end{proposition}
\proofatend
The fact that $(f(\theta_n))_{n \geq 0}$ almost surely converges follows the beginning of Proposition~\ref{prop:conv14}. 
We then separately prove the two remaining parts of the proposition.

\proofstep{Without strong convexity assumptions}
Using the same notation as in the proof of Proposition~\ref{prop:conv15}, we can replace the inequality $g_n(\theta_n^\star) \leq g_n(\theta)$ in Eq.~(\ref{eq:bcd:tmp}) by $g_n(\theta_n^\star) \leq g_n(\theta) - \frac{\rho}{2}\|\theta_n^\star-\theta\|_2^2$ (using Lemma~\ref{lemma:second}), and we obtain
\begin{displaymath}
\E[f(\theta_n)| \theta_{n-1}] \leq (1-\delta)f(\theta_{n-1}) + \delta \left( f^\star  + \frac{L}{2}\|\theta^\star-\theta_{n-1}\|_2^2 -\frac{\rho}{2}\|\theta^\star-\theta_{n}^\star\|_2^2 \right),
\end{displaymath}
We also remark that
 \begin{displaymath}
 \begin{split}
    \E\left[\| \theta^\star-\theta_n \|_2^2 | \theta_{n-1} \right] & = \sum_{i=1}^k \E\left[\| \theta^{\star i}-\theta_n^{i}\|_2^2 | \theta_{n-1} \right] = \sum_{i=1}^k (1-\delta) \| \theta^{\star i}-\theta_{n-1}^{i}\|_2^2 +\delta \| \theta^{\star i}-\theta_n^{\star i}\|_2^2 \\
    & =(1-\delta)\|\theta^\star -\theta_{n-1}\|_2^2 + \delta\|\theta^\star-\theta_n^\star\|_2^2.
    \end{split}
 \end{displaymath}
Combining the two previous inequalities yields
\begin{displaymath}
\E\left[f(\theta_n) + \frac{\rho}{2}\|\theta^\star- \theta_n\|_2^2 ~\Big|~\theta_{n-1} \right] \leq (1-\delta) f(\theta_{n-1}) +\delta f^\star + \frac{(1-\delta) \rho+ \delta L}{2}\|\theta^\star-\theta_{n-1}\|_2^2.
\end{displaymath}
Let us now define $r_n \defin \E[f(\theta_n)-f^\star]$. Taking the expectation in the previous inequality gives
\begin{equation}
\begin{split}
   r_n - (1-\delta) r_{n-1}  & \leq  \frac{\delta L+(1-\delta) \rho}{2}\E[\|\theta^\star-\theta_{n-1}\|_2^2] - \frac{\rho}{2} \E[\|\theta^\star-\theta_n^\star\|_2^2] \\
                        & \leq \frac{\delta L+(1-\delta)\rho}{2}\left(\E[\|\theta^\star-\theta_{n-1}\|_2^2] - \E[\|\theta^\star-\theta_n\|_2^2]\right).
\end{split} \label{eq:tmp_prop10}
\end{equation}
Summing these inequalities and using the fact that $r_n \leq r_{n-1}$ yields
\begin{displaymath}
   n \delta r_n + (1-\delta)(r_n-r_0) \leq \sum_{k=1}^n r_k - (1-\delta) r_{k-1} \leq \frac{\delta L+(1-\delta)\rho}{2}\|\theta^\star-\theta_{0}\|_2^2,
\end{displaymath}
which gives the desired convergence rate.

\proofstep{With strong convexity assumptions}
Assume now that $f$ is $\mu$-strongly convex. To simplify the notation, we
introduce the quantity $\xi_n \defin
\frac{1}{2}\E[\|\theta^\star-\theta_n\|_2^2]$. We can now rewrite the first inequality in~(\ref{eq:tmp_prop10}) as
$$   r_n + \rho \xi_n \leq (1-\delta) r_{n-1} + ((1-\delta) \rho + \delta L) \xi_{n-1}.$$
We are going to exploit two
inequalities. Since we have the second-order growth property $r_n \geq \mu \xi_n$, for all $\beta$ in $[0,1]$,
\begin{displaymath}
  \beta r_n + (\rho + (1-\beta)\mu)\xi_n \leq (1-\delta) r_{n-1} + ((1-\delta) \rho + \delta L) \xi_{n-1}.
\end{displaymath}
By choosing $\beta \defin \frac{(1-\delta)(\rho+\mu)}{(1-\delta)(\rho+\mu)+\delta L}$, it is easy to show that
\begin{displaymath}
  (1-\delta) r_n + ((1-\delta) \rho + \delta L)\xi_n \leq \frac{(1-\delta)(\rho+\mu)+\delta L  }{\rho+\mu} \left( (1-\delta) r_{n-1} + ((1-\delta) \rho + \delta L) \xi_{n-1}\right).
\end{displaymath}
Thus, we have by induction
\begin{displaymath}
  (1-\delta) r_n + ((1-\delta) \rho + \delta L)\xi_n \leq \left(\frac{(1-\delta)(\rho+\mu)+\delta L  }{\rho+\mu}\right)^n \left(  (1-\delta) r_{0} + ((1-\delta) \rho + \delta L) \xi_{0}\right),
\end{displaymath}
and again, since $\mu \xi_n \leq r_n$, we obtain the convergence rate of $(\xi_n)_{n \geq 0}$
\begin{displaymath}
   \xi_n \leq C \left(\frac{ (1-\delta)(\rho+\mu)+\delta L  }{\rho+\mu}\right)^n \leq \alpha^n \left( \frac{(1-\delta) r_0}{L}+\xi_0\right),
\end{displaymath}
where we have defined the quantities $\alpha \defin \frac{ (1-\delta)(\rho+\mu)+ \delta L
}{\rho+\mu}$ and $C\defin \frac{ (1-\delta) r_{0} + ((1-\delta) \rho + \delta L) \xi_{0}}{(1-\delta)(\rho+\mu) + \delta L}$. We now compute the
convergence rate of $(r_n)_{n\geq 0}$ by induction. Suppose that $r_{n-1} \leq
C' \alpha^{n-2}$ for some constant $C'$ and some $n \geq 2$. We have shown in~(\ref{eq:tmp_prop10})
that $r_n \leq (1-\delta) r_{n-1} + L \xi_{n-1}$. By using the induction
hypothesis, we have  $r_n \leq ((1-\delta) C'/\alpha + LC) \alpha^{n-1}$. We
therefore study under which conditions we have both $ ((1-\delta) C'/\alpha + LC)
\leq C'$ and $r_1 \leq C'$, which are sufficient conditions to have by
induction $r_n \leq C' \alpha^{n-1}$ for all $n$.
It is easy to show that the quantity $C' \defin  \frac{(1-\delta) r_0 + ((1-\delta) \rho+\delta L)
\xi_0}{\delta}$ satisfies such conditions.
\endproofatend

\myvspace{0.4}
The quantity $\delta\!=\!1/k$ represents the probability for
a block to be updated during an iteration. 
Note that updating all blocks ($\delta\!=\!1$) gives the same
results as in Section~\ref{sec:generic}. 
Linear convergence for strongly convex objectives with
block coordinate descent is classical since the
works of~\citet{tseng,nesterov6}. Results of the same nature have also 
been obtained by \citet{richtarik} for composite functions.

\section{Frank-Wolfe Scheme}\label{sec:conditional}
In this section, we show how to use surrogates to generalize the
Frank-Wolfe method, an old convex optimization technique that has 
regained some popularity in machine
learning \citep{zhang3,harchaoui3,hazan,zhang4}.  We present this approach in
Algorithm~\ref{alg:conditional}. 

\myvspace{0.3}
\begin{algorithm}[hbtp]
    \caption{Frank-Wolfe Scheme}\label{alg:conditional}
    \begin{algorithmic}[1]
    \INPUT $\theta_0 \in \Theta$; $N$ (number of iterations).
    \FOR{ $n=1,\ldots,N$}
    \STATE Let~$g_n$ be a majorant surrogate in~$\S_{L,L}(f,\theta_{n-1})$.
    \STATE Compute a search direction:
    \condvspace
    \begin{displaymath}
       \nu_n \in \argmin_{\theta \in \Theta} \Big[ g_n(\theta)-\frac{L}{2}\|\theta-\theta_{n-1}\|_2^2\Big]. 
    \condvspace
    \end{displaymath}
    \condvspace
    \STATE Line search:
     $  \alpha^\star \! \defin\!  \displaystyle \argmin_{\alpha \in [0,1]} g_n(\alpha \nu_n  + (1-\alpha) \theta_{n-1})$. 
    \STATE Update solution:
       $\theta_n \defin \alpha^\star\nu_n + (1-\alpha^\star)\theta_{n-1}.$ 
    \ENDFOR
    \OUTPUT $\theta_{N}$ (final estimate);
    \end{algorithmic}
\end{algorithm}

\condvspace
When $f$ is smooth and the ``gradient Lipschitz based surrogates'' from
Section~\ref{subsec:surrogates} are used, Algorithm~\ref{alg:conditional}
becomes the classical Frank-Wolfe method.\footnote{Note that the classical
Frank-Wolfe algorithm performs in fact the line search over the function $f$
and not $g_n$.} Our point of view is however more general since it allows 
for example to use
``proximal gradient surrogates''.
The next proposition gives a convergence rate.

\begin{proposition}[\bfseries Convex Analysis]\myvspace{0.55}~\label{prop:conv11}\newline
Assume that $f$ is convex and that $\Theta$ is bounded. Call $R\defin \max_{\theta_1,\theta_2 \in \Theta} \|\theta_1-\theta_2\|_2$ the diameter of~$\Theta$. Then, the sequence $(f(\theta_n))_{n \geq 0}$ provided by Algorithm~\ref{alg:conditional} converges to the minimum $f^\star$ of $f$ over $\Theta$ and
\condvspacesmall
\begin{displaymath}
f(\theta_n)-f^\star \leq \frac{2LR^2}{n+2} ~~~~\text{for all}~n \geq 1.
\condvspacesmall
\end{displaymath}
\end{proposition}
\proofatend
    We have from the strong convexity of $g_n$:
    \begin{displaymath}
       f(\theta_n) \leq g_n(\theta_n) \leq \min_{\alpha \in [0,1]} (1-\alpha) g_n(\theta_{n-1}) + \alpha g_n(\nu_n) - \frac{L}{2}\alpha(1-\alpha)\|\theta_{n-1}-\nu_n\|_2^2,
    \end{displaymath}
    where $\nu_n$ is defined in Algorithm~\ref{alg:conditional}. Let us now consider $\theta^\star$ such that $f(\theta^\star)=f^\star$. Then, we have
    \begin{equation}
    \begin{split}
       f(\theta_n) &\leq \min_{\alpha \in [0,1]} (1-\alpha) f(\theta_{n-1}) + \alpha\left( g_n(\nu_n)-\frac{L}{2}\|\nu_n-\theta_{n-1}\|_2^2\right) + \frac{L}{2}(\alpha - \alpha(1-\alpha))\|\nu_n-\theta_{n-1}\|_2^2.\\
       & \leq \min_{\alpha \in [0,1]} (1-\alpha) f(\theta_{n-1}) + \alpha\left( g_n(\theta^\star)-\frac{L}{2}\|\theta^\star-\theta_{n-1}\|_2^2\right) + \frac{\alpha^2LR^2}{2}\\
       & \leq \min_{\alpha \in [0,1]} (1-\alpha) f(\theta_{n-1}) + \alpha f^\star + \frac{\alpha^2LR^2}{2},
       \end{split} \label{eq:conditional}
    \end{equation}
    where we have first used the equality $g_n(\theta_{n-1})=f(\theta_{n-1})$, then the second inequality exploits $g_n(\nu_n)-\frac{L}{2}\|\nu_n-\theta_{n-1}\|_2^2 \leq g_n(\theta^\star)-\frac{L}{2}\|\theta^\star-\theta_{n-1}\|_2^2$ from the definition of $\nu_n$. Finally, we use the fact that $g_n(\theta^\star) = f^\star + h_n(\theta^\star)$ where $h_n$ is the approximation error function $g_n-f$ with $|h_n(\theta^\star)| \leq \frac{L}{2}\|\theta^\star-\theta_{n-1}\|_2^2$ is ensured by Lemma~\ref{lemma:basic}.
    Minimizing~(\ref{eq:conditional}) with respect to $\alpha$ and denoting by $r_n\defin f(\theta_n)-f^\star$ yields
    \begin{displaymath}
       r_n \leq \left\{ \begin{array}{ll}
       \frac{LR^2}{2} & \text{if}~~ r_{n-1} \leq LR^2 \\
       r_{n-1}\left(1-\frac{r_{n-1}}{2LR^2}\right) & \text{otherwise}
       \end{array}\right..
    \end{displaymath}
    These are the same relations used in the proof of Proposition~\ref{prop:conv2}, leading therefore to the same convergence rate.
\endproofatend

\condvspace
Other extensions of Algorithm~\ref{alg:conditional} can also easily be
designed by using our framework. We present for instance in the supplemental material
a randomized block Frank-Wolfe algorithm, revisiting
the recent
work of~\citet{lacoste}. 
\condvspace

\section{Accelerated Scheme}
A popular scheme for convex optimization is the accelerated proximal
gradient method~\citep{nesterov,beck}. By using surrogate functions, we
exploit similar ideas in Algorithm~\ref{alg:accelerated_batch}. When
using the ``Lipschitz gradient surrogates'' of
Section~\ref{subsec:surrogates}, Algorithm~\ref{alg:accelerated_batch} is
exactly the scheme 2.2.19 of \citet{nesterov4}. When using the
``proximal gradient surrogate'' and when $\mu=0$, it is equivalent to the
FISTA method of~\citet{beck}.
Algorithm~\ref{alg:accelerated_batch} consists of iteratively minimizing a surrogate 
computed at a point $\kappa_{n-1}$ extrapolated from $\theta_{n-1}$
and~$\theta_{n-2}$. It results in better
convergence rates, as shown in the next proposition by adapting
a proof technique of~\citet{nesterov4}.

  \myvspace{0.35}
\ifthenelse{\isundefined{\supplemental}}{
\begin{algorithm}[h!]
}{
\begin{algorithm}[hbtp]
}
    \caption{Accelerated Scheme}\label{alg:accelerated_batch}
    \begin{algorithmic}[1]
    \INPUT $\theta_0 \in \Theta$;  $N$; $\mu$ (strong convexity parameter);
    \STATE Initialization: $\kappa_0 \defin \theta_0$; $a_0=1$;
    \FOR{ $n=1,\ldots,N$}
    \STATE Choose a surrogate $g_n$ in $\S_{L,L+\mu}(f,\kappa_{n-1})$;
    \STATE Update solution:
       $\theta_n \defin \argmin_{\theta \in \Theta} g_n(\theta)$;
    \STATE Compute $a_n \geq 0$ such that:
    $$ \textstyle a_n^2 = (1-a_n)a_{n-1}^2 + \frac{\mu}{L+\mu} a_n;$$ 
    \STATE Set $\beta_n\defin \frac{a_{n-1}(1-a_{n-1})}{a_{n-1}^2 + a_n}$ and update $\kappa$:
    \begin{displaymath}
        \kappa_n \defin \theta_n + \beta_n(\theta_n-\theta_{n-1});
    \end{displaymath}
    \ENDFOR
    \OUTPUT $\theta_{N}$ (final estimate);
    \end{algorithmic}
\end{algorithm}

\begin{proposition}[\bfseries Convex Analysis]\myvspace{0.15}~\label{prop:accelerated}\newline
Assume that $f$ is convex. When $\mu=0$, the sequence $(\theta_n)_{n \geq 0}$ provided by Algorithm~\ref{alg:accelerated_batch} satisfies for all $n \geq 1$,
\condvspacesmall
\begin{displaymath}
f(\theta_n) - f^\star  \leq \frac{2L\|\theta_0-\theta^\star\|_2^2}{(n+2)^2}. 
\condvspacesmall
\end{displaymath}
When $f$ is $\mu$-strongly convex, we have instead a linear convergence rate:
for $n\geq 1$,
\condvspacesmall
\begin{displaymath}
f(\theta_n) - f^\star  \leq \left( 1- \sqrt{\frac{\mu}{L+\mu}}\right)^{n-1}\frac{L\|\theta_0-\theta^\star\|_2^2}{2}.
\end{displaymath}
\myvspace{0.4}
\end{proposition}
\proofatend
We follow the proof techniques introduced by~\citet{nesterov4} using the so
called ``estimate sequences'', and more precisely we adapt the proof of~\citet[][Theorem 2.2.8]{nesterov4}
to deal with our surrogate functions.

\proofstep{Preliminaries}
We rely heavily on Lemma~\ref{lemma:basic}, which we recall and expand here. Let us define $\rho \defin L+\mu$. Then, for all $\theta$ in $\Theta$,
\begin{equation}
   \begin{split}
   f(\theta_n) & \leq f(\theta) + \frac{L}{2}\|\theta-\kappa_{n-1}\|_2^2 - \frac{\rho}{2}\|\theta-\theta_n\|_2^2 \\
   & = f(\theta) + \frac{L}{2}\|\theta-\kappa_{n-1}\|_2^2 - \frac{\rho}{2}\|\theta-\kappa_{n-1}+\kappa_{n-1}-\theta_n\|_2^2 \\
   & = f(\theta) - \frac{\mu}{2}\|\theta-\kappa_{n-1}\|_2^2 - \frac{\rho}{2}\|\theta_n-\kappa_{n-1}\|_2^2 + \rho (\theta-\kappa_{n-1})^\top(\theta_n - \kappa_{n-1}).
   \end{split} \label{eq:tmp_accelerated1}
\end{equation}
To simplify the notation in the sequel, we introduce the quantity~$\xi_n \defin \theta_n-\kappa_{n-1}$, which~\citet{nesterov4} calls ``gradient mapping'', up to a multiplicative constant. Then,~(\ref{eq:tmp_accelerated1}) can be rewritten 
\begin{equation}
   f(\theta_n) \leq f(\theta) - \frac{\mu}{2}\|\theta-\kappa_{n-1}\|_2^2 - \frac{\rho}{2}\|\xi_n\|_2^2 + \rho (\theta-\kappa_{n-1})^\top \xi_n.\label{eq:tmp_accelerated2}
\end{equation}
\proofstep{Definition of the estimate sequence by induction}
Keeping in mind this key quantity, let us now proceed by induction to prove the main result. The recursion hypothesis $\H_{n}$ for $n\geq 1$ is the existence of a function $\barg_{n}: \Real^p \to \Real$ such that 
\begin{equation}
   \left\{  \begin{array}{ll}
       \barg_{n}(\theta) & = \barg_{n}^\star + \frac{\gamma_{n}}{2}\|\theta-v_{n}\|_2^2 \\
       \barg_{n}(\theta) & \leq f(\theta) + \frac{A_{n}}{2}\|\theta-\theta_0\|_2^2 ~~~ \forall \theta \in \Theta \\
       f(\theta_{n}) & \leq  \barg_{n}^\star \\
       (\rho a_{n} + \gamma_{n}) \kappa_{n} & = \rho a_{n} \theta_{n} + \gamma_{n} v_{n}
       \end{array} \right. , \tag{$\H_{n}$}
\end{equation}
for some $v_{n}$ and some values $A_{n},\gamma_{n}$ recursively defined as follows:
$A_{k}=A_{k-1}(1-a_{k-1})$ and $\gamma_k=(1-a_{k-1})\gamma_{k-1}+\mu a_{k-1}$ for all $k \geq 2$, and 
$A_1=L$, $\gamma_1=\rho$. We recall that the scalars $a_k$ are also defined in the
algorithm.
The functions $\barg_{n}$ which we are going to
recursively define are related to the ``estimate sequences'' introduced by~\citet{nesterov4}.
Along with the quantity $A_{n}$, they indeed reflect the convergence rate of
the algorithm, since $\H_{n}$ implies that $f(\theta_{n})-f^\star \leq
\frac{A_{n}}{2}\|\theta^\star-\theta_0\|_2^2$.

\proofstep{Initialization of the induction for $n=1$}
Let us first initialize the induction, by showing that $\H_1$ is true. We remark that $A_1=L$ and $\gamma_1=\rho$ are chosen such that
We can thus define 
\begin{displaymath}
   \barg_1(\theta) = f(\theta_1) + \frac{\gamma_1}{2}\|\theta-\theta_1\|_2^2.
\end{displaymath}
In other words, we define $v_1\defin\theta_1$ and $\barg_1^\star\defin
f(\theta_1)$, and we obviously have the first and third conditions of $\H_1$.
The second one is simply an application of Lemma~\ref{lemma:basic}, when
noticing that $\kappa_0=\theta_0$.
The last condition is also satisfied because $\kappa_1=\theta_1=v_1$ (since $\beta_1=0$ in the algorithm).

\proofstep{Induction argument}
Since we have shown that $\H_1$ is true, we now assume $\H_{n-1}$ for $n\geq 2$
and show $\H_n$. We define~$\barg_n: \Real^p \to \Real$ such that for all $\theta$ in $\Real^p$ 
\begin{equation}
\barg_n(\theta)  = (1-a_{n-1})\barg_{n-1}(\theta) + a_{n-1} \left( f(\theta_n) + \frac{\mu}{2}\|\theta-\kappa_{n-1}\|_2^2 + \frac{\rho}{2}\|\xi_n\|_2^2 - \rho (\theta-\kappa_{n-1})^\top \xi_n\right). \label{eq:tmp_accelerated3}
  \end{equation}
Because of~(\ref{eq:tmp_accelerated2}), the term between parenthesis on the right is smaller than $f(\theta)$ and thus, we have $\barg_n(\theta) \leq f(\theta) + (1-a_{n-1})\frac{A_{n-1}}{2}\|\theta-\theta_0\|_2^2=f(\theta) + \frac{A_{n}}{2}\|\theta-\theta_0\|_2^2$, by definition of $A_n$. Thus, the second condition of $\H_n$ is true. The function $\barg_n$ is moreover quadratic and the first condition is easy to check (using appropriate values for $\barg_n^\star$ and $v_n$).
Let us now check the third condition, namely that $f(\theta_n) \leq \min_{\theta \in \Theta} \barg_n(\theta)$. 
We first remark that
\begin{displaymath}
\begin{split}
   \barg_{n-1}(\theta) & = \barg_{n-1}^\star + \frac{\gamma_{n-1}}{2}\|\theta-v_{n-1}\|_2^2 \\
                       & \geq f(\theta_{n-1}) + \frac{\gamma_{n-1}}{2}\|\theta-v_{n-1}\|_2^2 \\
                       & \geq f(\theta_{n}) + \frac{\rho}{2}\|\xi_n\|_2^2 - \rho (\theta_{n-1}-\kappa_{n-1})^\top\xi_n + \frac{\gamma_{n-1}}{2}\|\theta-v_{n-1}\|_2^2, 
\end{split}
\end{displaymath}
The first inequality comes from the induction hypothesis $\H_{n-1}$ and the second inequality comes from~(\ref{eq:tmp_accelerated2}).
Then, we can combine this inequality with~(\ref{eq:tmp_accelerated3}).  
\begin{equation}
   \barg_n(\theta) \geq f(\theta_n) + \frac{\rho}{2}\|\xi_n\|_2^2 - (1-a_n)\rho(\theta_{n-1}-\kappa_{n-1})^\top\xi_n + B(\theta),\label{eq:tmp_accelerated4}
\end{equation}
where
\begin{displaymath}
B(\theta) \defin \frac{(1-a_{n-1})\gamma_{n-1}}{2}\|\theta-v_{n-1}\|_2^2 + \frac{a_{n-1}\mu}{2}\|\theta-\kappa_{n-1}\|_2^2 - \rho a_{n-1} (\theta-\kappa_{n-1})^\top \xi_n.
\end{displaymath}
Note that $B(\theta)$ is also the part of $\barg_n$ dependent on~$\theta$, and such that $v_n = \argmin_{\theta \in \Real^p} B(\theta)$.
Minimizing $B(\theta)$ yields
\begin{equation}
  v_n = \frac{1}{\gamma_{n}}\left( (1-a_{n-1})\gamma_{n-1} v_{n-1} + a_{n-1} \mu \kappa_{n-1}\right) + \frac{\rho a_{n-1}}{\gamma_n} \xi_n, \label{eq:tmp_accelerated5}
\end{equation}
where we recall that $\gamma_n = (1-a_{n-1})\gamma_{n-1} + \mu a_{n-1}$.
Moreover, we have the convexity inequality
\begin{displaymath}
\begin{split}
   B(\theta) & = \frac{\gamma_n}{2}\left(\frac{(1-a_{n-1})\gamma_{n-1}}{\gamma_n}\|\theta-v_{n-1}\|_2^2 + \frac{a_{n-1}\mu}{\gamma_n}\|\theta-\kappa_{n-1}\|_2^2\right) - \rho a_{n-1} (\theta-\kappa_{n-1})^\top \xi_n \\
   & \geq \frac{\gamma_n}{2}\left\|\theta- \left(\frac{(1-a_{n-1})\gamma_{n-1}}{\gamma_n}v_{n-1} + \frac{a_{n-1}\mu}{\gamma_n}\kappa_{n-1}\right)\right\|_2^2 - \rho a_{n-1} (\theta-\kappa_{n-1})^\top \xi_n.
   \end{split}
\end{displaymath}
and thus, using the closed form of $v_n$ computed in~(\ref{eq:tmp_accelerated5}), we have
\begin{displaymath}
\begin{split}
   B(v_n) & \geq \frac{\gamma_n}{2}\left\| \frac{\rho a_{n-1}}{\gamma_n} \xi_n \right\|_2^2-  \frac{\rho a_{n-1}(1-a_{n-1})\gamma_{n-1}}{\gamma_{n}}\left( v_{n-1}-\kappa_{n-1}\right)^\top \xi_n  - \frac{\rho^2 a_{n-1}^2}{\gamma_n} \|\xi_n\|_2^2 \\
          & = -\frac{\rho^2 a_{n-1}^2}{2\gamma_n}\|\xi_n\|_2^2 -  \frac{\rho a_{n-1}(1-a_{n-1})\gamma_{n-1}}{\gamma_{n}}\left( v_{n-1}-\kappa_{n-1}\right)^\top \xi_n.
\end{split}
\end{displaymath}
We can now obtain the following lower-bound on $\barg_n^\star \defin \min_{\theta \in \Real^p} \barg_n(\theta)$, plugging the value of $B(v_n)$ into~(\ref{eq:tmp_accelerated4}),
\begin{displaymath}
   \barg_n^\star \geq f(\theta_n) + \left(\frac{\rho}{2} - \frac{\rho^2 a_{n-1}^2}{2\gamma_n}\right)\|\xi_n\|_2^2 - (1-a_{n-1})\rho\left( \theta_{n-1}-\kappa_{n-1}  + \frac{a_{n-1}\gamma_{n-1}}{\gamma_n}(v_{n-1}-\kappa_{n-1})   \right)^\top\xi_n.
\end{displaymath}
Given the definitions of $\gamma_n$ and $a_n$, and the fact that
$\rho a_0^2=\gamma_1$, we also obviously have the relation $\rho
a_{n-1}^2=\gamma_n$ for all $n \geq 0$. This cancels the factor in front of
$\|\xi_n\|_2^2$. It is also easy to show that the fourth condition of $\H_{n-1}$ implies
$\theta_{n-1}-\kappa_{n-1}  + \frac{a_{n-1}\gamma_{n-1}}{\gamma_n}(v_{n-1}-\kappa_{n-1})=0$.

Since we have shown the three first conditions of $\H_n$, it remains to show
the last one, namely that $(\rho a_{n} + \gamma_{n}) \kappa_{n} = \rho a_{n} \theta_{n} + \gamma_{n} v_{n}$.
We first remark that~(\ref{eq:tmp_accelerated5}) can be rewritten
\begin{displaymath}
    \gamma_n v_n  = (1-a_{n-1})\gamma_{n-1} v_{n-1} + a_{n-1} (\mu-\rho) \kappa_{n-1} + \rho a_{n-1} \theta_n.
\end{displaymath}
Combining with the fourth condition of $\H_{n-1}$, we have
\begin{displaymath}
\begin{split}
    \gamma_n v_n  & = (1-a_{n-1})\left(  (\rho a_{n-1} + \gamma_{n-1}) \kappa_{n-1}  -\rho a_{n-1} \theta_{n-1} \right) + a_{n-1} (\mu-\rho) \kappa_{n-1} + \rho a_{n-1} \theta_n \\ 
                  & = -(1-a_{n-1})a_{n-1}\rho \theta_{n-1} + \rho a_{n-1} \theta_n \\
                  & = \gamma_n\left( \theta_{n-1} + \frac{1}{a_{n-1}}(\theta_{n}-\theta_{n-1})\right),
\end{split}
\end{displaymath}
where we use the relation $\rho a_{n-1}^2 = \gamma_n$ and the recursive relation between $\gamma_n$ and $\gamma_{n-1}$ to remove the terms depending on $\kappa_{n-1}$ in the first equation. Now that we have a simple form describing~$v_n$, we can finally show,
\begin{displaymath}
    \frac{\rho a_{n} \theta_n + \gamma_n v_n}{\rho a_{n} + \gamma_n} = \theta_n + \frac{\gamma_n(1/a_{n-1} - 1)}{\rho a_{n}+\gamma_n} (\theta_n-\theta_{n-1}).
\end{displaymath}
And some simple computation shows that the right part of this equation is equal to $\kappa_n$.
In other words, the factor in front of $(\theta_n-\theta_{n-1})$ is equal to~$\beta_n$, 
and the last condition of $\H_n$ is satisfied.

\proofstep{Obtaining the convergence rate}
Since $\H_n$ is true for all $n \geq 1$, we have $f(\theta_{n})-f^\star \leq
\frac{A_{n}}{2}\|\theta^\star-\theta_0\|_2^2$ and thus it remains to compute
the convergence rate of the sequence $A_n$ to prove the main result.
We follow here the proof of~\citet[][Lemma 2.2.4]{nesterov4}.
Let us first look at the case $\mu=0$. It is easy to show by induction that for all $n\geq 1$, we have
$A_n=La_{n-1}^2$. Moreover 
\begin{displaymath}
  \frac{1}{a_n}- \frac{1}{a_{n-1}} = \frac{a_{n-1}-a_n}{a_{n-1}a_n} =  \frac{a_{n-1}^2-a_n^2}{a_{n-1}a_n(a_{n-1}+a_n)}  =  \frac{a_{n-1}^2a_n}{a_{n-1}a_n(a_{n-1}+a_n)}  =  \frac{a_{n-1}}{a_{n-1}+a_n}  \geq \frac{1}{2}
\end{displaymath}
where we use the relation $a_n^2=(1-a_n)a_{n-1}^2$ and and the fact that $a_n
\leq a_{n-1}$ for all $n\geq 1$.
Thus, we have 
\begin{displaymath}
   \frac{1}{a_n} - \frac{1}{a_0} = \frac{1}{a_n} - 1 \geq \frac{n}{2},
\end{displaymath}
and $a_n \leq 2/(n+2)$. Since $A_n=L a_{n-1}^2$, this gives us the desired convergence rate.

When $\mu > 0$, we have the relation $a_n^2=(1-a_n)a_{n-1}^2-\frac{\mu}{\rho}a_n$. It is then easy to show by induction 
that for all $n\geq 0$, we have $a_n \geq \sqrt{\frac{\mu}{\rho}}$. Thus, $A_n \leq \left(1-\sqrt{\frac{\mu}{\rho}}\right)^{n-1}A_1$. Since $A_1=L$, we have obtain the second convergence rate.
\endproofatend

\section{Incremental Scheme}\label{sec:incremental}
This section is devoted to objective
functions~$f$ that split into many components:
\condvspace
\begin{equation}
f(\theta) = \frac{1}{T}\sum_{t=1}^T f^t(\theta). \label{eq:def_incremental}
\end{equation}
The most classical method exploiting such a structure when $f$ is smooth is probably the stochastic gradient descent (SGD)
and its variants~\citep[see][]{bottou5}. It consists of drawing at
iteration~$n$ an index~${\hat t}_n$ and 
updating the solution as $\theta_n \!\leftarrow\! \theta_{n-1} \!- \eta_n \nabla  f^{{\hat t}_n}(\theta_{n-1})$
with a scalar~$\eta_n$.
Another popular algorithm is the
\emph{stochastic mirror descent}~\citep[see][]{juditsky} for general non-smooth
convex problems, a setting we do not consider in this paper since non-smooth functions
do not always admit first-order surrogates.

Recently, it was shown by \citet{shalev2} and \citet{leroux} that
linear convergence rates could be obtained for strongly convex functions~$f^t$.
The SAG 
algorithm of \citet{leroux} for smooth unconstrained
optimization is an approximate gradient descent strategy, where an estimate of~$\nabla f$ is incrementally updated at each iteration. The work of
\citet{shalev2} for composite optimization is a dual coordinate ascent method called SDCA
which performs incremental updates in the primal~(\ref{eq:def_incremental}).
Unlike SGD, both SAG and SDCA require storing information
about past iterates.

In a different context, incremental EM algorithms have been proposed
by~\citet{neal}, where surrogates of a log-likelihood are incrementally updated.
By using similar ideas, we present in
Algorithm~\ref{alg:generic_incremental} a scheme for
solving~(\ref{eq:def_incremental}), which we call MISO.  
In the next propositions, we study its convergence properties. 

\myvspace{0.3}
\ifthenelse{\isundefined{\supplemental}}{
\begin{algorithm}[h!]
}{
\begin{algorithm}[hbtp]
}
    \caption{Incremental Scheme MISO}\label{alg:generic_incremental}
    \begin{algorithmic}[1]
    \INPUT $\theta_0 \in \Theta$; $N$ (number of iterations).
    \STATE Choose surrogates $g_0^t$ of $f^t$ near $\theta_0$ for all $t$;
    \FOR{ $n=1,\ldots,N$}
    \STATE Randomly pick up one index $\hat{t}_n$ and choose a surrogate $g_n^{\hat{t}_n}$ of $f^{\hat{t}_n}$ near $\theta_{n-1}$. Set $g^t_n \defin g^t_{n-1}$ for $t \neq \hat{t}_n$;
    \STATE Update solution:
$       \theta_n \in {\displaystyle \argmin_{\theta \in \Theta}} \frac{1}{T} \sum_{t=1}^T g_n^t(\theta)$. 
    \ENDFOR
    \OUTPUT $\theta_{N}$ (final estimate);
    \end{algorithmic}
\end{algorithm}

\begin{proposition}[\bfseries Non-Convex Analysis]~\label{prop:conv13}\newline
Assume that the surrogates $g_n^{{\hat t}_n}$ from
Algorithm~\ref{alg:generic_incremental} are majorant and are in $\S_{L}(f^{{\hat
t}_n},\theta_{n-1})$.
 Then, the conclusions of Proposition~\ref{prop:conv1} hold with probability one.
\end{proposition}
\proofatend
The proof is very similar to the one of Proposition~\ref{prop:conv14}.
We proceed in several steps.

\proofstep{Almost sure convergence of $f(\theta_n)$}
Let us denote by $\barg_n \defin \frac{1}{T}\sum_{t=1}^T g_n^t$. We have the
following recursion relation
\begin{displaymath}
   \barg_n = \barg_{n-1} + g_n^{{\hat t}_n} - g_{n-1}^{{\hat t}_n}, 
\end{displaymath}
where the surrogates and the index ${\hat t}_n$ are chosen in the algorithm.
This allows us to obtain the following inequalities, which hold with probability one
\begin{displaymath}
\begin{split}
   \barg_{n}(\theta_n) & \leq \barg_n(\theta_{n-1}) =  \barg_{n-1}(\theta_{n-1}) + g_n^{{\hat t}_n}(\theta_{n-1}) - g_{n-1}^{{\hat t}_n}(\theta_{n-1})  \\
 & = \barg_{n-1}(\theta_{n-1}) + f^{{\hat t}_n}(\theta_{n-1}) - g_{n-1}^{{\hat t}_n}(\theta_{n-1}) \leq \barg_{n-1}(\theta_{n-1}).
 \end{split}
\end{displaymath}
The first inequality is true by definition of $\theta_n$ and the second one because $\barg_{n-1}^{{\hat t}_n}$ is a majorant surrogate of $f^{{\hat t}_n}$. 
The sequence $(\barg_n(\theta_n))_{ n \geq 0}$ is thus monotonically decreasing, bounded below with probability one and thus converges almost surely.
Note now that the previous inequalities imply
\begin{equation}
 \E[\barg_{n}(\theta_n)] - \E[\barg_{n-1}(\theta_{n-1})] \leq   \E[f^{{\hat t}_n}(\theta_{n-1}) - g_{n-1}^{{\hat t}_n}(\theta_{n-1})]. \label{eq:incr:tmp}
\end{equation}
The non-positive term $\E[\barg_{n}(\theta_n)] - \E[\barg_{n-1}(\theta_{n-1})]$ is the summand of a converging sum. Thus, the non-positive terms $\E[f^{{\hat t}_n}(\theta_{n-1}) - g_{n-1}^{{\hat t}_n}(\theta_{n-1})]$ is also the summand of a converging sum and we have
\begin{displaymath}
\begin{split}
\E\left[ \sum_{n=0}^{+\infty} g_{n}^{{\hat t}_{n+1}}(\theta_{n}) - f^{{\hat t}_{n+1}}(\theta_{n})\right] & = \sum_{n=0}^{+\infty} \E[g_{n}^{{\hat t}_{n+1}}(\theta_{n}) - f^{{\hat t}_{n+1}}(\theta_{n})]  \\
   & = \sum_{n=0}^{+\infty} \E[\E[g_{n}^{{\hat t}_{n+1}}(\theta_{n}) - f^{{\hat t}_{n+1}}(\theta_{n})| \theta_{n}]] \\
   & = \sum_{n=0}^{+\infty} \E[\barg_{n}(\theta_{n}) - f(\theta_{n})] \\
   & =\E\left[\sum_{n=0}^{+\infty} \barg_{n}(\theta_{n}) - f(\theta_{n})\right] < +\infty, \\
\end{split}
\end{displaymath}
where we use two times Beppo-L\'evy theorem to exchange the expectation and the sum signs in front of non-negative quantities.
As a result, the term $\barg_{n}(\theta_{n}) - f(\theta_{n})$ converges almost surely to $0$, implying the almost sure convergence of $f(\theta_n)$.

\proofstep{Asymptotic stationary point conditions}
Let us denote by ${\bar h}_n \defin \barg_n-f$ which is differentiable with
$L$-Lipschitz continuous gradient.
Then, for all $\theta$ in $\Theta$,
\begin{displaymath}
\nabla f(\theta_n,\theta-\theta_n) = \nabla \barg_n(\theta_n,\theta-\theta_n) - \nabla {\bar h}_n(\theta_n)^\top(\theta-\theta_n). 
\end{displaymath}
We have $\nabla \barg_n(\theta_n,\theta-\theta_n) \geq 0$ by definition of $\theta_n$, and $\|\nabla {\bar h}_n(\theta_n)\|_2^2 \leq 2L {\bar h}_n(\theta_n)$, following a similar argument as in the proof of Proposition~\ref{prop:conv1}.
Since we have shown that ${\bar h}_n(\theta_n)$ almost surely converges to zero, we conclude as in the proof of Proposition~\ref{prop:conv14}.
\endproofatend

\begin{proposition}[\bfseries Convex Analysis]~\label{prop:conv16}\newline
Assume that $f$ is convex.  
Define $f^\star \defin \min_{\theta \in \Theta} f(\theta)$ and $\delta \! \defin \! \frac{1}{T}$. When the surrogates~$g_n^t$ in 
Algorithm~\ref{alg:generic_incremental} are majorant and in $\S_{L,\rho}(f^t,\theta_{n-1})$ with $\rho \! \geq \! L$, we have
\begincondeq
\E[ f(\theta_n)-f^\star] \leq \frac{L\|\theta^\star-\theta_0\|_2^2}{2\delta n}  ~~~\text{for all}~n\geq 1. \label{eq:incr:rate}
\endcondeq
Assume now that $f$ is $\mu$-strongly convex. For all $n \! \geq \!1$, 
\begincondeq
\left\{ \begin{array}{lcr}
    \!\!\E[ \|\theta^\star \! -\! \theta_n\|_2^2]\!\! & \!\! \leq \!\! & \!\! \left((1\!-\!\delta)\!+\!\delta \frac{L}{\rho+\mu}\right)^n \|\theta^\star-\theta_0\|_2^2 \\
    \!\!\E[ f(\theta_n) \! - \! f^\star] \!\! & \!\! \leq \!\! & \!\! \left((1\!-\!\delta) \!+\!  \delta \frac{ L  }{\rho+\mu}\right)^{n-1}\frac{L\|\theta^\star-\theta_0\|_2^2}{2}
    \end{array} \right. \!.  \label{eq:incr:ratemu}
\endcondeq

\end{proposition}
\proofatend
The almost sure convergence of $f(\theta_n)$ was shown in Proposition~\ref{prop:conv13}.
We now prove the proposition in several steps and start with some preliminaries.

\proofstep{Preliminaries}
Let us denote by $\kappa_{n-1}^t$ the point in $\Theta$ such that $g_n^t$ is in $\S_{L,\rho}(f^t,\kappa_{n-1}^t)$. 
We remark that such points are drawn according to the following conditional probability distribution:
$$\PPP(\kappa_{n-1}^t = \theta_{n-1} | \theta_{n-1}) = \delta ~~\text{and}~~ \PPP(\kappa_{n-1}^t = \kappa_{n-2}^t | \theta_{n-1})=1-\delta,$$ 
where $\delta \defin 1/T$. Thus we have for all $t$ in $\{1,\ldots,T\}$ and all $n \geq 1$,
\begin{equation}
   \E[\|\theta^\star-\kappa_{n-1}^t\|_2^2]=   \E[\E[\|\theta^\star-\kappa_{n-1}^t \|_2^2| \theta_{n-1}]] =  \delta\E[\|\theta^\star-\theta_{n-1}\|_2^2] + (1-\delta) \E[\|\theta^\star-\kappa_{n-2}^t\|_2^2].\label{eq:incr:tmp2}
\end{equation}
The other relation we need is an extension of Lemma~\ref{lemma:basic} to the incremental setting. For all $\theta$ in~$\Theta$, we have
\begin{equation}
f(\theta_n) \leq f(\theta) + \frac{1}{T}\sum_{t=1}^T \left( \frac{L}{2}\|\theta-\kappa_{n-1}^t\|_2^2 -\frac{\rho}{2}\|\theta-\theta_n\|_2^2  \right). \label{eq:incr:tmp4}
\end{equation}
The proof of this relation is similar to that of Lemma~\ref{lemma:basic}, exploiting to $\rho$-strong convexity of ${\bar g}_n$.  We
can now study the first part of the proposition.

\proofstep{Monotonic decrease of $\E[f(\theta_n)]$}
Note that $\E[g_{n-1}^{{\hat t}_n}(\theta_{n-1})]= \E[\E[g_{n-1}^{{\hat t}_n}(\theta_{n-1})| \theta_{n-1}]] = \E[\barg_{n-1}(\theta_{n-1})]$.
Applying this relation to Eq.~(\ref{eq:incr:tmp}), we have
\begin{displaymath}
   \E[f(\theta_n)] \leq \E[\barg_{n}(\theta_n)] \leq  \E[f^{{\hat t}_n}(\theta_{n-1})]= \E[\E[f^{{\hat t}_n}(\theta_{n-1})|\theta_{n-1}]] =  \E[f(\theta_{n-1})],
\end{displaymath}
where the first inequality comes from the fact that $f \leq \barg_n$ (see proof of Proposition~\ref{prop:conv13}).

\proofstep{Non-strongly convex case ($\rho=L$); convergence rate}
Denote by $A_n \defin \E[\frac{1}{2T}\sum_{t=1}^T \|\theta^\star-\kappa_n^t\|_2^2]$ and by $\xi_n \defin \frac{1}{2}\E[\|\theta^\star-\theta_n\|_2^2]$. Then, we have from~(\ref{eq:incr:tmp4}) and by taking the expectation
\begin{displaymath}
    \E[f(\theta_n)-f^\star] \leq LA_{n-1} - L \xi_n.
\end{displaymath}
It follows from~(\ref{eq:incr:tmp2}) that $A_n = \delta\xi_{n} + (1-\delta) A_{n-1}$ and thus
\begin{displaymath}
    \E[f(\theta_n)-f^\star] \leq \frac{L}{\delta}(A_{n-1} - A_n).
\end{displaymath}
By summing the above inequalities, and using the fact that $\E[f(\theta_n)-f^\star]$ is monotonically decreasing, we obtain that
\begin{displaymath}
    n\E[f(\theta_n)-f^\star] \leq \frac{L A_0}{\delta},
\end{displaymath}
leading to the convergence rate of Eq.~(\ref{eq:incr:rate}), since $A_0 = \frac{1}{2}\|\theta^\star-\theta_0\|_2^2$.

\proofstep{$\mu$-strongly convex case}
Suppose now that $f$ is $\mu$-strongly convex.
We will prove the proposition by induction. Assume that for some $n \geq 1$, we have $A_{n-1} \leq \beta^{n-1}\xi_0$ with $\beta\defin \frac{ (1-\delta)(\rho+\mu)+\delta L}{\rho+\mu}$.
We have from~(\ref{eq:incr:tmp4}) and the second-order growth condition of Lemma~\ref{lemma:second}
\begin{displaymath}
  \mu \xi_n \leq \E[f(\theta_n) - f^\star] \leq L A_{n-1} - \rho\xi_n,
\end{displaymath}
which is true for all $n \geq 1$. Combining the previous inequality, Eq.~(\ref{eq:incr:tmp2}), and the induction hypothesis, we have
\begin{displaymath}
  A_n = \delta \xi_n + (1-\delta) A_{n-1} \leq \left( \frac{\delta L}{\mu+\rho} + (1-\delta) \right) \beta^{n-1} \xi_0 = \beta^{n}\xi_0.
\end{displaymath}
Since we have $A_0=\xi_0$, the induction hypothesis is true for all $n \geq 0$. Since we have from~(\ref{eq:incr:tmp4}) $\xi_n \leq \frac{L}{\rho+\mu}A_{n-1}$, and $\E[f(\theta_n) - f^\star] \leq L A_{n-1}$, we finally have shown the desired convergence rate~(\ref{eq:incr:ratemu}).
\endproofatend

Interestingly, the proof and the convergence rates of Proposition
\ref{prop:conv16} are similar to those of the block coordinate scheme.  For
both schemes, the current iterate~$\theta_n$ can be shown to be the minimizer
of an approximate surrogate function which splits into different parts.  Each
iteration randomly picks up one part, and updates it. 
Like SAG or SDCA, we obtain linear convergence for strongly convex
functions~$f$, even though the upper bounds obtained for SAG and SDCA are
better than ours.

It is also worth noticing that for smooth unconstrained problems,
MISO and SAG yield different, but related, update rules.  Assume for instance
that ``Lipschitz gradient surrogates'' are used. At iteration~$n$ of MISO, each
function $g_n^t$ is a surrogate of~$f^t$ near some~$\kappa_{n-1}^t$. The update
rule of MISO can be shown to be
 $\theta_n \!\leftarrow
\!\frac{1}{T}\sum_{t=1}^T\!  \kappa_{n-1}^t \!-\!  \frac{1}{TL} \sum_{t=1}^T
\!\nabla f^t(\kappa_{n-1}^t)$; in comparison, the update rule of SAG is 
$\theta_n \!\leftarrow \! \theta_{n-1} \!-\!  \frac{1}{TL} \sum_{t=1}^T \!\nabla f^t(\kappa_{n-1}^t)$.

The next section complements the theoretical analysis of the scheme MISO by
numerical experiments and practical implementation heuristics.

\condvspace

\section{Experiments}
In this section, we show that MISO is efficient
for solving large-scale machine learning problems.

\condvspacesmall
\subsection{Experimental Setting}
\condvspacesmall
We consider $\ell_2$- and $\ell_1$- logistic regression without intercept, and
denote by~$m$ the number of samples and by~$p$ the number of features. The corresponding
optimization problem can be written
\begin{equation}
   \min_{\theta \in \Real^p} \frac{1}{m} \sum_{t=1}^m \log(1+e^{-y_t \x^{t\top}\theta}) + \lambda \psi(\theta), \label{eq:logistic}
\condvspace
\end{equation}
where the regularizer $\psi$ is either the $\ell_1$- or squared $\ell_2$-norm.  The $y_t$'s are
in~$\{-1,+1\}$ and the $\x^t$'s are vectors in $\Real^p$ with unit
$\ell_2$-norm.
We use four classical datasets described in the following table:

\ifthenelse{\isundefined{\supplemental}}{}{
\begin{center}
}
\begin{tabular}{|l|c|c|c|c|}
\hline
name & $m$ & $p$ & storage & size (GB) \\ 
\hline
\textsf{alpha} & $250\,000$ & $500$ & dense & $1$\\
\hline
\textsf{rcv1} & $781\,265$ & $47\,152$ & sparse & $0.95$ \\
\hline
\textsf{covtype} & $581\,012$ & $54$ & dense & $0.11$ \\
\hline
\textsf{ocr} & $ 2\,500\,000$ & $1\,155$ & dense & $23.1$ \\
\hline
\end{tabular}
\ifthenelse{\isundefined{\supplemental}}{}{
\end{center}
}

Three datasets, \textsf{alpha}, \textsf{rcv1} and \textsf{ocr} were obtained from the 2008 Pascal large scale learning challenge.\footnote{\url{http://largescale.ml.tu-berlin.de}.} 
The dataset~\textsf{covtype} is available from the LIBSVM website.\footnote{\url{http://www.csie.ntu.edu.tw/~cjlin/libsvm/}.}
We have chosen to test several software packages including
LIBLINEAR 1.93~\citep{fan2}, the
ASGD and SGD implementations of L.
Bottou (version 2)\footnote{\url{http://leon.bottou.org/projects/sgd}.}, an implementation of SAG 
kindly provided to us by the authors of \citet{leroux}, the FISTA method of~\citet{beck} implemented in the SPAMS
toolbox\footnote{\myvspace{0.1}\url{http://spams-devel.gforge.inria.fr/}.}, and SHOTGUN~\citep{bradley2}.
All these softwares are coded in C++ and were compiled using \textsf{gcc}.
Experiments were run on a single core of a 2.00GHz Intel Xeon CPU E5-2650 using
$64$GB of RAM, and all computations were done in double precision. All the
timings reported do not include data loading into memory. Note that we could not run the
softwares SPAMS, LIBLINEAR and SHOTGUN on the dataset~\textsf{ocr} because
of index overflow issues.

\subsection{On Implementing MISO}
The objective function~(\ref{eq:logistic})
splits into $m$ components $f^t: \theta \mapsto
\log(1+e^{-y_t \x^{t\top}\theta}) + \lambda \psi(\theta)$. 
It is thus natural to consider the incremental scheme of Section~\ref{sec:incremental}
together with the proximal gradient surrogates of Section~\ref{subsec:surrogates}.
Concretely, we build at iteration~$n$ of MISO a surrogate $g_n^\hattn$ of $f^\hattn$ as follows:
$g_n^\hattn: \theta \mapsto l^\hattn(\theta_{n-1}) + \nabla
l^\hattn(\theta_{n-1})^\top(\theta-\theta_{n-1}) +
\frac{L}{2}\|\theta-\theta_{n-1}\|_2^2 + \lambda \psi(\theta)$,
where~$l^t$ is the logistic function $\theta \mapsto \log(1+e^{-y_t \x^{t\top}\theta})$.
 
After removing the dependency over $n$ to simplify the notation, all the surrogates can be rewritten as
$g^t: \theta \mapsto a_{t} + \z^{t \top}\theta + \frac{L}{2}\|\theta\|_2^2 + \lambda \psi(\theta)$,
where~$a_t$ is a constant and $\z^t$ is a vector in $\Real^p$.  Therefore, all
surrogates can be ``summarized'' by the pair $(a_t,\z^t)$, quantities which we keep into
memory during the optimization. Then, finding the estimate $\theta_n$ amounts to minimizing a function of the form
$\theta \mapsto \barz_n^\top\theta + \frac{L}{2}\|\theta\|_2^2 + \lambda \psi(\theta)$, where~$\barz_n$ is the
average value of the quantities~$\z^t$ at iteration $n$. 
It is then easy to see that obtaining~$\barz_{n+1}$ from~$\barz_n$ can be done in $O(p)$ operations with the following update:
$\barz_{n+1} \leftarrow \barz_n + (\z^{\hattn}_{\text{new}} - \z^{\hattn}_{\text{old}})/m$.

One issue is that building the surrogates~$g^t$ requires choosing some constant~$L$. 
An upper bound on the Lipschitz constants of the
gradients $\nabla l^t$ could be used here. However, we have observed that significantly faster
convergence could be achieved by using a smaller value, probably because a
local Lipschitz constant may be better adapted than a global one. By studying
the proof of Proposition~\ref{prop:conv16}, we notice indeed that our convergence
rates can be obtained without majorant surrogates, when we simply have:
$\E[f^t(\theta_n)] \leq \E[g_n^t(\theta_n)]$ for all~$t$ and~$n$.
This motivates the following heuristics:\\
\mybullet MISO1: start by performing one pass over $\eta\!=\!5\%$ of the data to select a constant $L'$ yielding the smallest decrease of the objective, and set $L=L' \eta$;\\
\mybullet MISO2: in addition to MISO1, check the inequalities $f^{{\hat
t}_n}(\theta_{n-1}) \! \leq \! g_{n-1}^{{\hat t}_n}(\theta_{n-1})$ during the
optimization. After each pass over the data, if the rate of satisfied
inequalities drops below $50\%$, double the value of $L$.

Following these strategies, we have implemented the scheme MISO in C++. The
resulting software package will be publicly released with an open source
license.

\subsection{$\ell_2$-Regularized Logistic Regression}\label{subsec:l2log}
We compare LIBLINEAR, FISTA, SAG, ASGD, SGD, MISO1, MISO2 and MISO2 with
$T=1000$ blocks (grouping some observations into minibatches).  LIBLINEAR was
run using the option \textsf{-s 0 -e 0.000001}.  The implementation of SAG includes
a heuristic line search in the same spirit as MISO2, introduced by~\citet{leroux}.
Every method was stopped after~$50$ passes over the
data.
We considered three regularization regimes, \textsf{high} 
$(\lambda\!=\!10^{-3})$, \textsf{medium} $(\lambda\!=\!10^{-5})$ and \textsf{low}
$(\lambda\!=\!10^{-7})$. We present in
Figure~\ref{fig:exp1} the values of the objective function during the
optimization for the regime \textsf{medium}, both in terms of passes over the
data and training time. The regimes \textsf{low} and \textsf{high} are provided
as supplemental material only.
Note that to reduce the memory load, we used a minibatch strategy for the
dataset \textsf{rcv1} with $T=10\,000$ blocks.

Overall, there is no clear winner from this experiment, and the preference for
an algorithm depends on the dataset, the required precision, or the
regularization level. The best methods seem to be consistently MISO, ASGD and
SAG and the slowest one FISTA.  Note that this apparently mixed result is a
significant achievement.  We have indeed focused on state-of-the-art solvers,
which already significantly outperform a large number of other
baselines~\citep[see][]{bottou5,fan2,leroux}.

\ifthenelse{\isundefined{\supplemental}}{
\def\widthfigleft{0.48}
\def\widthfigright{0.51}
\begin{figure}[t]
}{
\def\widthfigleft{0.45}
\def\widthfigright{0.48}
\begin{figure}[hbtp]
}
\centering
\includegraphics[width=\widthfigleft\linewidth]{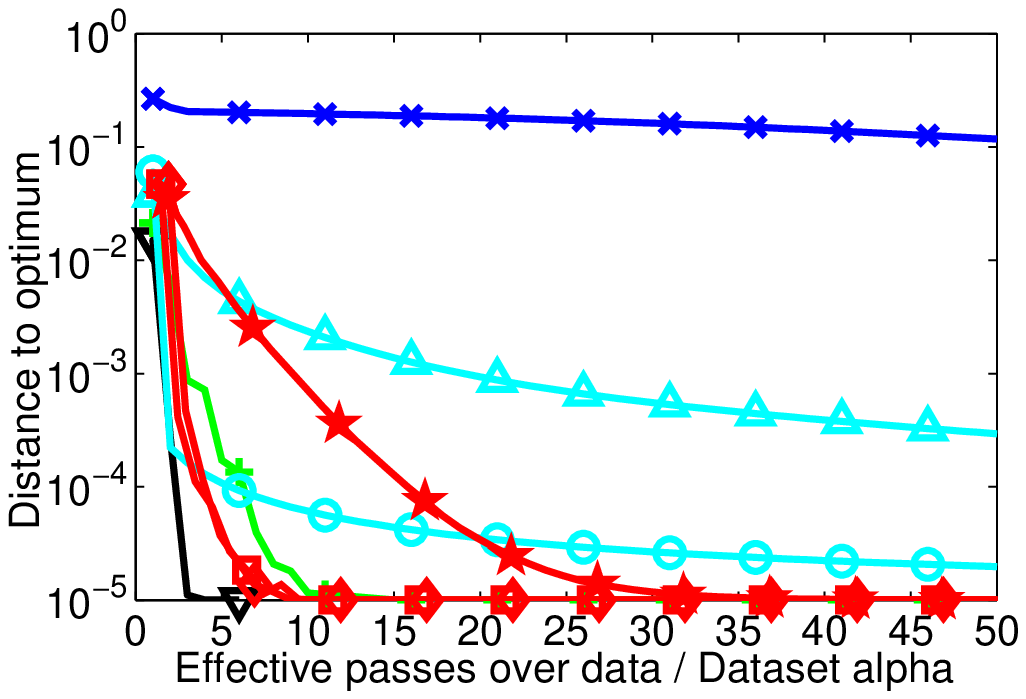}\hfill
\includegraphics[width=\widthfigright\linewidth]{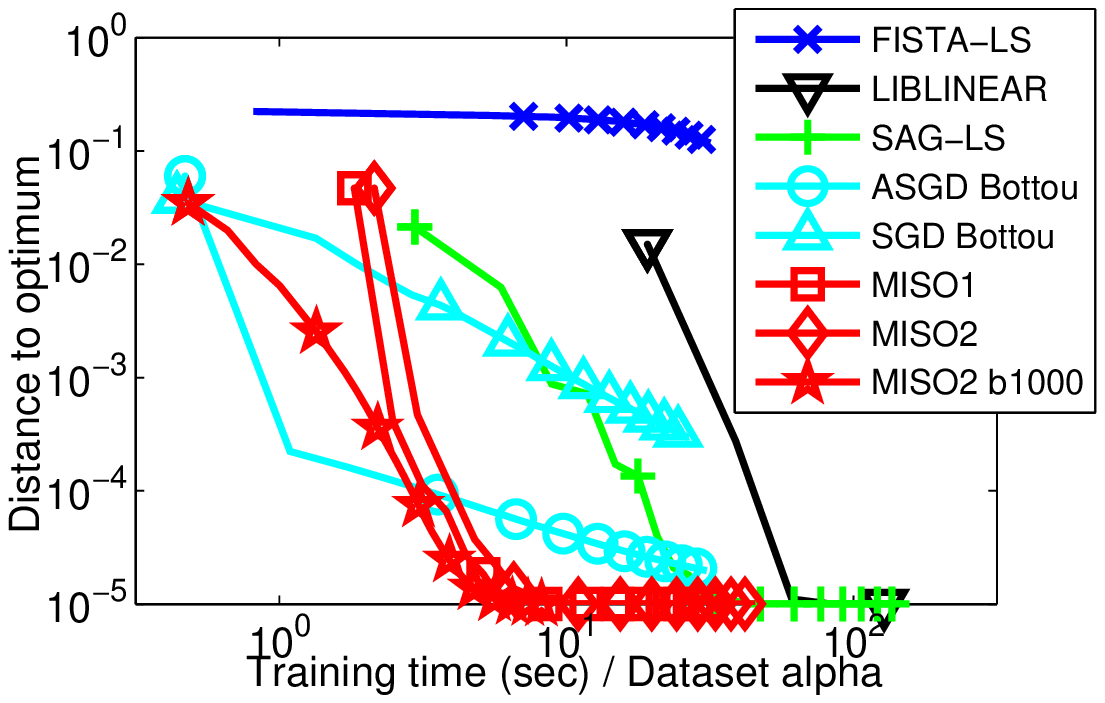}\\
\includegraphics[width=\widthfigleft\linewidth]{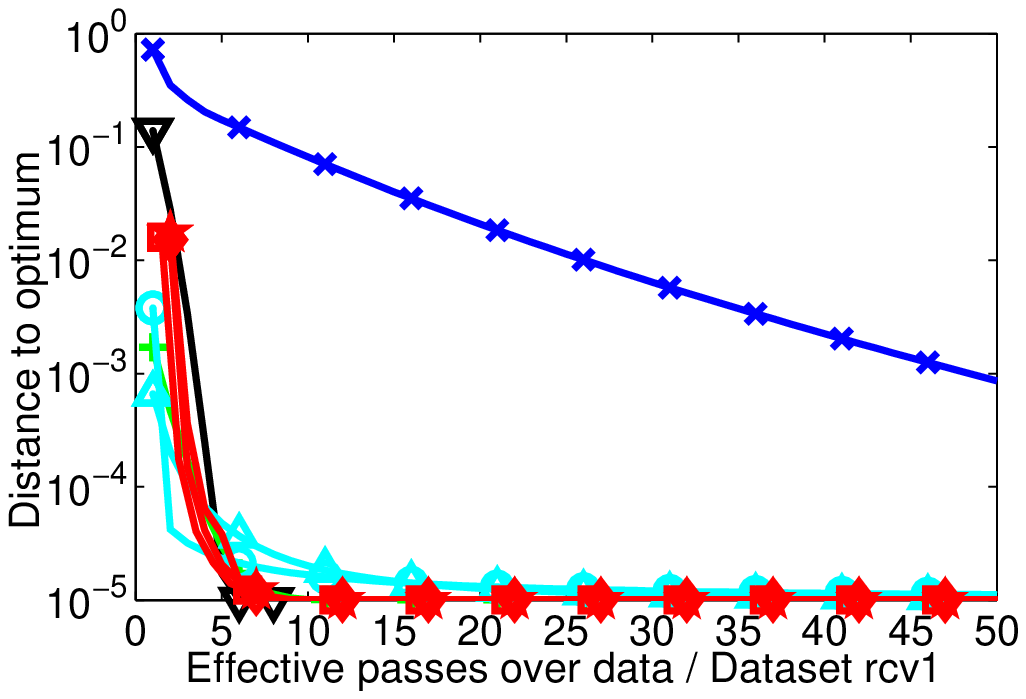}\hfill
\includegraphics[width=\widthfigright\linewidth]{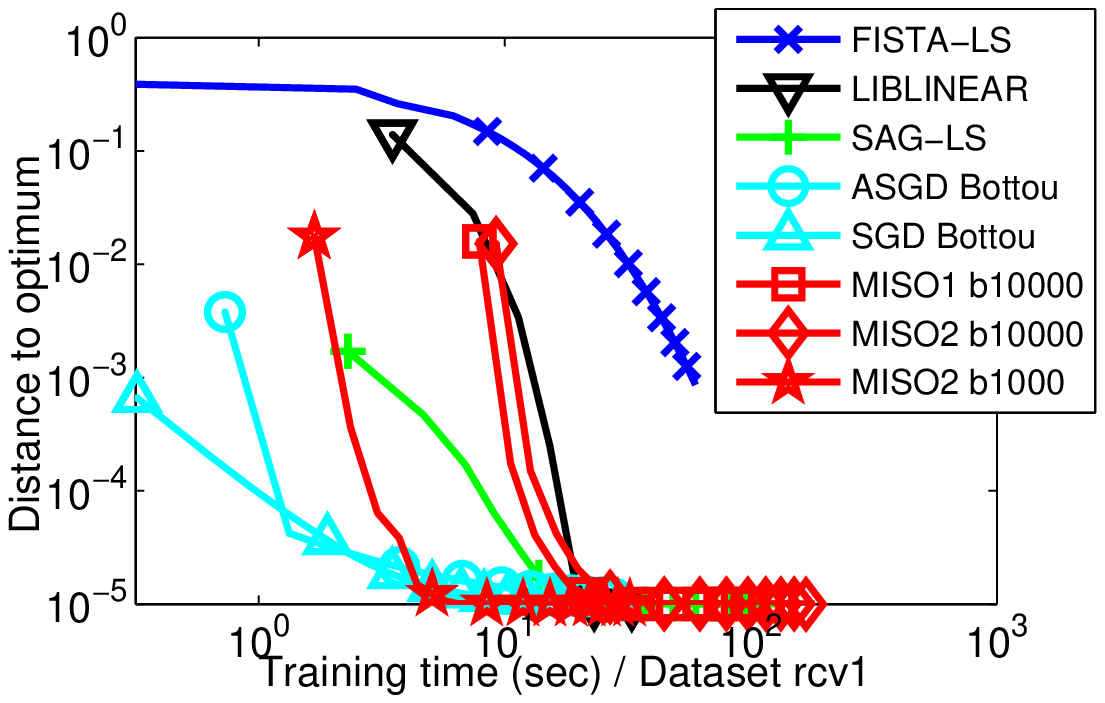}\\
\includegraphics[width=\widthfigleft\linewidth]{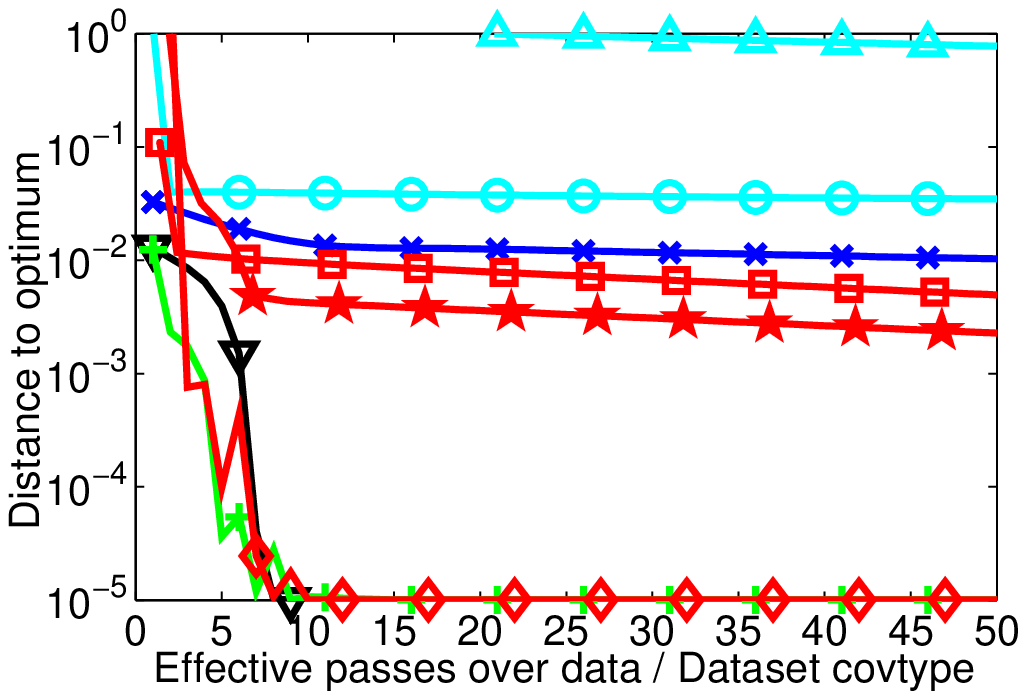}\hfill
\includegraphics[width=\widthfigright\linewidth]{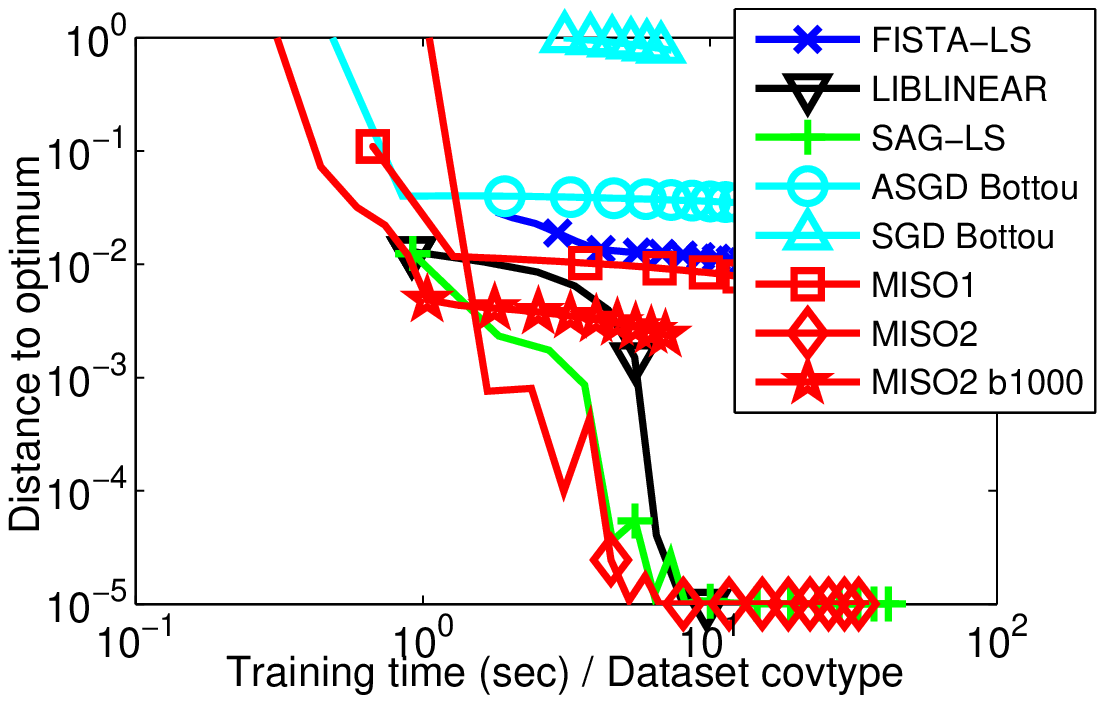}\\
\includegraphics[width=\widthfigleft\linewidth]{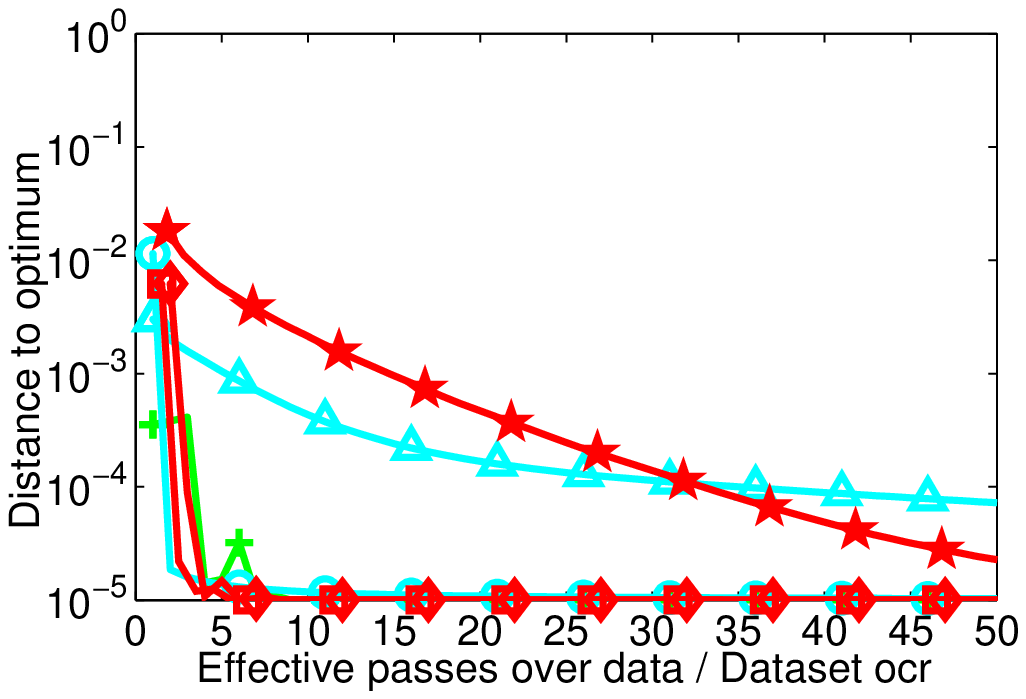}\hfill
\includegraphics[width=\widthfigright\linewidth]{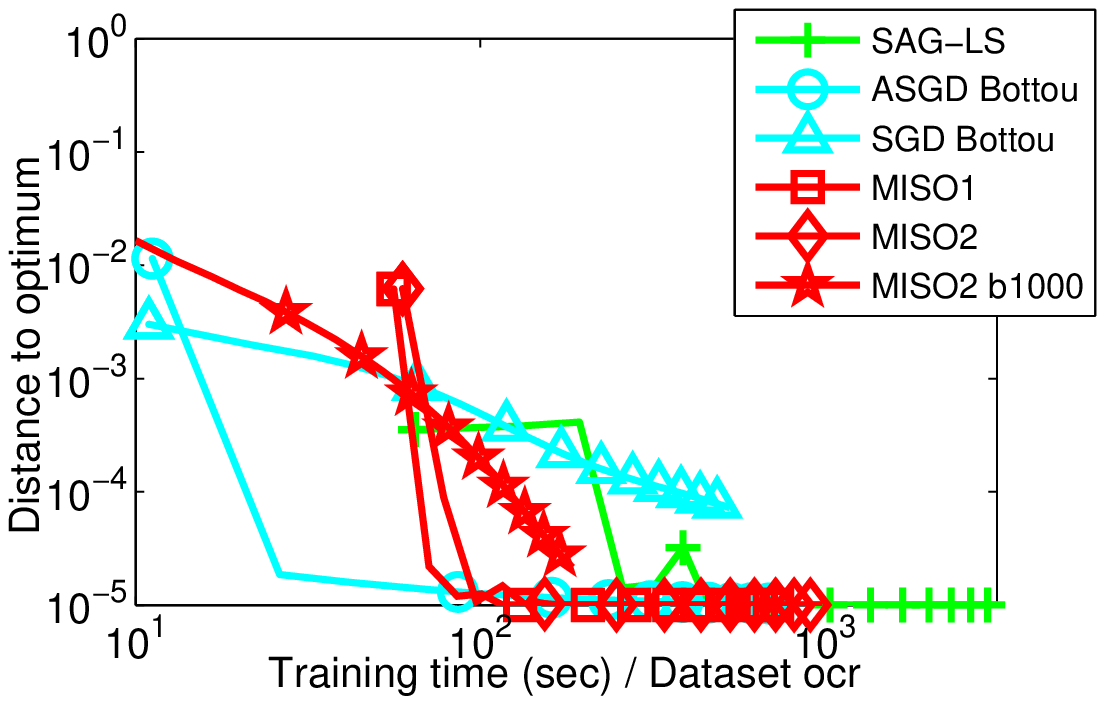}
\myvspace{0.5}
\caption{Results for $\ell_2$-logistic regression with $\lambda\!=\!10^{-5}$.}
\myvspace{0.3}
\label{fig:exp1}
\end{figure}

\myvspace{0.1}
\subsection{$\ell_1$-Regularized Logistic Regression}
Since SAG, SGD and ASGD cannot deal with $\ell_1$-regularization, we compare
here LIBLINEAR, FISTA, SHOTGUN and MISO. We use for LIBLINEAR the option
\textsf{-s 6 -e 0.000001}. We proceed as in Section~\ref{subsec:l2log},
considering three regularization regimes yielding different sparsity
levels. We report the results for one of them in Figure~\ref{fig:exp2} and provide
the rest as supplemental material. In this experiment, our method outperforms
other competitors, except LIBLINEAR on the dataset \textsf{rcv1} when 
a high precision is required (and the regularization is low).
We also remark that a low precision solution is often achieved
quickly using the minibatch scheme (MISO2 b1000), but this
strategy is outperformed by MISO1 and MISO2 for high precisions.

\ifthenelse{\isundefined{\supplemental}}{
\begin{figure}[t]
}{
\begin{figure}[hbtp]
}
\centering
\includegraphics[width=\widthfigleft\linewidth]{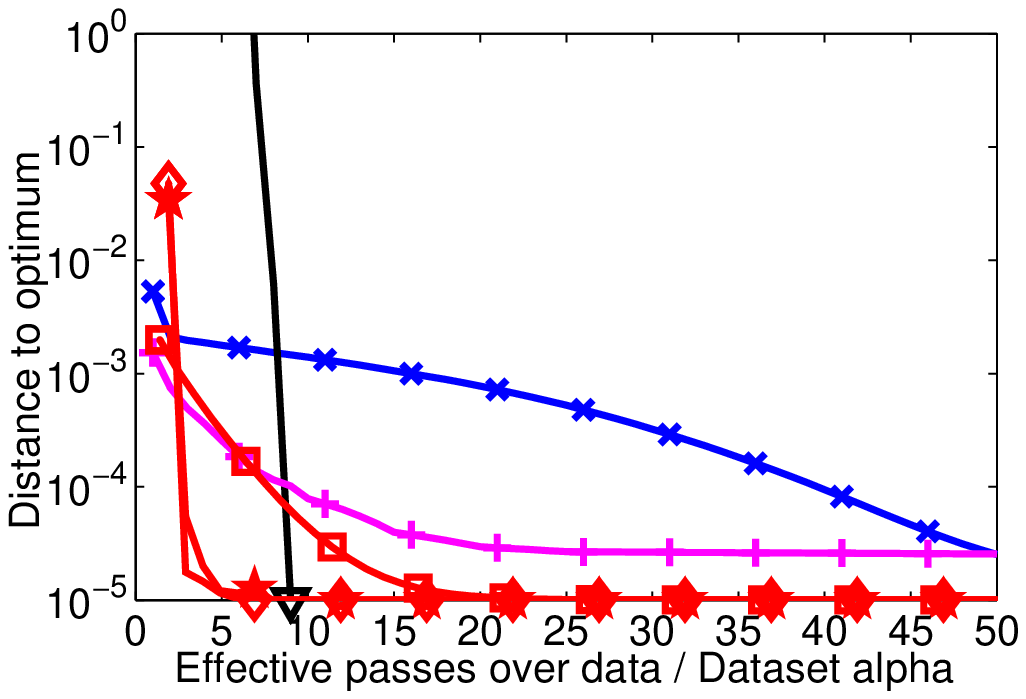}\hfill
\includegraphics[width=\widthfigright\linewidth]{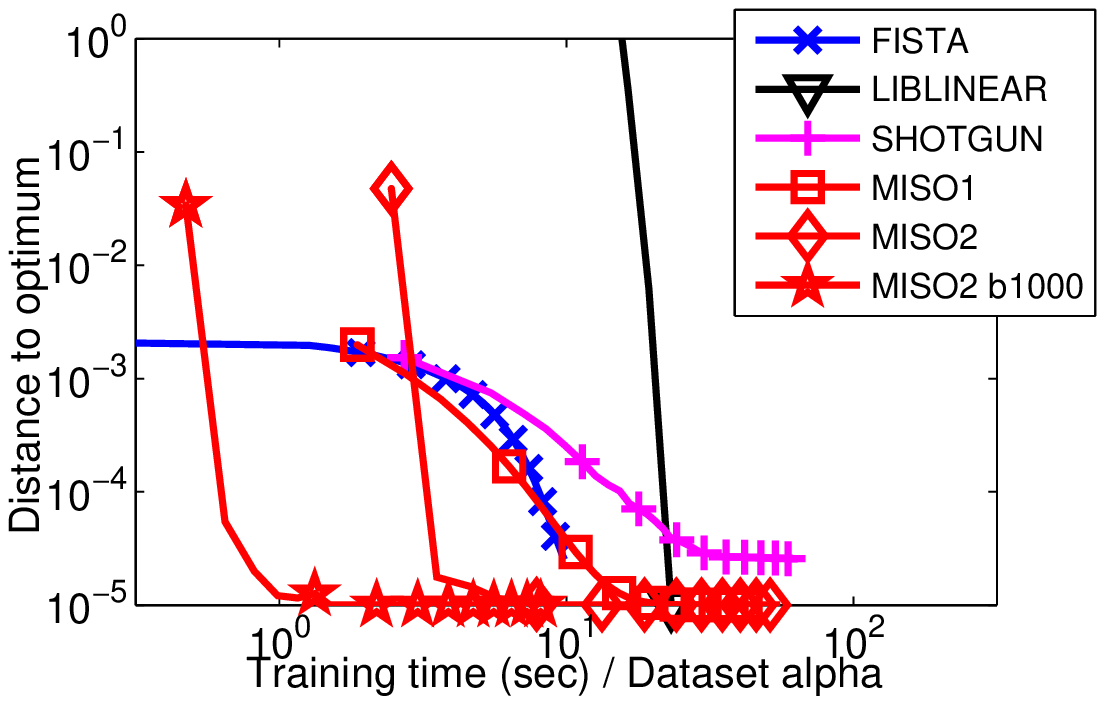}\\
\includegraphics[width=\widthfigleft\linewidth]{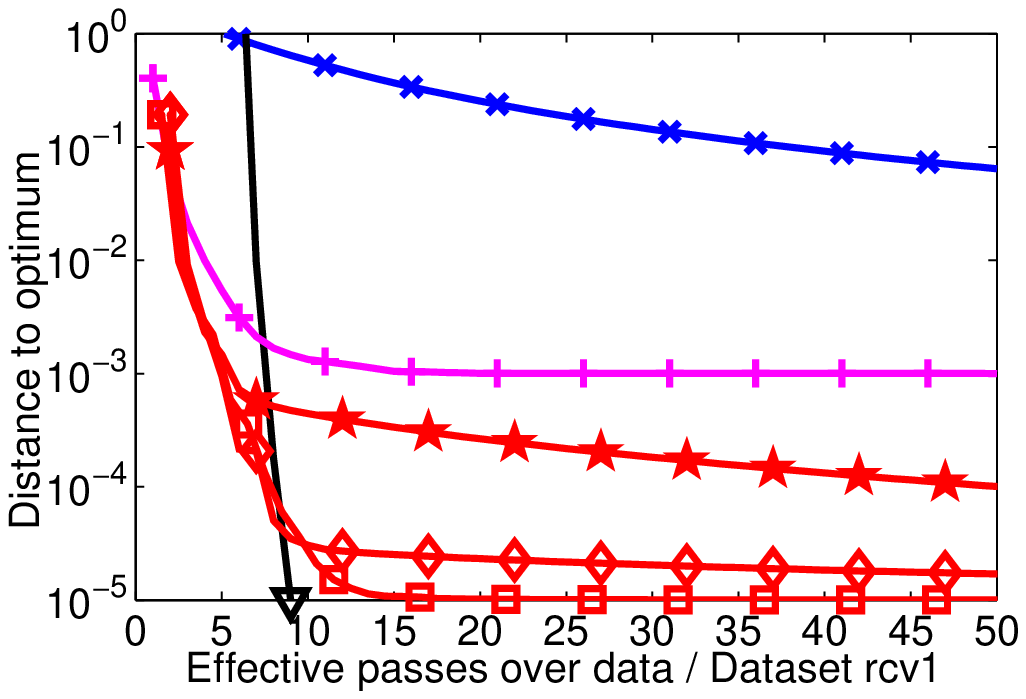}\hfill
\includegraphics[width=\widthfigright\linewidth]{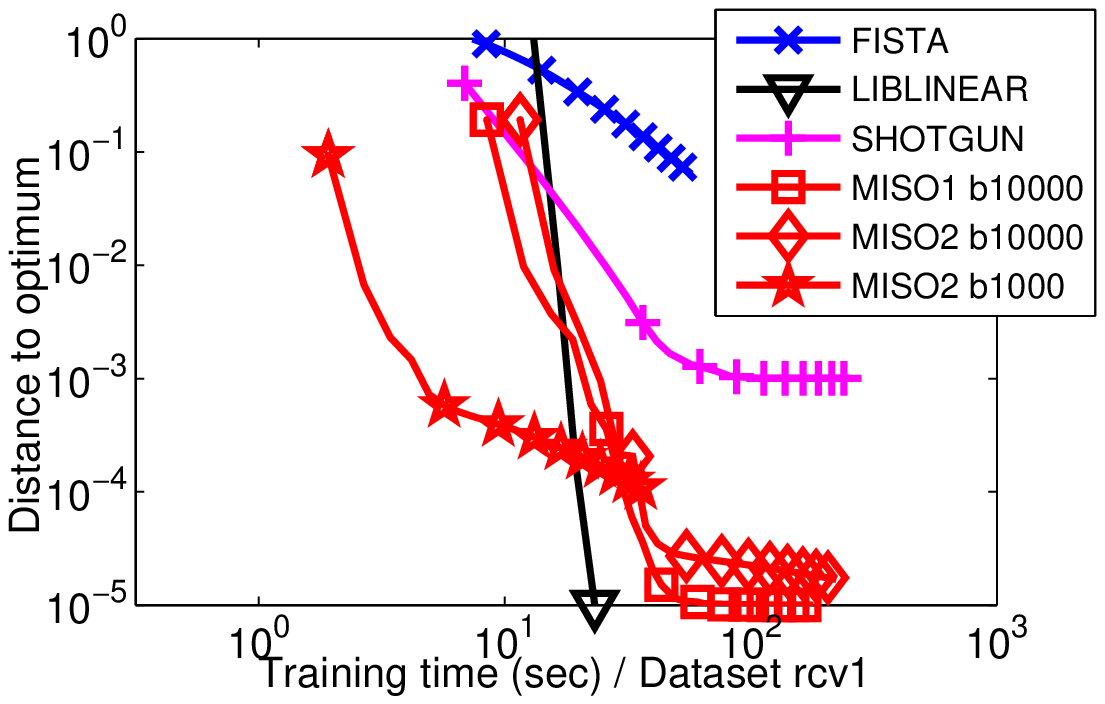}\\
\includegraphics[width=\widthfigleft\linewidth]{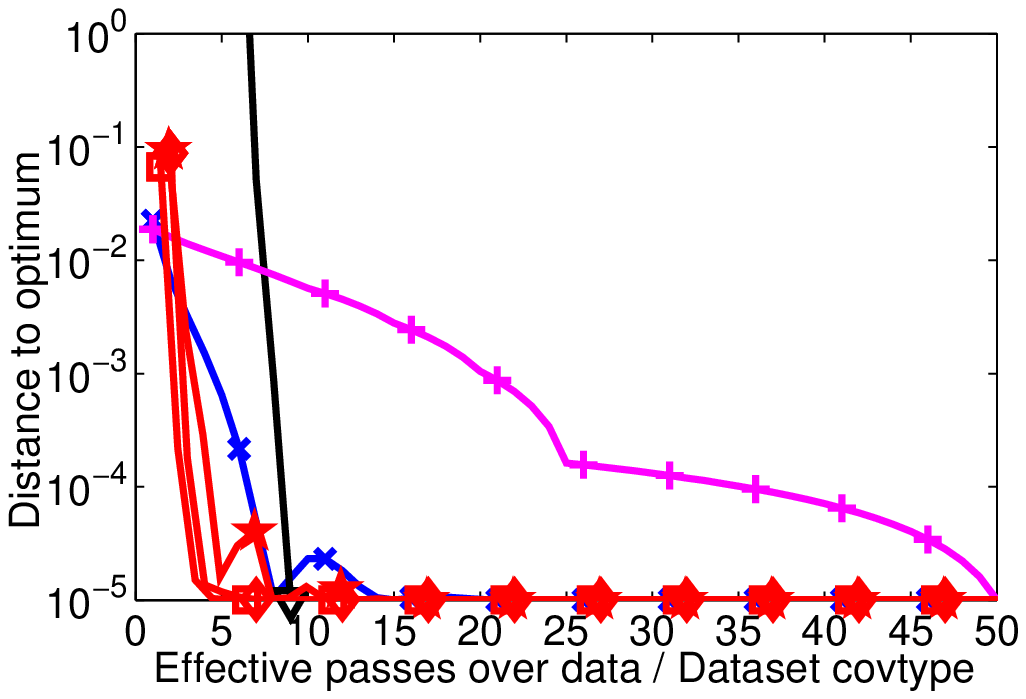}\hfill
\includegraphics[width=\widthfigright\linewidth]{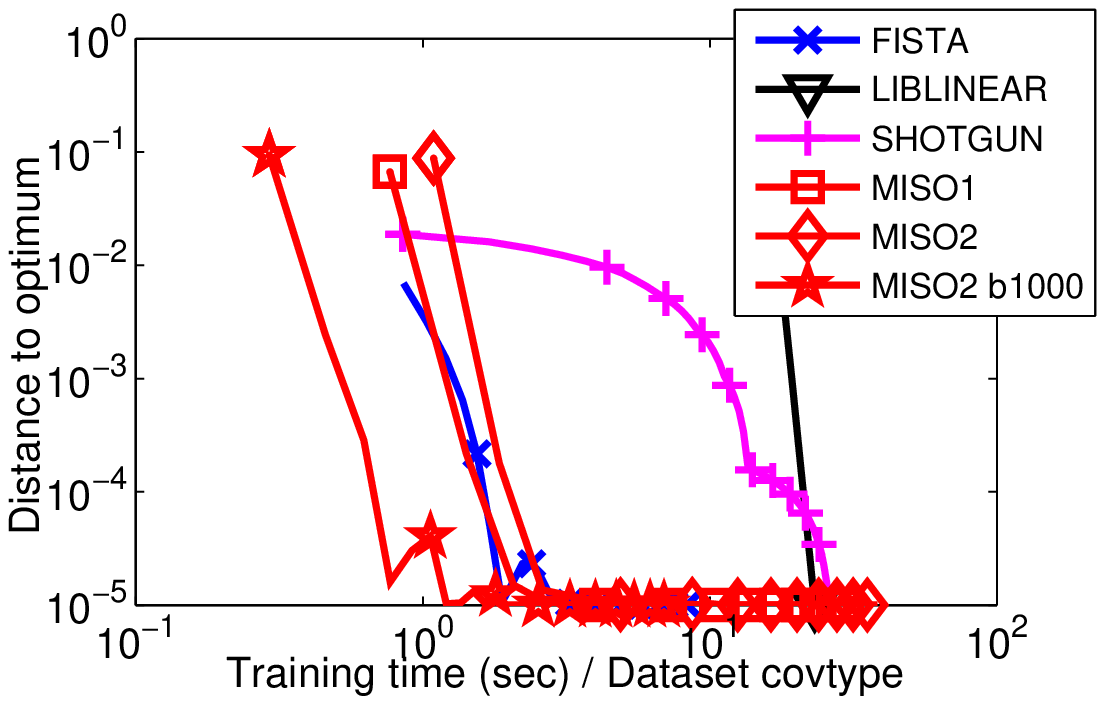}\\
\includegraphics[width=\widthfigleft\linewidth]{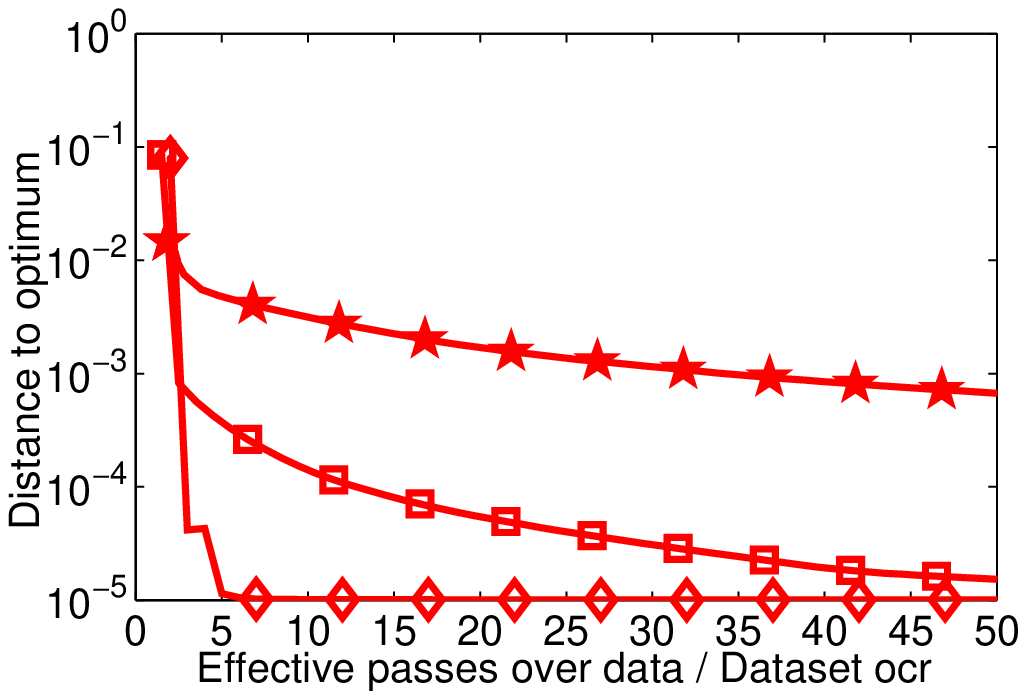}\hfill
\includegraphics[width=\widthfigright\linewidth]{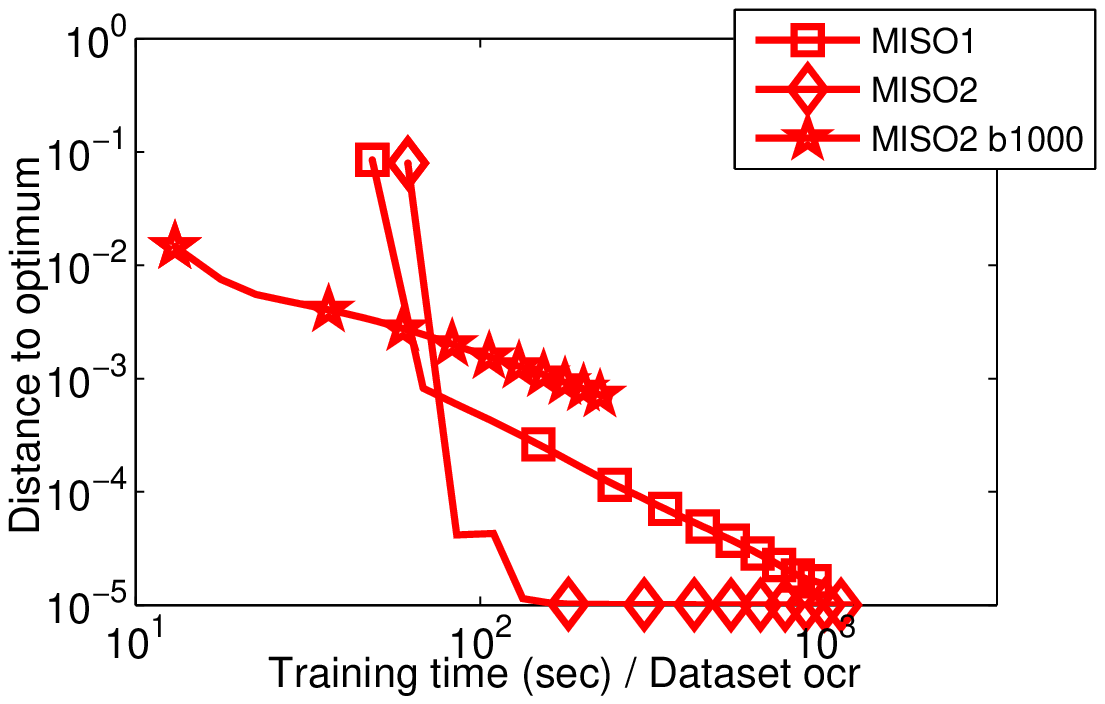}
\myvspace{0.5}
\caption{Benchmarks for $\ell_1$-logistic regression. $\lambda$ was chosen to obtain a solution with $10\%$ nonzero coefficients.}
\label{fig:exp2}
\myvspace{0.4}
\end{figure}

\myvspace{0.1}

\section{Conclusion}
In this paper, we have introduced a flexible optimization framework based
on the computation of ``surrogate functions''. We have revisited numerous
schemes and discovered new ones. For each of them, we have studied
convergence guarantees for non-convex problems and convergence rates for
convex ones. Our methodology led us in particular to the design of an
incremental algorithm, which has theoretical properties and empirical
performance matching state-of-the-art solvers for large-scale machine
learning problems.

In the future, we are planning to study fully stochastic or memoryless variants
of our framework. As in the incremental setting, it consists of
drawing a single training point at each iteration, but the algorithm does not
keep track of all past information. This is essentially a strategy followed
by~\citet{neal} and~\citet{mairal7} in the respective contexts of EM and sparse
coding algorithms. This would be particularly important for processing sparse
datasets with a large number of features, where storing (dense) information
about the past surrogates is cumbersome.

\vspace*{-0.05cm}
\section*{Acknowledgments}
\vspace*{-0.01cm}
JM would like to thank Zaid Harchaoui, Francis Bach, Simon
Lacoste-Julien, Mark Schmidt, Martin Jaggi, and Bin Yu for fruitful
discussions. This work was supported by 
Quaero, (funded by OSEO, the French state agency for innovation), by the Gargantua project (program Mastodons - CNRS), and by the
Center for Science of Information (CSoI), an NSF Science and Technology
Center, under grant agreement CCF-0939370.

{\small 
\bibliography{abbrev,main}
\bibliographystyle{icml2013}
}

\clearpage
\appendix
\onecolumn
\icmltitle{Supplementary Material \\ Optimization with First-Order Surrogate Functions}
\vskip 0.3in

 \paragraph{Outline.} 
   In Appendix~\ref{appendix:background}, we present simple mathematical definitions.
   Appendix~\ref{appendix:maths} contains useful mathematical
   results, which are used in the paper. In Appendix~\ref{appendix:mechanisms},
   we present various mechanisms to build first-order surrogate functions; it is
   in fact a more rigorous version of Section~\ref{subsec:surrogates}, where 
   all claims are proved. In Appendix~\ref{appendix:schemes}, we present the
   block Frank-Wolfe optimization scheme. Finally, all proofs of propositions
   are given in Appendix~\ref{appendix:proofs}, and
   Appendix~\ref{appendix:exps} contains additional experimental results.

\section{Mathematical Background}\label{appendix:background}
For self-containedness purposes, we introduce in this section some mathematical
definitions. Most of them can be found in classical textbooks on
optimization~\citep[e.g.,][]{bertsekas,boyd,borwein,nocedal,nesterov4}.
\begin{definition}[\bfseries Directional Derivative]~\label{def:derivative}\\
Let us consider a function $f: \Theta \subseteq \Real^p \to \Real$, where
$\Theta$ is a convex set, and $\theta,\theta'$ be in $\Theta$. When it exists,
the following limit is called the directional derivative of $f$ at $\theta$ in
the direction $\theta'-\theta$:
$$ \nabla f(\theta,\theta'-\theta) \defin \lim_{t \to 0^+}
\frac{f(\theta+t(\theta'-\theta)) - f(\theta)}{t}.$$
When $f$ is differentiable at $\theta$, 
directional derivatives always exist and we have $\nabla
f(\theta,\theta'-\theta)=\nabla f(\theta)^\top (\theta'-\theta)$.
\end{definition}

\begin{definition}[\bfseries Feasible Direction]~\label{def:admissible}\\
Let $\Theta \subseteq \Real^p$ be a convex set and $\theta$ be a point in
$\Theta$. A vector $\z$ in $\Real^p$ is a feasible direction if
$\theta+\z$ is in $\Theta$. In other words, $\z$ can be written as
$\theta'-\theta$, where $\theta'$ is in $\Theta$.
\end{definition}

\begin{definition}[\bfseries Stationary Point]~\label{def:stationary}\\
Let us consider a function $f: \Theta \subseteq \Real^p \to \Real$, where $\Theta$ is a convex set, such that $f$ admits directional derivatives everywhere in $\Theta$ for every feasible direction. Let $\theta$ be a point in $\Theta$. We say that $\theta$ is a stationary point if for all $\theta' \neq \theta$ in $\Theta$,
\begin{equation}
   \nabla f(\theta, \theta'-\theta) \geq 0.\label{eq:stationary}
\end{equation}
When $f$ is differentiable and $\theta$ is in the interior of $\Theta$, this condition reduces to $\nabla f(\theta)=0$. When~$f$ is convex and $\theta$ is also in the interior of $\Theta$, this condition reduces to $0 \in \partial f(\theta)$, where $\partial f$ is the subdifferential of $f$. 
\end{definition}
\begin{proof}
   Let us assume that $\theta$ is a stationary point and $f$ is differentiable
   at $\theta$. Then, for all~$\theta'$ in~$\Theta$, 
   $\nabla f(\theta, \theta'-\theta)= \nabla f(\theta)^\top (\theta'-\theta) \geq 0$. In particular, since $\theta$ is in the interior of $\Theta$, we can find $\theta'$ such that $\theta'-\theta= -\delta\nabla f(\theta)$ for some $\delta > 0$ small enough. Thus, we necessarily have $\nabla f(\theta)=0$. The converse is trivial.

   The equivalence between~(\ref{eq:stationary}) and $0 \in \partial f(\theta)$ when $f$ is convex but non-differentiable can be found in~\citet[][Proposition 3.1.6]{borwein}.
\end{proof}

\begin{definition}[\bfseries Lipschitz Continuity]~\label{def:lipschitz}\\
A function $f: \Theta \subseteq \Real^p \to \Real$ is called Lipschitz if there exists a constant $L>0$ such that for all $\theta,\theta'$ in $\Theta$, we have
\begin{displaymath}
   |f(\theta')-f(\theta)| \leq L \|\theta-\theta'\|_2.
\end{displaymath}
In that case, we say that the function is $L$-Lipschitz.
\end{definition}

\begin{definition}[\bfseries Strong Convexity]~\label{def:strong_convexity}\\
Let $\Theta$ be a convex set. A function $f: \Theta \subseteq \Real^p \to \Real$ is called $\mu$-strongly convex when there exists a constant $\mu > 0$ such that for all $\theta'$ in $\Theta$, the function $\theta \mapsto f(\theta)-\frac{\mu}{2}\|\theta-\theta'\|_2^2$ is convex. 
This definition is equivalent to having for all $\alpha$ in $[0,1]$ and $\theta$, $\theta'$ in $\Theta$,
\begin{equation}
  f(\alpha \theta + (1-\alpha)\theta') \leq \alpha f(\theta) + (1-\alpha) f(\theta') - \frac{\mu}{2} \alpha (1-\alpha)\|\theta-\theta'\|_2^2.\label{eq:strong_convexity}
\end{equation}
Note that the value $\mu=0$ leads to the classical definition of convex functions.
\end{definition}
\begin{proof}
Let us consider $\theta'$ in $\Theta$ and define the function $g: \theta \mapsto f(\theta)-\frac{\mu}{2}\|\theta-\theta'\|_2^2$.
This function is convex if and only if for all $\alpha$ in $[0,1]$, we have
\begin{displaymath}
   g(\alpha \theta + (1-\alpha)\theta') \leq \alpha g(\theta) + (1-\alpha) g(\theta').
\end{displaymath}
In other words, if and only if 
\begin{displaymath}
   f(\alpha \theta + (1-\alpha)\theta') - \frac{\mu}{2}\alpha^2\|\theta-\theta'\|_2^2 \leq \alpha\left(f(\theta) - \frac{\mu}{2}\|\theta-\theta'\|_2^2\right) + (1-\alpha) f(\theta'),
\end{displaymath}
which is equivalent to~(\ref{eq:strong_convexity}).
\end{proof}

\section{Useful Mathematical Results}\label{appendix:maths}
We provide in this section a few propositions and lemmas which are used in
this paper.

\begin{lemma}[\bfseries Convex Surrogate for Functions with Lipschitz Gradient]~\label{lemma:upperlipschitz}\newline
Let $f: \Real^p \to \Real$ be differentiable and $\nabla f$ be $L$-Lipschitz continuous. Then, for all $\theta,\theta'$ in $\Real^p$,
   \begin{equation}
       |f(\theta') - f(\theta) - \nabla f(\theta)^\top (\theta'-\theta)| \leq \frac{L}{2}\|\theta-\theta'\|_2^2.\label{eq:lipschitz}
   \end{equation}
\end{lemma}
\begin{proof}
This lemma is classical~\citep[see][Lemma 1.2.3 and its proof]{nesterov4}.
\end{proof}
Note that Eq.~(\ref{eq:lipschitz}) does not imply the
gradient of a differentiable function~$f$ to be $L$-Lipschitz continuous. 
The equivalence is only true in some cases, as shown in the following lemma.
\begin{lemma}[\bfseries Relation between Quadratic Surrogates and Lipschitz Constants]~\label{lemma:upperlipschitz2}\newline
Let $f: \Real^p \to \Real$ be a differentiable function.
   Assume that for all $\theta,\theta'$ in~$\Real^p$,
   inequality~(\ref{eq:lipschitz}) holds.  Then, $\nabla f$ is $L$-Lipschitz continuous
   when one of the following conditions is true:
   \begin{enumerate}
\setlength\itemindent{25pt}
      \item $f$ is convex;
      \item $f$ is twice differentiable;
      \item $\nabla f$ is Lipschitz continuous (the lemma then provides the Lipschitz constant $L$).
   \end{enumerate}
\end{lemma}
\begin{proof}~\\
   {\bfseries First point:}\\
   a proof of the first point can be found in~\citet[][Theorem 2.1.5]{nesterov4}.

   {\noindent \bfseries Second point:}\\
   To prove the second point, we upper-bound the extremal eigenvalues of the Hessian matrix. Let us fix $\theta$ in $\Real^p$. Since $f$ is twice differentiable at $\theta$, we have for all $\theta'$ in $\Real^p$
    \begin{displaymath}
        \nabla f(\theta') - \nabla f(\theta) = \nabla^2 f (\theta)(\theta'-\theta) + o( \|\theta'-\theta\|_2),
    \end{displaymath}
    and thus
    \begin{equation}
        (\theta'-\theta)^\top(\nabla f(\theta') - \nabla f(\theta)) = (\theta'-\theta)^\top\nabla^2 f (\theta)(\theta'-\theta) + o( \|\theta'-\theta\|_2^2). \label{eq:hessianbound}
    \end{equation}
    Summing twice Eq.~(\ref{eq:lipschitz}) without the absolute values when exchanging the roles of $\theta$ and~$\theta'$ gives
   \begin{displaymath}
      (\theta'-\theta)^\top(\nabla f(\theta') - \nabla f(\theta)) \leq L\|\theta-\theta'\|_2^2.
   \end{displaymath}
   Plugging Eq.~(\ref{eq:hessianbound}) into this inequality yields $\|\nabla^2 f(\theta)\|_2 \leq L$. To conclude, we use a mean value theorem, as done for example in~\citet[][Lemma 1.2.2]{nesterov4}
   \begin{displaymath}
   \begin{split}
      \|\nabla f(\theta') - \nabla f(\theta)\|_2 & = \left\| \int_{t=0}^1 \nabla^2 f(\theta + t (\theta'-\theta)) (\theta'-\theta)dt \right\|_2 \\
                                               & \leq \int_{t=0}^1 \|\nabla^2 f(\theta  + t(\theta'-\theta))\|_2 dt \|\theta'-\theta\|_2 \\
                                               & \leq L \|\theta-\theta'\|_2.
   \end{split}
   \end{displaymath}
   {\noindent \bfseries Third point:}\\
   Proving the third point is more difficult due to the lack of smoothness
   assumptions on~$f$. However, when making the explicit assumption that $\nabla f$ 
   is Lipschitz continuous, we can show that Eq.~(\ref{eq:lipschitz})
   provides us a Lipschitz constant.
   The proof exploits some results from nonsmooth analysis developed
   by~\citet{clarke}.
   We essentially use a mean value theorem for
   multi-dimensional Lipschitz functions~\citep[][Proposition 2.6.5]{clarke},
   exploiting the fact that a Lipschitz function is differentiable almost
   everywhere (Rademacher theorem). This allows us to follow a similar proof as
   for the twice differentiable case.

   More precisely, we have that $\nabla f$ is Lipschitz continuous and thus
   differentiable almost everywhere on $\Theta$.
   Let us call the Hessian matrix  $\nabla^2 f(\theta)$ at a point $\theta$ in
   $\Real^p$, when it exists.  Then, we have at such a point $\|\nabla^2
   f(\theta)\|_2 \leq L$, following the beginning of the second point's proof.
   Then, it turns out that for all $\theta$ in $\Real^p$, the following mean
   value theorem holds for almost all $\theta'$~\citep[see][proof of Proposition 2.6.5]{clarke}: 
   \begin{displaymath}
   \nabla f(\theta') - \nabla f(\theta) =
   \int_{t=0}^1 \nabla^2 f(\theta + t (\theta'-\theta)) (\theta'-\theta)dt.
   \end{displaymath}
   This comes from the fact that for almost all $\theta'$, the intersection of the line
   segment $[\theta,\theta']$ and the set where $\nabla^2 f$ is not defined
   has $0$ one-dimensional measure~\citep[see again][Proposition 2.6.5]{clarke}.
   We therefore have for almost all $\theta'$ (and a fixed
   $\theta$), $\|\nabla f(\theta)-\nabla f(\theta')\|_2 \leq
   L\|\theta-\theta'\|_2$ and the general result comes from a continuity
   argument.
\end{proof}
 
\begin{lemma}[\bfseries Surrogate for Functions with Lipschitz Hessian]~\label{lemma:uppersecond}\newline
   Let $f: \Real^p \to \Real$ be a twice differentiable function with $M$-Lipschitz continuous Hessian. Then for all $\theta,\theta'$ in $\Real^p$,
   \begin{displaymath}
       \left|f(\theta') - f(\theta) - \nabla f(\theta)^\top (\theta'-\theta) - \frac{1}{2}(\theta'-\theta)^\top \nabla^2f(\theta)(\theta'-\theta)\right| \leq \frac{M}{6}\|\theta-\theta'\|_2^3.
   \end{displaymath}
\end{lemma}
\begin{proof}
This is again a classical lemma~\citep[see][Lemma 1.2.4]{nesterov4}.
\end{proof}

\begin{lemma}[\bfseries Lower Surrogate for Strongly Convex Functions]~\label{lemma:lower}\newline
 Let $f: \Real^p \to \Real$ be a $\mu$-strongly convex function. Suppose that $f$ is differentiable, then the following inequality holds for all $\theta,\theta'$ in $\Real^p$:
 \begin{displaymath}
     f(\theta') \geq f(\theta) + \nabla f(\theta)^\top(\theta'-\theta) + \frac{\mu}{2}\|\theta-\theta'\|_2^2.
 \end{displaymath}
\end{lemma}
\begin{proof}
    $\theta' \mapsto f(\theta') - \frac{\mu}{2}\|\theta-\theta'\|_2^2$ is convex and differentiable and is therefore above its tangent at~$\theta$, immediately leading to the desired inequality.
\end{proof}

\begin{lemma}[\bfseries Second-Order Growth Property]~\label{lemma:second}\newline
 Let $f: \Real^p \to \Real$ be a $\mu$-strongly convex function and $\Theta \subseteq \Real^p$ be a convex set.
 Let $\theta^\star$ be the minimizer of $f$ on $\Theta$. Then, the following condition holds for all $\theta$ in $\Theta$:
 \begin{displaymath}
     f(\theta) \geq f(\theta^\star) + \frac{\mu}{2}\|\theta-\theta^\star\|_2^2.
 \end{displaymath}
\end{lemma}
\begin{proof}
   Let us define the function $g: \theta \mapsto f(\theta) - \frac{\mu}{2}\|\theta-\theta^\star\|_2^2$.
   We show that $\theta^\star$ is a minimizer of the convex function $g$ by
   looking at first-order optimality conditions based on directional
   derivatives.
   For all $\theta$ in $\Theta$, we have
   \begin{displaymath}
   \begin{split}
       \nabla g (\theta^\star,\theta-\theta^\star) & = \lim_{t \to 0^+} \frac{f(\theta^\star + t(\theta-\theta^\star)) - f(\theta^\star)  - \frac{\mu t^2}{2} \|\theta-\theta^\star\|_2^2}{t} \\
       & =
\lim_{t \to 0^+} \frac{f(\theta^\star + t(\theta-\theta^\star)) - f(\theta^\star)}{t} = \nabla f(\theta^\star,\theta-\theta^\star) \geq 0,
          \end{split}
   \end{displaymath}
   where $\nabla f(\theta^\star,\theta-\theta^\star)$ is non-negative because $\theta^\star$ is a stationary point of $f$ on~$\Theta$. Thus, $\theta^\star$ is also a stationary point of the function~$g$ on $\Theta$, and 
   is a minimizer of~$g$ on~$\Theta$ since $g$ is convex~\citep[][Proposition 2.1.2]{borwein}. This is sufficient to conclude.
\end{proof}

\begin{lemma}[\bfseries Lipschitz Continuity of Minimizers for Parameterized Functions]~\label{lemma:argmin}\newline
   Let $f: \Real^{p_1} \times \Theta_2 \to \Real$ be a function of two variables where $\Theta_2 \subseteq \Real^{p_2}$ is a convex set. Assume that  
   \begin{itemize}
\setlength\itemindent{25pt}
       \item $\theta_1 \mapsto f(\theta_1,\theta_2)$ is differentiable for all~$\theta_2$ in $\Theta_2$;
       \item $\theta_2 \mapsto \nabla_{1} f(\theta_1,\theta_2)$ is $L$-Lipschitz continuous for all $\theta_1$ in $\Real^{p_1}$;
       \item $\theta_2 \mapsto f(\theta_1,\theta_2)$ is $\mu$-strongly convex for all $\theta_1$ in~$\Real^{p_1}$. 
   \end{itemize}
   Then, the function $\theta_1 \mapsto \argmin_{\theta_2 \in \Theta_2} f(\theta_1,\theta_2)$ is well defined and $\frac{L}{\mu}$-Lipschitz.
\end{lemma}
\begin{proof}
   Let us consider $\theta_1, \theta_1'$ in $\Real^{p_1}$ and the corresponding (unique by strong convexity) solutions $\theta_2^\star \defin \argmin_{\theta_2 \in \Theta_2} f(\theta_1,\theta_2)$ and $\theta_2^{\prime\star} \defin \argmin_{\theta_2 \in \Theta_2} f(\theta_1',\theta_2)$. 
   From the second-order growth condition of Lemma~\ref{lemma:second}, we have
   $$
   \frac{\mu}{2}\|\theta_2^\star-\theta_2^{\prime\star}\|_2^2 \leq f(\theta_1,\theta_2^{\prime\star}) - f(\theta_1,\theta_2^\star), 
   $$
   and
   $$
   \frac{\mu}{2}\|\theta_2^\star-\theta_2^{\prime\star}\|_2^2 \leq f(\theta_1',\theta_2^\star) - f(\theta_1',\theta_2^{\prime\star}), 
   $$
   Define the function $g: \kappa \mapsto f(\kappa,\theta_2^{\prime\star}) - f(\kappa,\theta_2^\star)$ and sum the above inequalities. We obtain
   \begin{displaymath} 
   \mu\|\theta_2^\star-\theta_2^{\prime\star}\|_2^2 \leq g(\theta_1)-g(\theta_1').
   \end{displaymath}
   We notice that the gradient of $g$ is bounded: for all $\kappa$ in $\Real^{p_1}$, $\|\nabla g(\kappa)\|_2 =\| \nabla_1 f(\kappa,\theta_2^{\prime\star}) - \nabla_1 f(\kappa,\theta_2^\star) \| \leq L\|\theta_2^{\prime\star} - \theta_2^\star\|_2$. We use here the fact that $\nabla_1 f$ is $L$-Lipschitz with respect to its second argument. Thus, $g$ is Lipschitz with constant $L\|\theta_2^{\prime\star} - \theta_2^\star\|_2$ and
 \begin{displaymath} 
    \mu\|\theta_2^\star-\theta_2^{\prime\star}\|_2^2 \leq L\|\theta_2^\star-\theta_2^{\prime\star}\|_2 \|\theta_1-\theta_1'\|_2.
 \end{displaymath}
This is sufficient to conclude.
\end{proof}
\begin{lemma}[\bfseries Differentiability of Optimal Value Functions]~\label{lemma:danskin}\\
   Let us consider a function $f$ defined as in Lemma~\ref{lemma:argmin} and with the same properties.
   Define the optimal value function $\tilde{f}(\theta_1) \defin
   \min_{\theta_2 \in \Theta_2} f(\theta_1,\theta_2)$.
   Then, $\tilde{f}$ is differentiable and $\nabla \tilde{f} (\theta_1) = \nabla_{1} f(\theta_1,\theta_2^\star)$, where
   $\theta_2^\star \defin \argmin_{\theta_2 \in \Theta_2} f(\theta_1,\theta_2)$.
   Moreover, 
   \begin{enumerate}
\setlength\itemindent{25pt}
      \item when $f$ is convex and $\theta_1 \mapsto \nabla_1 f(\theta_1,\theta_2)$ is $L'$-Lipschitz continuous for all $\theta_2$ in $\Theta_2$, the function~$\tilde{f}$ is convex and $\nabla \tilde{f}$ is Lipschitz continuous with constant~$L'$;
      \item when $\theta_1 \mapsto f(\theta_1,\theta_2)$ is concave for all $\theta_2$ in $\Theta_2$, the function $\tilde{f}$ is concave and $\nabla \tilde{f}$ is Lipschitz continuous with constant~$\frac{2L^2}{\mu}$;
      \item when $\theta_1 \mapsto f(\theta_1,\theta_2)$ is affine for all $\theta_2$ in $\Theta_2$, the function $\tilde{f}$ is concave and $\nabla \tilde{f}$ is Lipschitz continuous with constant~$\frac{L^2}{\mu}$.
   \end{enumerate}
\end{lemma}
\begin{proof}
Note that this lemma is a variant of a theorem introduced by~\citet{danskin}. We first prove the
differentiability of $f$ before detailing how to obtain the Lipschitz
constants.

{\noindent \bfseries Differentiability of $f$:}\\
Let us consider $\theta_1$ and $\theta_1'$ in $\Real^{p_1}$, and let us use the same notation and definitions as in the proof of Lemma~\ref{lemma:argmin}. 
Then, we have
\begin{equation}
\begin{split}
\tilde{f}(\theta_1')-\tilde{f}(\theta_1) & = f(\theta_1',\theta_2^{\prime\star}) - f(\theta_1,\theta_2^{\star}) \\
                                         & = f(\theta_1',\theta_2^{\prime\star}) - f(\theta_1',\theta_2^\star) +   f(\theta_1',\theta_2^\star) - f(\theta_1,\theta_2^{\star}) \\
                                         & = g(\theta_1') +  f(\theta_1',\theta_2^\star) - f(\theta_1,\theta_2^{\star}) \\
                                         & = g(\theta_1') +  \nabla_1 f(\theta_1,\theta_2^\star)^\top(\theta_1'-\theta_1) + o(\|\theta_1'-\theta_1\|_2),
\end{split} \label{eq:firstpoint}
\end{equation}
where $g$ is defined in the proof of Lemma~\ref{lemma:argmin}.
Recall that the function $g$ is Lipschitz with constant $L\|\theta_2^{\prime\star} - \theta_2^\star\|_2$ (see the proof of Lemma~\ref{lemma:argmin}). Thus,
\begin{equation}
   |g(\theta_1')| \leq |g(\theta_1)-g(\theta_1')| \leq L\|\theta_2^{\prime\star} - \theta_2^\star\|_2 \|\theta_1'-\theta_1\|_2\leq \frac{L^2}{\mu}\|\theta_1'-\theta_1\|_2^2,\label{eq:secondpoint}
\end{equation}
where the first inequality uses the fact that $g(\theta_1') \leq 0$ and $g(\theta_1) \geq 0$. The last inequality uses Lemma~\ref{lemma:argmin}. We can now show that
\begin{displaymath}
\tilde{f}(\theta_1') = \tilde{f}(\theta_1) + \nabla_1 f(\theta_1,\theta_2^\star)^\top(\theta_1'-\theta_1) + o(\|\theta_1'-\theta_1\|_2).
\end{displaymath}
The function $\tilde{f}$ thus admits a first-order Taylor expansion and is differentiable. Moreover, we have $\nabla \tilde{f}(\theta_1) = \nabla_1 f(\theta_1,\theta_2^\star)$.

{\noindent \bfseries Proof of the first point:}\\
When $f$ is jointly convex in $\theta_1$ and~$\theta_2$, it is easy to show
that $\tilde{f}$ is also convex~\citep[][Section 3.2.5]{boyd}.

By explicitly upper-bounding the quantity $o(\|\theta_1'-\theta_1\|_2)$ in Eq.~(\ref{eq:firstpoint}) using the $L'$-Lipschitz continuity of $\nabla_1 f$ in its first argument and the inequality $g(\theta_1') \leq 0$, we have 
\begin{displaymath}
0 \leq \tilde{f}(\theta_1')-\tilde{f}(\theta_1) - \nabla\tilde{f}(\theta_1)^\top(\theta_1'-\theta_1) \leq \frac{L'}{2}\|\theta_1'-\theta_1\|_2^2,
\end{displaymath}
we can apply Lemma~\ref{lemma:upperlipschitz2} to ensure that $\nabla \tilde{f}$ is $L'$-Lipschitz continuous.

{\noindent \bfseries Proof of the second point:}\\
$-\tilde{f}$ is a pointwise supremum of convex functions and is therefore convex~\citep[see][Section 3.2.3]{boyd}.
Then, we have from Eq.~(\ref{eq:firstpoint}) and using the concavity of $\theta_1 \mapsto f(\theta_1,\theta_2^\star)$:
\begin{displaymath}
  \tilde{f}(\theta_1')-\tilde{f}(\theta_1) \geq g(\theta_1') +\nabla \tilde{f}(\theta_1)^\top(\theta_1'-\theta_1).
\end{displaymath}
Thus, 
\begin{displaymath}
  0\leq -\tilde{f}(\theta_1')+\tilde{f}(\theta_1)+ \nabla \tilde{f}(\theta_1)^\top(\theta_1'-\theta_1) \leq |g(\theta_1')| \leq \frac{L^2}{\mu}\|\theta_1-\theta_1'\|_2^2,
\end{displaymath}
where the last inequality was shown in Eq.~(\ref{eq:secondpoint}). We can then
apply Lemma~\ref{lemma:upperlipschitz2} to the convex function~$-\tilde{f}$ and
we obtain the desired Lipschitz constant~$\frac{2L^2}{\mu}$.

{\noindent \bfseries Proof of the third point:}\\
When $\theta_1 \mapsto f(\theta_1,\theta_2)$ is affine, $\nabla_1 f(\theta_1,\theta_2)$ is independent of $\theta_1$.  
\begin{displaymath}
\begin{split}
   \|\nabla \tilde{f}(\theta_1') - \nabla \tilde{f}(\theta_1)\|_2 & =  \|\nabla_1 f(\theta_1',\theta_2^{\prime\star}) - \nabla_1 \tilde{f}(\theta_1,\theta_2^\star)\|_2 \\
   & =  \|\nabla g(\theta_1)\|_2 \leq L\|\theta_2-\theta_2^\star\|_2 \leq \frac{L^2}{\mu}\| \theta_1-\theta_1'\|_2,
\end{split}
\end{displaymath}
where the upper-bound on the gradient of $g$ was shown in the proof of Lemma~\ref{lemma:argmin}.
\end{proof}
\begin{lemma}[\bfseries Pythagoras Relation]~\label{lemma:pythagoras}\newline
Let $\theta,\kappa,\nu$ in $\Real^p$. Then
\begin{displaymath}
   \|\kappa-\theta\|_2^2 + 2 (\kappa-\theta)^\top(\theta-\nu) = \|\kappa-\nu\|_2^2 - \|\theta - \nu\|_2^2.
\end{displaymath}
\end{lemma}
\begin{lemma}[\bfseries Regularity of Residual Functions]~\label{lemma:convexerror}\\
Let $f,g: \Real^p \to \Real$ be two functions. Define the difference function $h\defin g-f$. Then,
   \begin{enumerate}
\setlength\itemindent{25pt}
       \item if $g$ is $\rho$-strongly convex and $f$ differentiable with $L$-Lipschitz continuous gradient, with $\rho \geq L$, the function~$h$ is $(\rho-L)$-strongly convex; \label{step:convexerror}
       \item if $g$ and $f$ are convex and differentiable with $L$-Lipschitz continuous gradient, $\nabla h$ is $L$-Lipschitz continuous.
       \item if $g$ and $f$ are $\mu$-strongly convex and differentiable with $L$-Lipschitz continuous gradient, $\nabla h$ is $(L-\mu)$-Lipschitz continuous.
   \end{enumerate}
\end{lemma}
\begin{proof}~\\
{\noindent \bfseries Proof of the first point:}\\
   Let $\theta'$ be in $\Real^p$, and define $l: \theta \mapsto g(\theta)-f(\theta) - \frac{(\rho-L)}{2}\|\theta-\theta'\|_2^2$.
   Then, 
   \begin{displaymath}
       l(\theta)= \left(g(\theta) - \frac{\rho}{2}\|\theta-\theta'\|_2^2\right) + \left(\frac{L}{2}\|\theta-\theta'\|_2^2 - f(\theta)\right),
   \end{displaymath}
   The left term inside parentheses is convex by definition of strong convexity. Let us call the right term $l': \theta \mapsto \frac{L}{2}\|\theta-\theta'\|_2^2 - f(\theta)$. The function $l'$ is differentiable and we can show that it is above its tangent, therefore convex. Let us indeed fix $\kappa$ in $\Real^p$:
   \begin{displaymath}
   \begin{split}
      l'(\theta)= \frac{L}{2}\|\theta-\theta'\|_2^2 - f(\theta) & \geq -f(\kappa) -\nabla f(\kappa)^\top(\theta-\kappa) - \frac{L}{2}\|\theta-\kappa\|_2^2 + \frac{L}{2}\|\theta-\theta'\|_2^2 \\
       & = \left(\frac{L}{2}\|\kappa-\theta'\|_2^2 -f(\kappa) \right) + L(\kappa-\theta')^\top(\theta-\kappa) -\nabla f(\kappa)^\top(\theta-\kappa) \\
       & = l'(\kappa) + \nabla l'(\kappa)^\top (\theta-\kappa).
    \end{split}
   \end{displaymath}
The first inequality comes from Lemma~\ref{lemma:upperlipschitz} applied to the function $f$ at $\kappa$. The second equality is simply due to the trivial relation described in Lemma~\ref{lemma:pythagoras}.

{\noindent \bfseries Proof of the second and third points:}\\
    We simply prove the third point, and then obtain the second point by choosing $\mu=0$.
    We have for all~$\theta$ and $\theta'$ in $\Real^p$, according to Lemma~\ref{lemma:upperlipschitz} and~\ref{lemma:lower}
   \begin{displaymath}
       \frac{\mu}{2}\|\theta-\theta'\|_2^2 \leq f(\theta') - f(\theta) - \nabla f(\theta)^\top (\theta'-\theta) \leq \frac{L}{2}\|\theta-\theta'\|_2^2,
   \end{displaymath}
   and 
   \begin{displaymath}
      -\frac{L}{2}\|\theta-\theta'\|_2^2 \leq -g(\theta') + g(\theta) + \nabla g(\theta)^\top (\theta'-\theta) \leq -\frac{\mu}{2}\|\theta-\theta'\|_2^2.
   \end{displaymath}
   Summing the two inequalities we have that
    \begin{displaymath}
      |h(\theta') - h(\theta) - \nabla h(\theta)^\top (\theta'-\theta)| \leq \frac{L-\mu}{2}\|\theta-\theta'\|_2^2.
   \end{displaymath}
   where $h\defin g-f$. Since $h$ is differentiable with a Lipschitz gradient, the result follows from Lemma~\ref{lemma:upperlipschitz2} (whether $h$ is convex or not).
\end{proof}

\section{Mechanisms to Construct First-Order Surrogate Functions}\label{appendix:mechanisms}
We provide here some details and justifications to Section~\ref{subsec:surrogates}. 
We start with a basic lemma, which gives us elementary techniques to combine surrogate functions.
\begin{lemma}[\bfseries Combination Rules for Majorant First-Order Surrogates]~\label{lemma:combination}\newline
 Let us consider two functions $f: \Real^p \to \Real$ and $f': \Real^p \to \Real$, and majorant surrogate functions $g$ in $\S_{L}(f,\kappa)$ and $g'$ in $\S_L(f',\kappa)$ for some $\kappa$ in $\Theta$. Then, the following combination rules hold:
 \begin{itemize}
\setlength\itemindent{25pt}
    \item {\bfseries Linear combination}: for all $\alpha,\beta > 0$, $\alpha g+\beta g'$ is a majorant surrogate function in $\S_{\alpha L + \beta L'}(\alpha f+\beta f',\kappa)$;
    \item {\bfseries Transitivity}: consider $g''$ a majorant surrogate in $\S_{L''}(g,\kappa)$. Then, $g''$ is a majorant surrogate in $\S_{L+L''}(f,\kappa)$;
    \item {\bfseries Negation}: the function $g'': \theta \mapsto -g(\theta) + \frac{L}{2}\|\theta-\kappa\|_2^2$ is a majorant surrogate in $\S_{2L}(-f,\kappa)$. 
  \end{itemize}
 \end{lemma}
\begin{proof}
 The first two points are easy to check. For the last one, we have for all
 $\theta$ in $\Theta$, $g(\theta)-f(\theta) \leq
 \frac{L}{2}\|\theta-\kappa\|_2^2$ according to
 Lemma~\ref{lemma:basic}. The proposed surrogate is therefore majorant for~$-f$.
 We can now define the approximation error function $h'': \theta \mapsto f(\theta)-g(\theta)+\frac{L}{2}\|\theta-\kappa\|_2^2$, which is
 differentiable with $2L$-Lipschitz continuous gradient and $g''$ is in $\S_{2L}(-f,\kappa)$ (we have used the fact that $\theta \mapsto \frac{L}{2}\|\theta-\kappa\|_2^2$ and $h \defin g-f$ are both differentiable and their gradients are $L$-Lipschitz).
\end{proof}

In the next paragraphs, we justify the different surrogates we have introduced in Section~\ref{subsec:surrogates}.
\paragraph{Lipschitz Gradient Surrogates.}~\\
When $f$ is differentiable and $\nabla f$ is $L$-Lipschitz, we consider
the following surrogate:
\begin{displaymath}
    g: \theta \mapsto f(\kappa) + \nabla f(\kappa)^\top (\theta-\kappa) + \frac{L}{2}\|\theta-\kappa\|_2^2.
\end{displaymath}
By applying~Lemma~\ref{lemma:upperlipschitz} and studying the approximation
error $h\defin g-f$, we immediately obtain that $g$ is a majorant surrogate in
$\S_{2L,L}(f,\kappa)$.
When $f$ is convex, we can use Lemma~\ref{lemma:convexerror} to prove that $g$
is in $\S_{L,L}(f,\kappa)$ and $\S_{L-\mu,L}(f,\kappa)$ when $f$ is
$\mu$-strongly convex.

\paragraph{Proximal Gradient Surrogates.}~\\
Assume that $f$ splits into $f = f_1 + f_2$, where $f_1$ is differentiable with a $L$-Lipschitz gradient. Then, we have
presented the following surrogate 
\begin{displaymath}
    g: \theta \mapsto f_1(\kappa) + \nabla f_1(\kappa)^\top (\theta-\kappa) + \frac{L}{2}\|\theta-\kappa\|_2^2 + f_2(\theta).
\end{displaymath}
Following the same arguments as in the previous paragraph, we have that $g$ is in $\S_{2L}(f,\kappa)$.
Moreover, when $f_1$ is convex, $g$ is in $\S_{L}(f,\kappa)$. If $f_2$ is also convex, $g$ is in $\S_{L,L}(f,\kappa)$.
When $f_1$ is $\mu$-strongly convex, $g$ is in $\S_{L-\mu}(f,\kappa)$. If $f_2$ is also convex, $g$ is in $\S_{L-\mu,L}(f,\kappa)$.
 
\paragraph{DC Programming Surrogates.}~\\
Assume that $f = f_1 + f_2$, where $f_2$ is concave and differentiable with a $L_2$-Lipschitz gradient. Then, we have presented the following surrogate 
\begin{displaymath}
   g: \theta \mapsto f_1(\theta) + f_2(\kappa) + \nabla f_2(\kappa)^\top (\theta-\kappa).
\end{displaymath}
It is easy to see that $g$ is a majorant surrogate since $f_2$ is concave and
below its tangents. It is also easy to see that the approximation error $g-f$ has
a $L_2$-Lipschitz continuous gradient.

\paragraph{Variational Surrogates.}~\\
Let $f$ be a function defined on $\Real^{p_1} \times \Real^{p_2}$. Let $\Theta_1 \subseteq \Real^{p_1}$ and $\Theta_2 \subseteq \Real^{p_2}$ be two convex sets. 
Define~$\tilde{f}$ as $\tilde{f}(\theta_1) \defin \min_{\theta_2 \in \Theta_2} f(\theta_1,\theta_2)$ and
assume that 
\begin{itemize}
\setlength\itemindent{10pt}
\item $\theta_1 \mapsto f(\theta_1,\theta_2)$ is differentiable for all~$\theta_2$ in $\Theta_2$;
\item $\theta_2 \mapsto \nabla_{1} f(\theta_1,\theta_2)$ is $L$-Lipschitz for all $\theta_1$ in $\Real^{p_1}$;
\item $\theta_1 \mapsto \nabla_{1} f(\theta_1,\theta_2)$ is $L'$-Lipschitz for all $\theta_2$ in $\Theta_2$;
\item $\theta_2 \mapsto f(\theta_1,\theta_2)$ is $\mu$-strongly convex for all $\theta_1$ in~$\Real^{p_1}$. 
\end{itemize}
Let us fix $\kappa_1$ in $\Theta_1$. Then, we can show that the following function is a majorant surrogate in $\S_{L''}(\tilde{f},\kappa)$ for some $L'' > 0$:
\begin{displaymath}
g: \theta_1 \mapsto f(\theta_1,\kappa_2^\star) ~\text{with}~~ \kappa_2^\star \defin \argmin_{\theta_2 \in \Theta_2} \tilde{f}(\kappa_1,\theta_2).
\end{displaymath}
 We can indeed apply Lemma~\ref{lemma:danskin} which ensures that $\tilde{f}$ is
 differentiable with $\nabla \tilde{f}(\theta_1) = \nabla_1
 f(\theta_1,\theta_2^\star)$ and $\theta_2^\star \defin \argmin
 f(\theta_1,\theta_2)$. Considering the approximation error function $h\defin g-\tilde{f}$, we indeed have that $h(\kappa_1)=0$, $\nabla h(\kappa_1)=0$ and since $\theta_2^\star$ as a function of $\theta_1$ is
 Lipschitz according to Lemma~\ref{lemma:argmin}, we also have that $\nabla h$ is
 Lipschitz continuous. 
 
When $f$ is jointly convex in $\theta_1$ and~$\theta_2$, $\tilde{f}$
is itself convex and $\nabla \tilde{f}$ is $L'$-Lipschitz continuous according
to Lemma~\ref{lemma:danskin}. We can then apply  Lemma~\ref{lemma:convexerror} to
obtain that $\nabla h$ is $L'$-Lipschitz continuous such that we can choose $L''=L'$.

\paragraph{Saddle Point Surrogates.}~\\
Let us make the same assumptions as in the previous paragraph with the following exceptions
\begin{itemize}
\setlength\itemindent{10pt}
\item $\theta_2 \!\mapsto\! f(\theta_1,\theta_2)$ is $\mu$-strongly concave for all $\theta_1$ in~$\Real^{p_1}$; 
\item $\theta_1  \!\mapsto\! f(\theta_1,\theta_2)$ is convex for all $\theta_2$ in $\Theta_2$;
\item $\tilde{f}(\theta_1) \defin \max_{\theta_2 \in \Theta_2} f(\theta_1,\theta_2).$
\end{itemize}
Then, $\tilde{f}$ is convex as the pointwise supremum of convex functions~\citep[see][]{boyd} and we can show that the function below is a majorant surrogate in $\S_{2L''}(\tilde{f},\kappa_1)$:
$$ g: \theta_1 \mapsto f(\theta_1,\kappa_2^\star) + \frac{L''}{2}\|\theta_1-\kappa_1\|_2^2,$$
where $L''\defin \max(2L^2 /\mu,L')$.
When $\theta_1 \mapsto f(\theta_1,\theta_2)$ is affine, we can instead choose $L''\defin L^2 /\mu$.

We indeed apply the same methodology as in the previous paragraph.
Lemma~\ref{lemma:danskin} tells us that the function
 $-\tilde{f}$ is differentiable with $2L^2/\mu$-Lipschitz continuous gradient
 (only $L^2/\mu$ in the affine case). 
Then, we have that the function $\theta_1
 \mapsto -f(\theta_1,\kappa_2^\star)$ is in $\S_{L''}(-\tilde{f},\kappa_1)$ by
 using Lemma~\ref{lemma:convexerror}.  We then apply the negation rule of
 Lemma~\ref{lemma:combination} to conclude.

\paragraph{Jensen Surrogates.}~\\
Let us recall the definition of Jensen surrogates.
Following \citet{lange2}, we consider a convex function
$f: \Real \mapsto \Real$, a vector $\x$ in $\Real^p$ and define $\tilde{f}:
\Real^p \to \Real$ as $\tilde{f}(\theta) \defin f(\x^\top \theta)$ for
all $\theta$. Let $\w$ in $\Real_+^p$ be a weight vector such that $\w \geq 0$, $\|\w\|_1=1$ and $\w_i \neq 0$ whenever $\x_i \!\neq\! 0$. Then, we consider the following function~$g$ for any $\kappa$ in~$\Real^p$
\begin{displaymath}
   g: \theta \mapsto \sum_{i=1}^p \w_i f \left(\frac{\x_i}{\w_i}( \theta_i-\kappa_i) + \x^\top \kappa\right),
\end{displaymath}
Assume that $f$ is differentiable with a $L$-Lipschitz gradient and $\w_i
\defin |\x_i|^\nu / \|\x\|_\nu^\nu$ for some~$\nu \geq 0$.\footnote{With an abuse of notation, $\|\x\|_{0}^0$ denotes the $\ell_0$-pseudo norm, also denoted by~$\|\x\|_0$.}
$\nabla f$ is obviously Lipschitz with constant $L\|\x\|_2^2$. $g$ is also convex,
differentiable with Lipschitz continuous gradient with constant $L'$
obtained below with simple calculations:
\begin{itemize}
\setlength\itemindent{10pt}
   \item if $\nu=0$, $L' = L\|\x\|_\infty^2\|\x\|_0$;
   \item if $\nu=1$, $L' = L\|\x\|_\infty\|\x\|_1$;
   \item if $\nu=2$, $L' = L\|\x\|_2^2$.
\end{itemize}
The fact that $g$ is majorant is a simple application of Jensen inequality.  It
is also obvious that $g(\kappa)=f(\kappa)$ and that $\nabla g(\kappa)=\nabla
f(\kappa)$.  We now apply Lemma~\ref{lemma:convexerror}, noticing that we
always have $L'$ greater than $L\|\x\|_2^2$, and we have that $g$ is in
$\S_{L'}(\tilde{f},\kappa)$. 

\paragraph{Quadratic Surrogates.}~\\
When $f$ is twice differentiable and admits a matrix~$\HH$ in such that $\nabla^2 f - \HH$ is always positive definite, the following 
function is a first-order majorant surrogate:
\begin{displaymath}
    g: \theta \mapsto f(\kappa) + \nabla f(\kappa)^\top (\theta-\kappa) + \frac{1}{2}(\theta-\kappa)^\top\HH(\theta-\kappa).
\end{displaymath}
The fact that it is majorant is simply an application of the mean-value theorem.

\section{Additional Optimization Scheme: Block Frank-Wolfe}\label{appendix:schemes}
We provide in this section an additional optimization scheme, combining the
ideas of Sections~\ref{sec:conditional} and~\ref{sec:bcd} with
separability assumptions on the surrogates $g_n$ and $\Theta$.  It results in a
block coordinate version of the Frank-Wolfe optimization scheme presented in
Algorithm~\ref{alg:conditional_bcd} generalizing a procedure recently
introduced by~\citet{lacoste}. More precisely, the algorithm of \citet{lacoste}
corresponds to using a quadratic surrogate as provided by
Lemma~\ref{lemma:upperlipschitz} when $f$ is smooth with $L$-Lipschitz
gradient, and performing a line search on the function $f$ instead of $g_n$.
Our approach on the other hand can afford to have a non-smooth component in $f$
and in that sense is more general. Note that \citet{lacoste} also presents
duality gap guarantees and various extensions and applications, which we
do not consider in our paper.
\begin{algorithm}[hbtp]
    \caption{Block Frank-Wolfe Scheme}\label{alg:conditional_bcd}
    \begin{algorithmic}[1]
    \INPUT $\theta_0=(\theta_0^1,\ldots,\theta_0^k) \in \Theta=\Theta^1 \times \ldots \times \Theta^k$ (initial point); $N$ (number of iterations).
    \FOR{ $n=1,\ldots,N$}
    \STATE Compute a separable majorant surrogate function $g_n=\sum_{i=1}^kg_n^i$ in $\S_{L,L}(f,\theta_{n-1})$;
    \STATE Randomly pick one block $\hati_n$ in $\{1,\ldots,k\}$ and compute a search direction:
    \begin{displaymath}
       \nu_n^{\hati_n} \in \argmin_{\theta^{\hati_n} \in \Theta^{\hati_n}} \left[ g_n^{\hati_n}(\theta^{\hati_n})-\frac{L}{2}\|\theta^{\hati_n}-\theta^{\hati_n}_{n-1}\|_2^2\right]. 
    \end{displaymath}
    \STATE Line search:
    \begin{displaymath}
       \alpha^\star \defin \argmin_{\alpha \in [0,1]} g_n^{\hati_n}((1-\alpha) \theta_{n-1}^{\hati_n} + \alpha\nu_n^{\hati_n}). 
    \end{displaymath}
    \STATE Update $\theta_n^{\hati_n}$:
    \begin{displaymath}
       \theta_n^{\hati_n} \defin (1-\alpha^\star) \theta_{n-1}^{\hati_n} + \alpha^\star\nu_n^{\hati_n}. 
    \end{displaymath}
    \ENDFOR
    \OUTPUT $\theta_{N} = (\theta_N^1,\ldots,\theta_N^k)$ (final estimate);
    \end{algorithmic}
\end{algorithm}
\begin{proposition}[\bfseries Convergence Rate for Algorithm~\ref{alg:conditional_bcd}]~\label{prop:conv12}\newline
Let $f$ be convex, bounded below and $f^\star$ be the minimum of $f$ on
$\Theta=\Theta^1 \times \ldots \times \Theta^k$. Assume that~$\Theta$ is bounded and call $R\defin \max_{\theta_1,\theta_2 \in \Theta} \|\theta_1-\theta_2\|_2$ its diameter.
The sequence $(f(\theta_n))_{n \geq 0}$ provided by Algorithm~\ref{alg:conditional_bcd} converges almost surely to $f^\star$ and
we have for all $n \geq 1$,
\begin{displaymath}
\E[f(\theta_n)-f^\star] \leq  \frac{2LR^2}{2+\delta(n-n_0)},
\end{displaymath}
where $\delta\defin 1/k$ and $n_0\defin \left\lceil \log\left(\frac{2(f(\theta_0)-f^\star)}{LR^2}-1\right) / \log\left(\frac{1}{1-\delta}\right) \right\rceil$ if $f(\theta_0)-f^\star > LR^2$ and $n_0\defin0$ otherwise.
\end{proposition}
\proofatend
Let us denote by $\nu_n^\star \defin \argmin_{\theta \in \Theta} \left[ g_n(\theta)-\frac{L}{2}\|\theta-\theta_{n-1}\|_2^2\right]$.
Because of the separability of the surrogate function $g_n$, we have after a few calculations
\begin{displaymath}
\E [f(\theta_n) | \theta_{n-1}] \leq \E[g_n(\theta_n) | \theta_{n-1}] \leq (1-\delta) f(\theta_{n-1}) + \delta\min_{\alpha \in [0,1]} g_n((1-\alpha) \theta_{n-1}+\alpha\nu_n^\star).
\end{displaymath}
Following the proof of Proposition~\ref{prop:conv11}, we have
\begin{displaymath}
\E [f(\theta_n) | \theta_{n-1}] \leq (1-\delta) f(\theta_{n-1}) + \delta\min_{\alpha \in [0,1]} \left[(1-\alpha) f(\theta_{n-1}) + \alpha f^\star + \frac{\alpha^2LR^2}{2}\right].
\end{displaymath}
Taking the expectation and defining $r_n \defin \E[f(\theta_n)-f^\star]$, we have
\begin{displaymath}
  r_n \leq (1-\delta) r_{n-1} + \delta \min_{\alpha \in [0,1]} (1-\alpha) r_{n-1} + \frac{\alpha^2 LR^2}{2},
\end{displaymath}
where we have used Jensen inequality similarly as in the proof of Proposition~\ref{prop:conv15}.

Minimizing with respect to $\alpha$ yields
\begin{displaymath}
   r_n \leq (1-\delta) r_{n-1} + \delta \left\{ \begin{array}{ll}
      \frac{LR^2}{2} & \text{if}~~ r_{n-1} > LR^2 \\
      r_{n-1}\left(1-\frac{r_{n-1}}{2LR^2}\right) & \text{otherwise}.
      \end{array} \right.
\end{displaymath}
This is the same recursive relations as in the proof of
Proposition~\ref{prop:conv15}, and we therefore obtain the same convergence
rate.
\endproofatend
The proof is given in Appendix~\ref{appendix:proofs}.

\section{Proofs of Lemmas and Propositions}\label{appendix:proofs}
We present in this section the proofs of the different lemmas and propositions in the paper.
\printproofs

\section{Additional Experimental Results}\label{appendix:exps}
Figures~\ref{fig:exp5} and~\ref{fig:exp6} presents benchmarks for $\ell_2$-logistic regressions with a different regularization parameter than Figure~\ref{fig:exp1}.
Similarly, we present $\ell_1$-logistic regressions benchmarks in Figures~\ref{fig:exp3} and~\ref{fig:exp4} with a different sparsity level than Figure~\ref{fig:exp2}.
\begin{figure}[hbtp]
\centering
\includegraphics[width=0.35\linewidth]{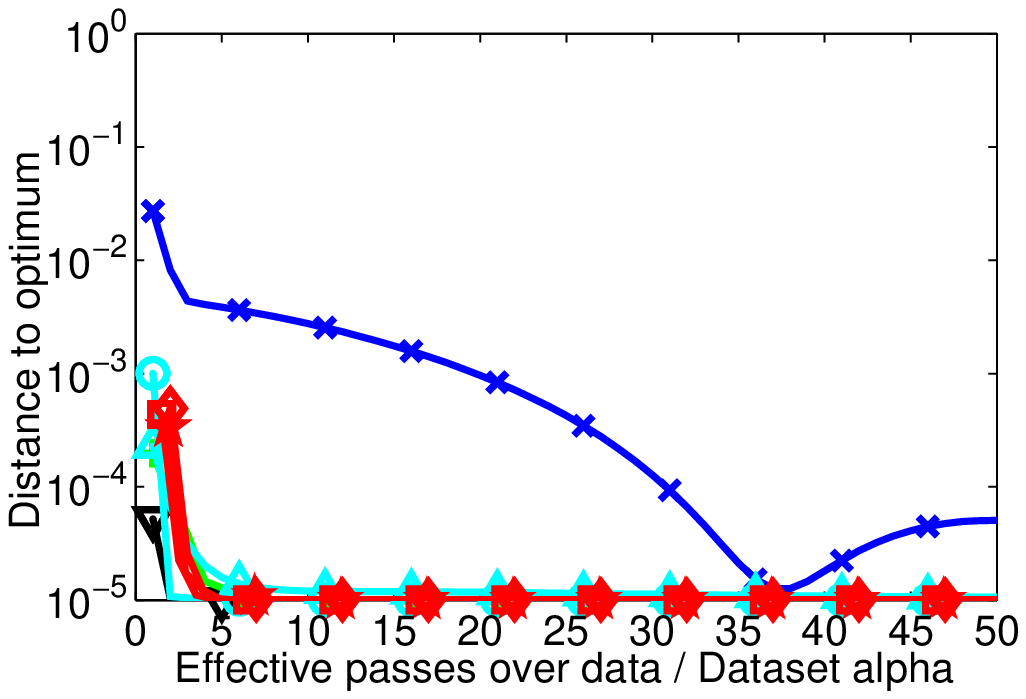}~
\includegraphics[width=0.375\linewidth]{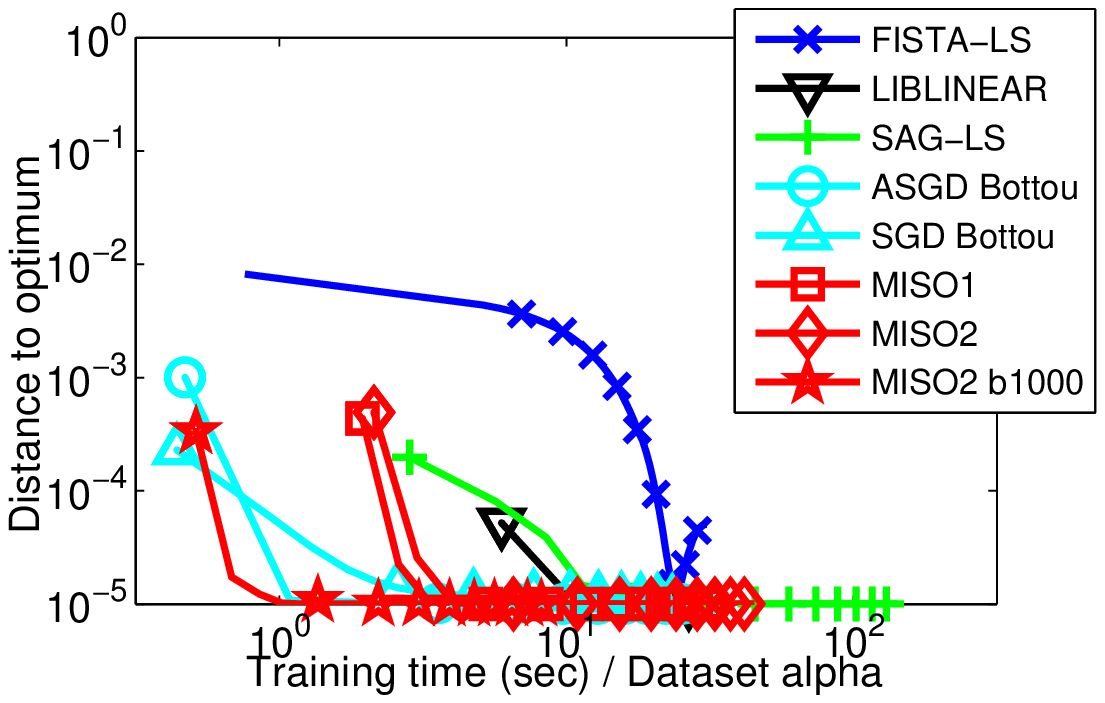}\\
\includegraphics[width=0.35\linewidth]{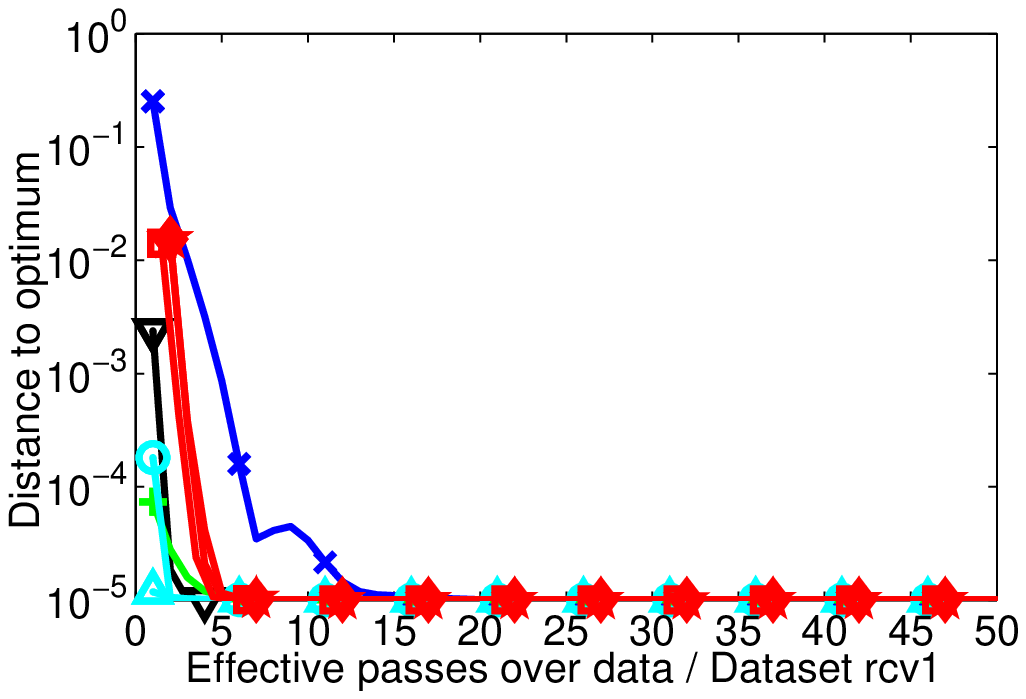}~
\includegraphics[width=0.375\linewidth]{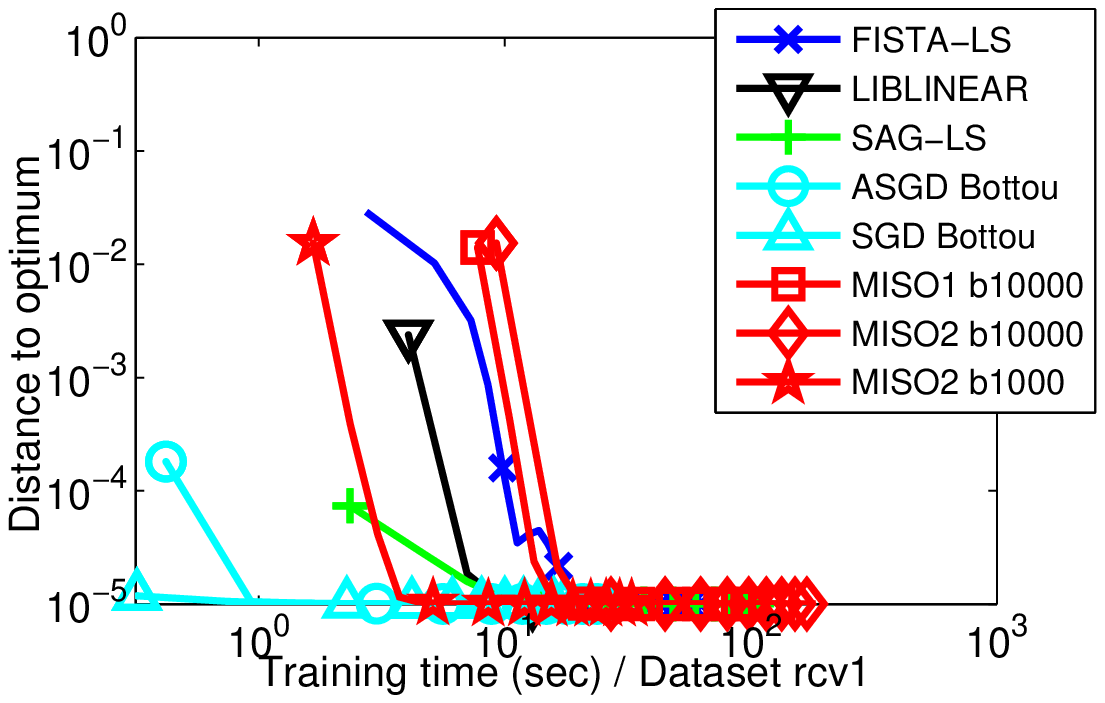}\\
\includegraphics[width=0.35\linewidth]{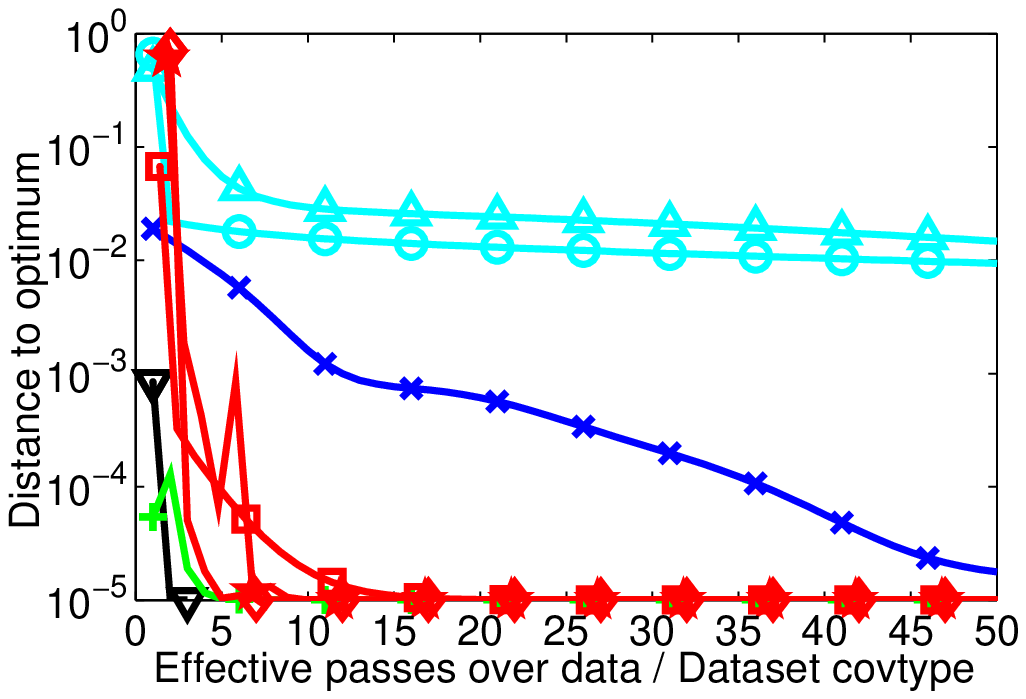}~
\includegraphics[width=0.375\linewidth]{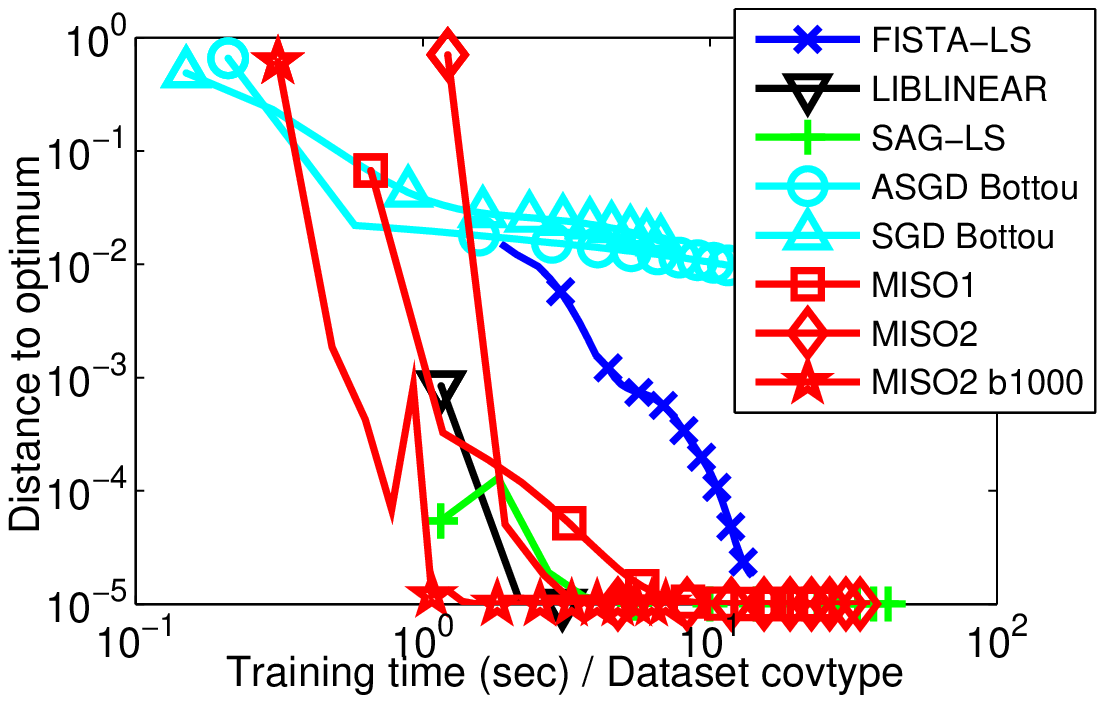}\\
\includegraphics[width=0.35\linewidth]{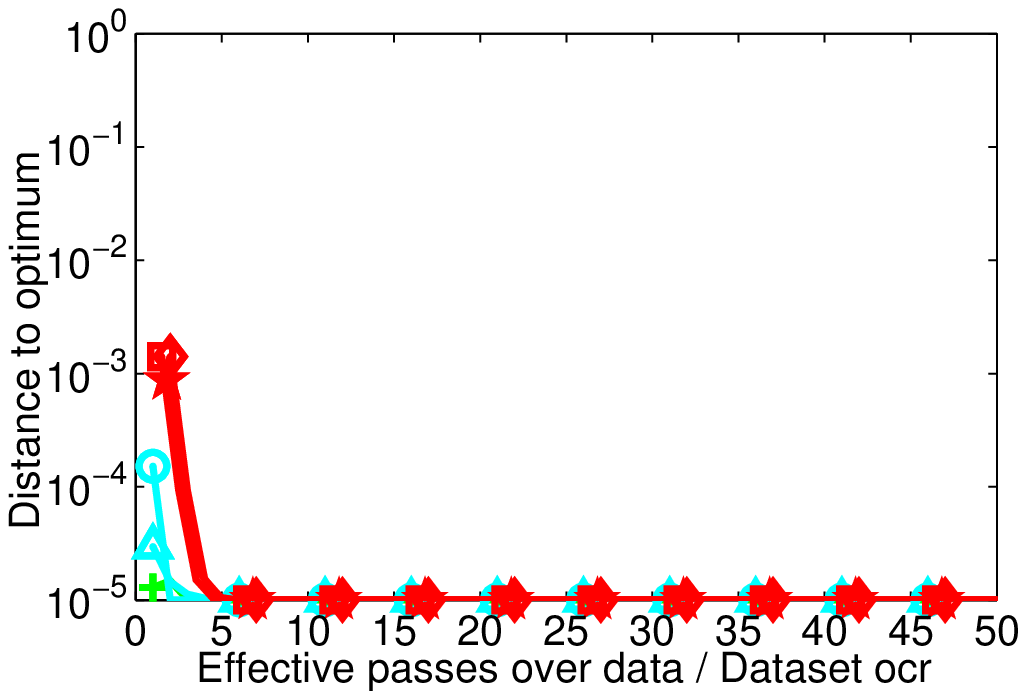}~
\includegraphics[width=0.375\linewidth]{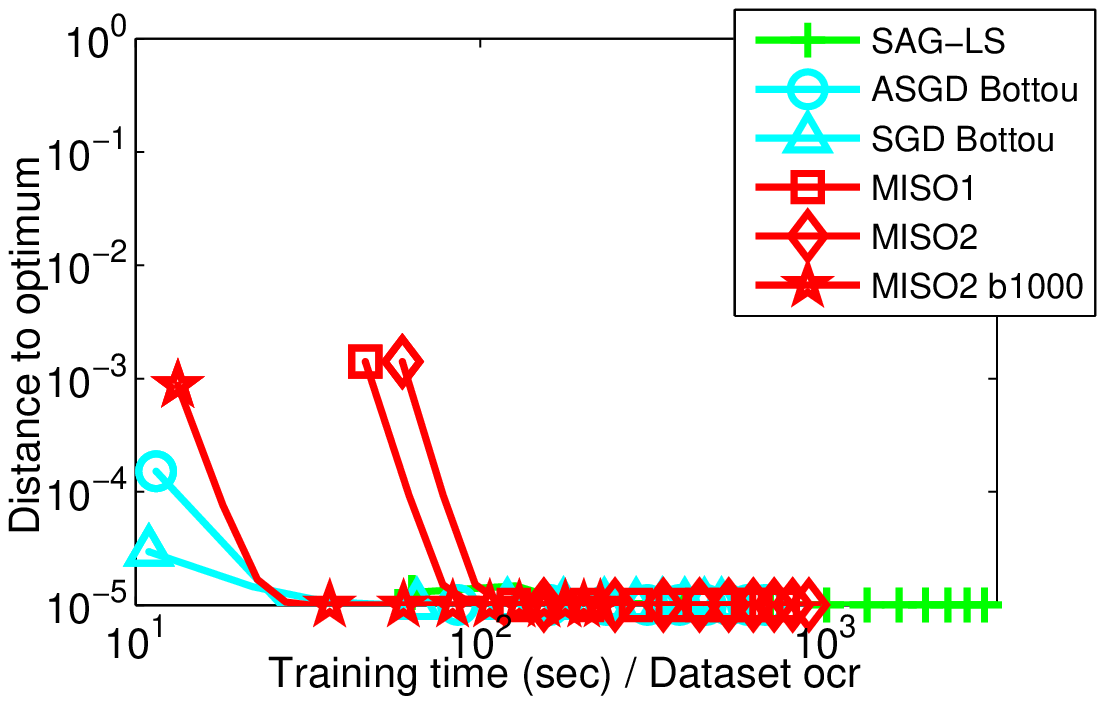}
\myvspace{0.2}
\caption{Benchmarks for $\ell_2$-logistic regression with $\lambda=10^{-3}$.}
\label{fig:exp5}
\end{figure}
\begin{figure}[hbtp]
\centering
\includegraphics[width=0.35\linewidth]{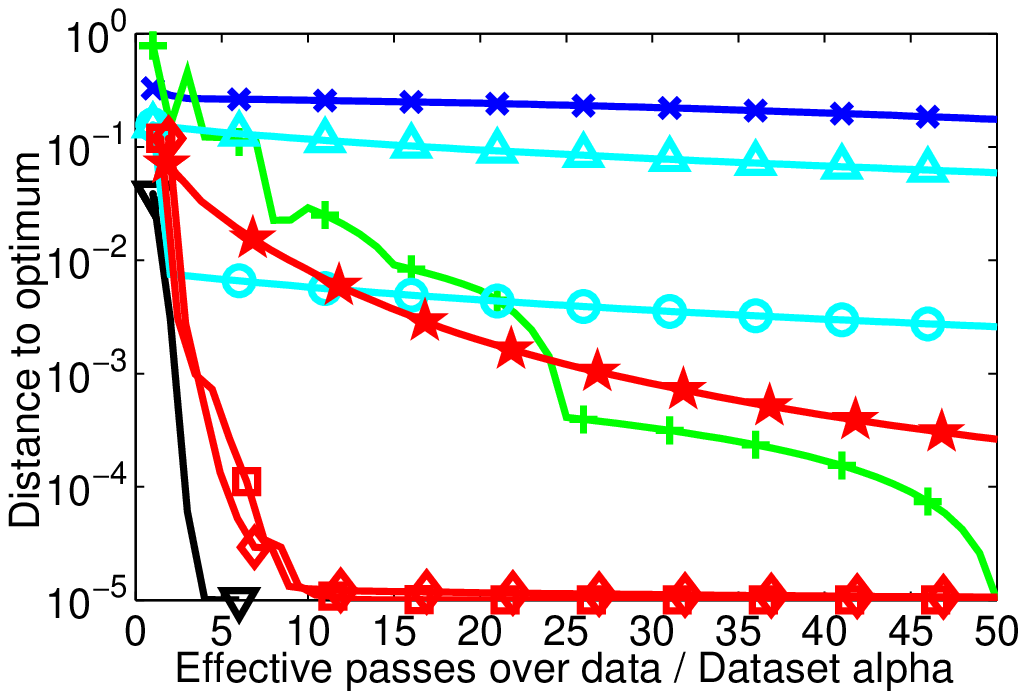}~
\includegraphics[width=0.375\linewidth]{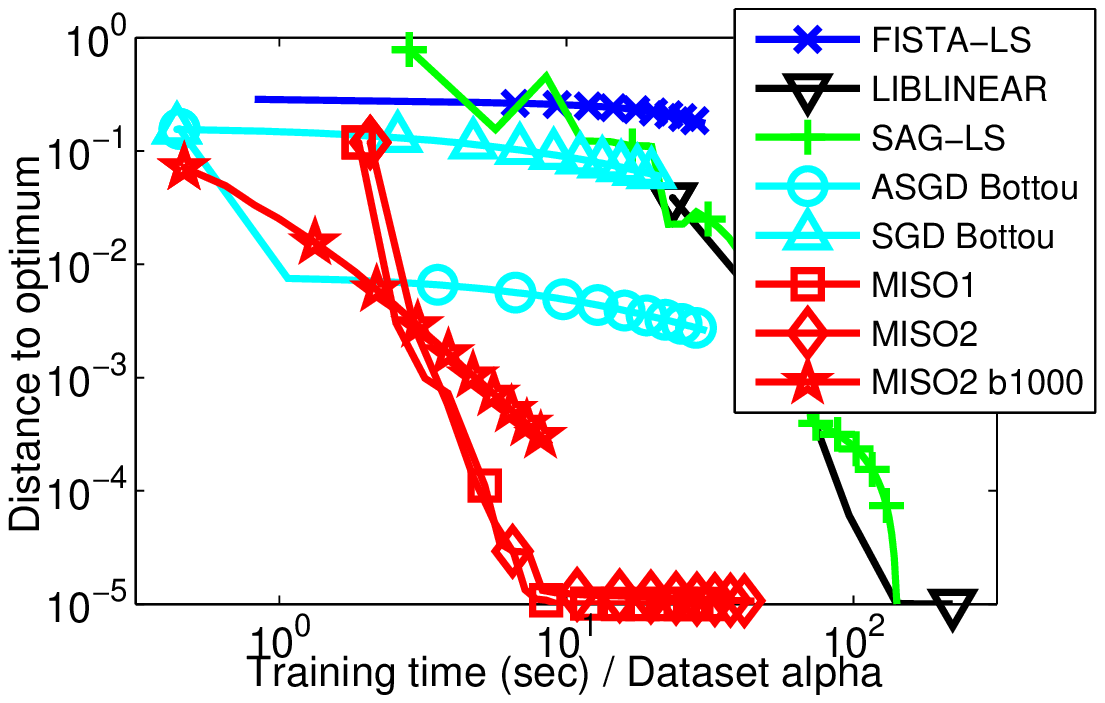}\\
\includegraphics[width=0.35\linewidth]{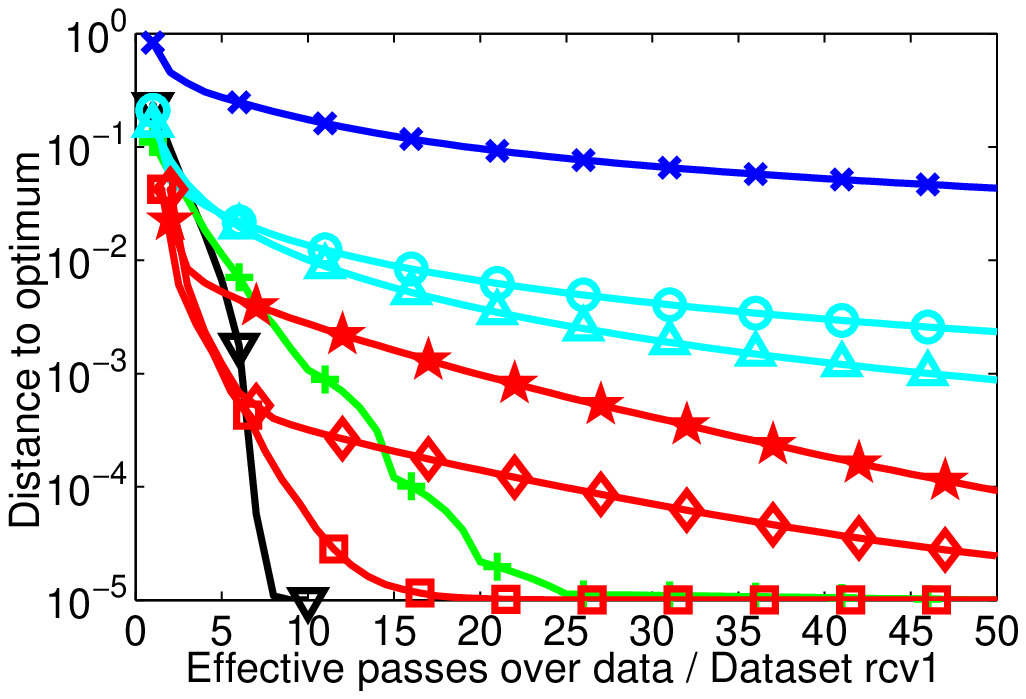}~
\includegraphics[width=0.375\linewidth]{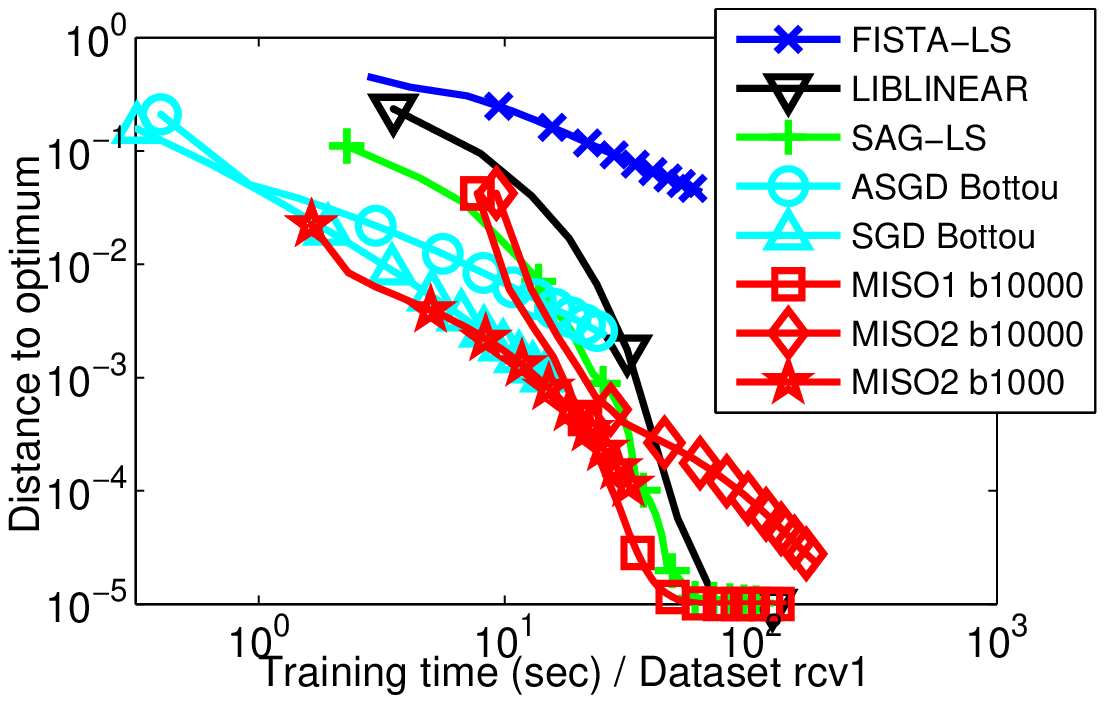}\\
\includegraphics[width=0.35\linewidth]{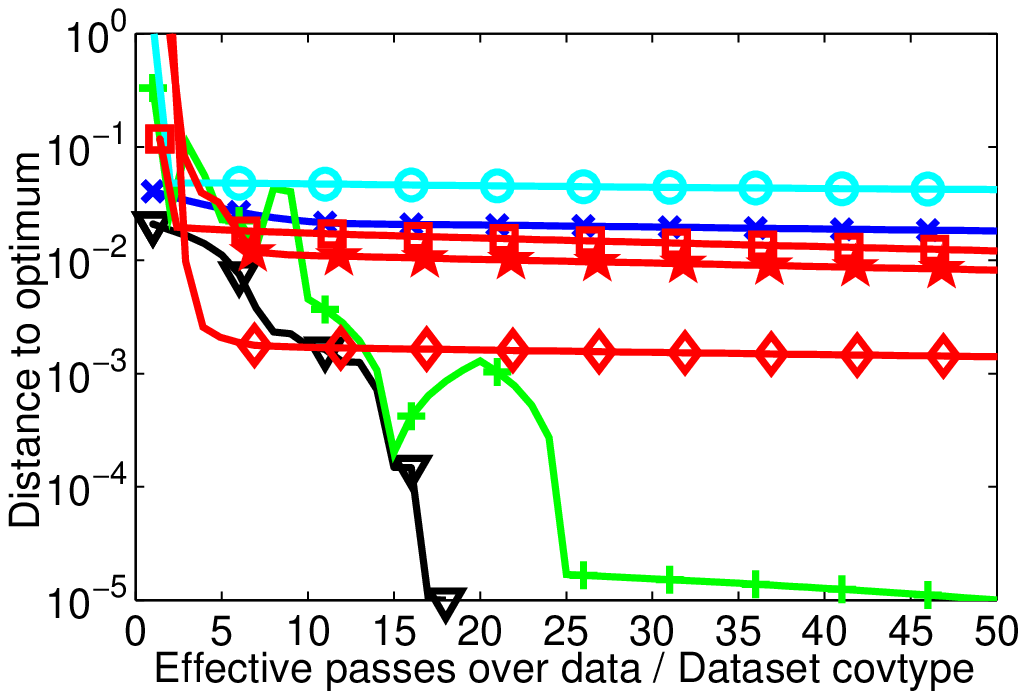}~
\includegraphics[width=0.375\linewidth]{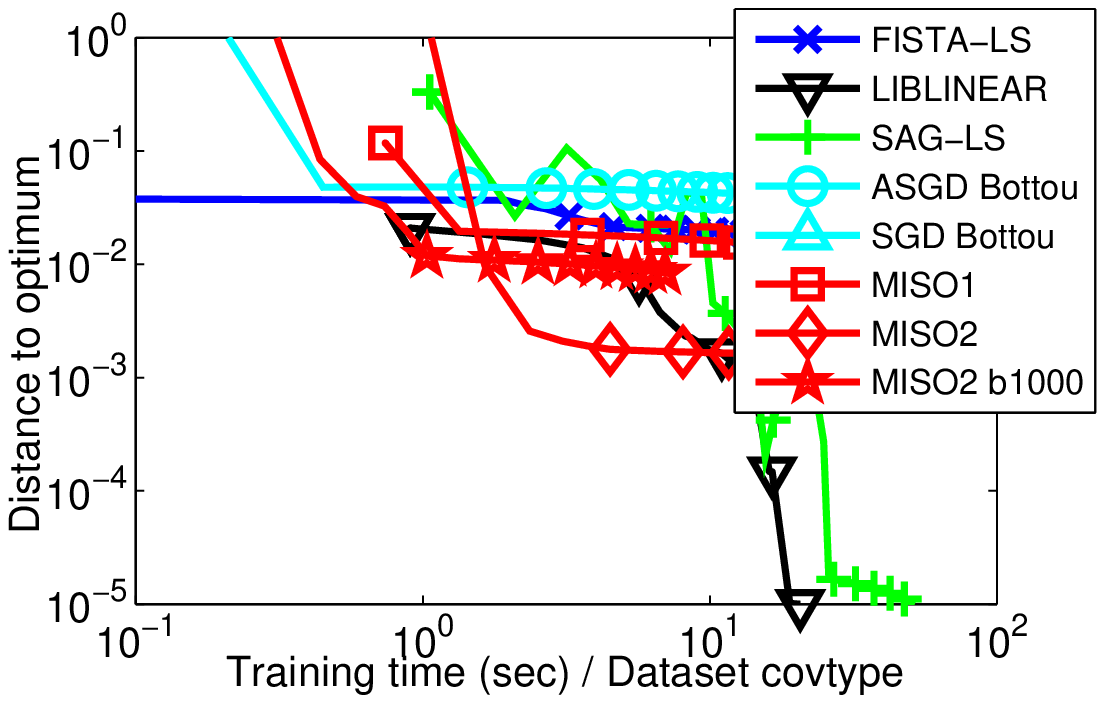}\\
\includegraphics[width=0.35\linewidth]{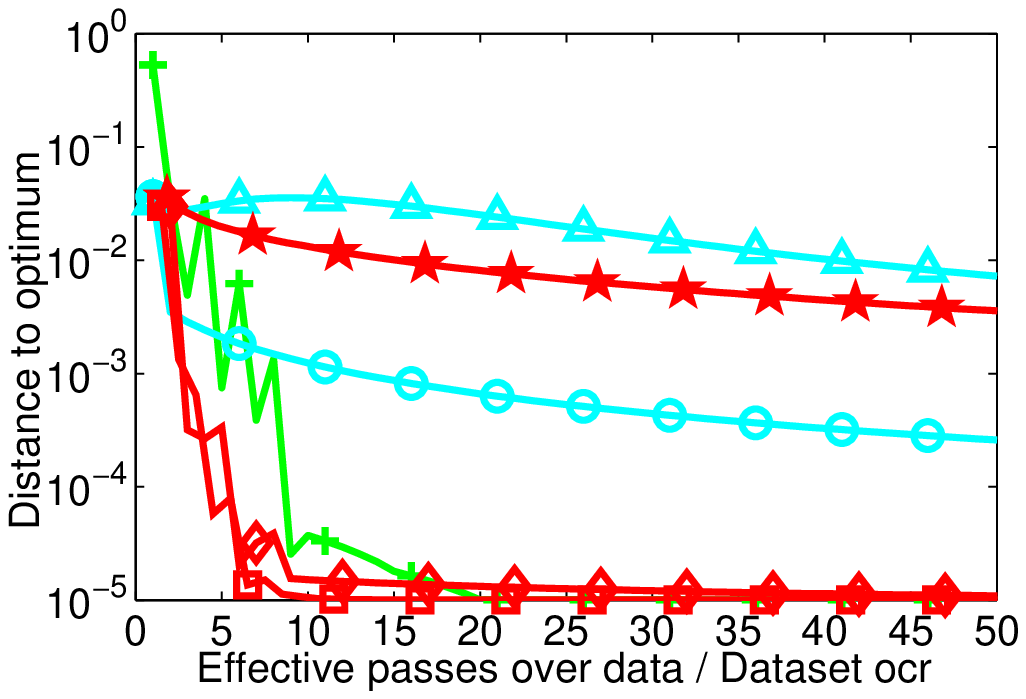}~
\includegraphics[width=0.375\linewidth]{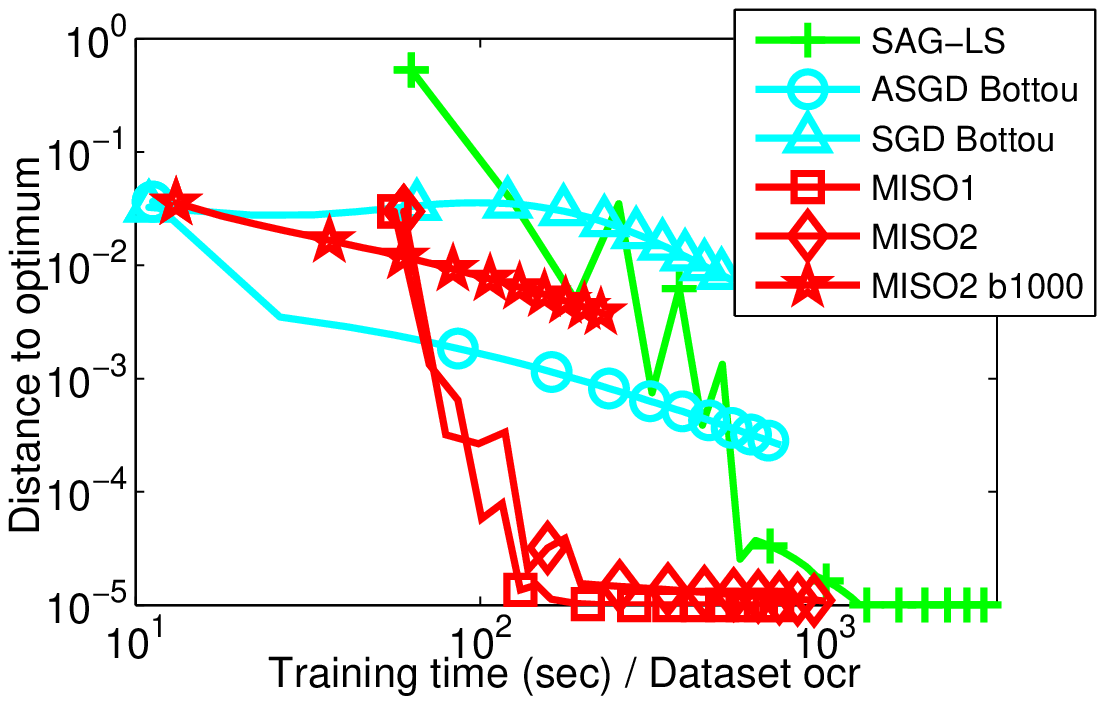}
\myvspace{0.2}
\caption{Benchmarks for $\ell_2$-logistic regression with $\lambda=10^{-7}$.}
\label{fig:exp6}
\end{figure}

\begin{figure}[hbtp]
\centering
\includegraphics[width=0.35\linewidth]{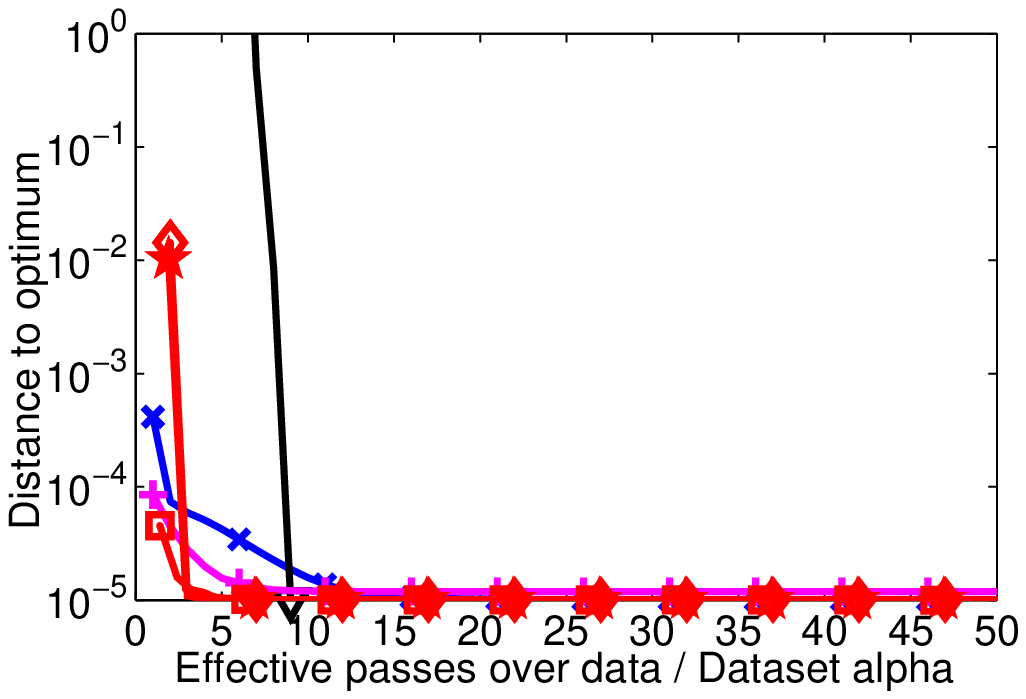}~
\includegraphics[width=0.375\linewidth]{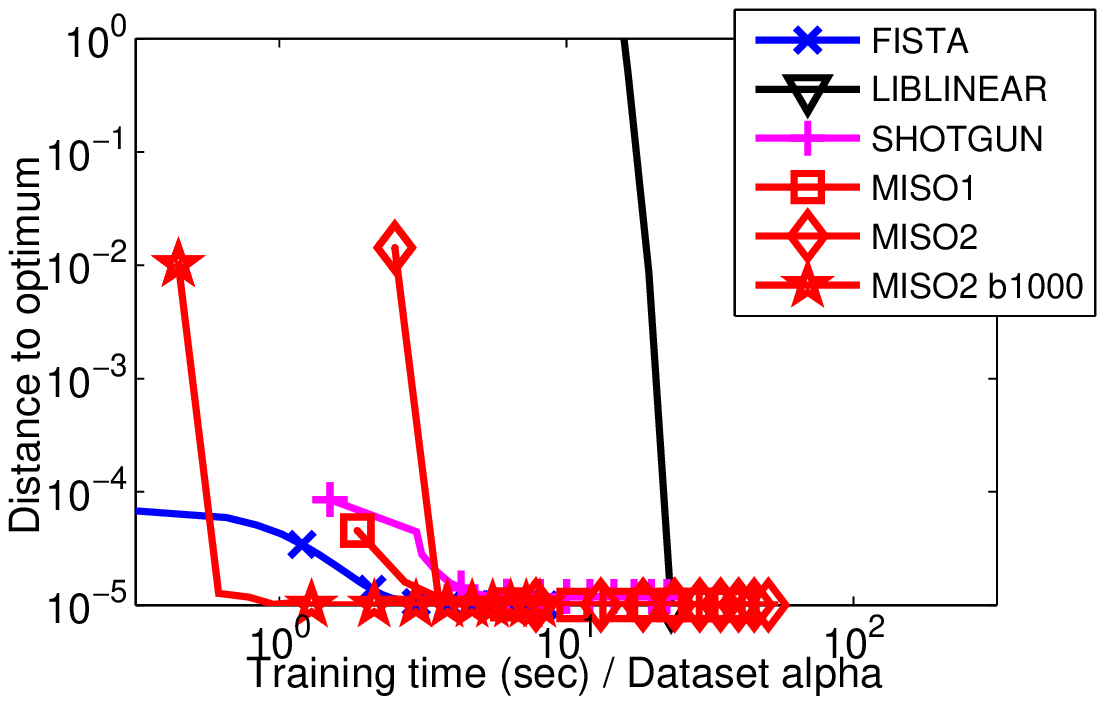}\\
\includegraphics[width=0.35\linewidth]{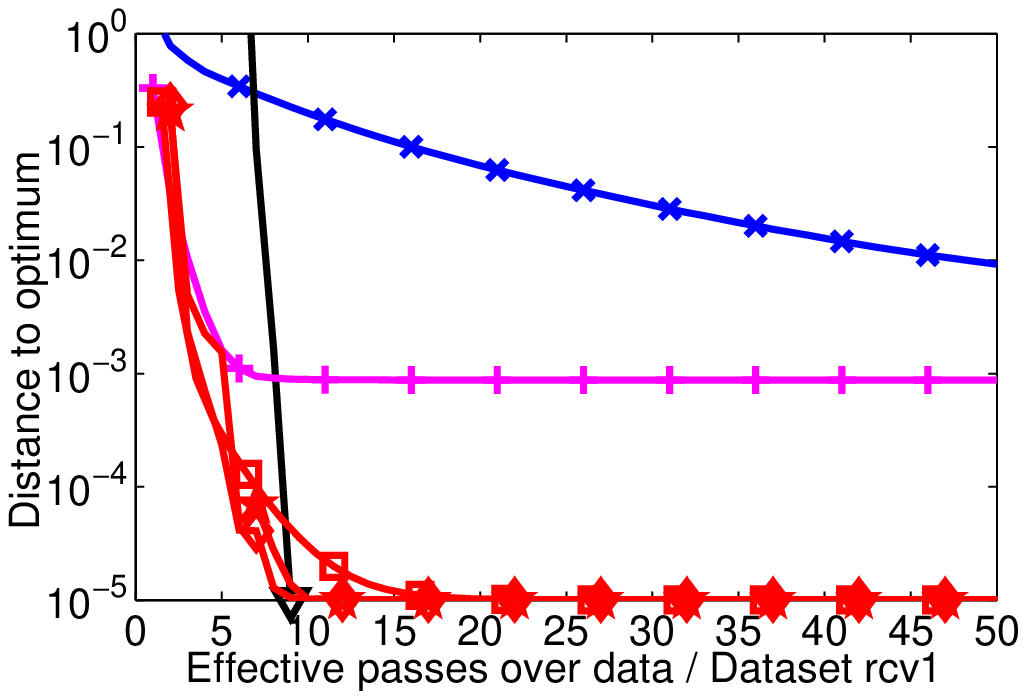}~
\includegraphics[width=0.375\linewidth]{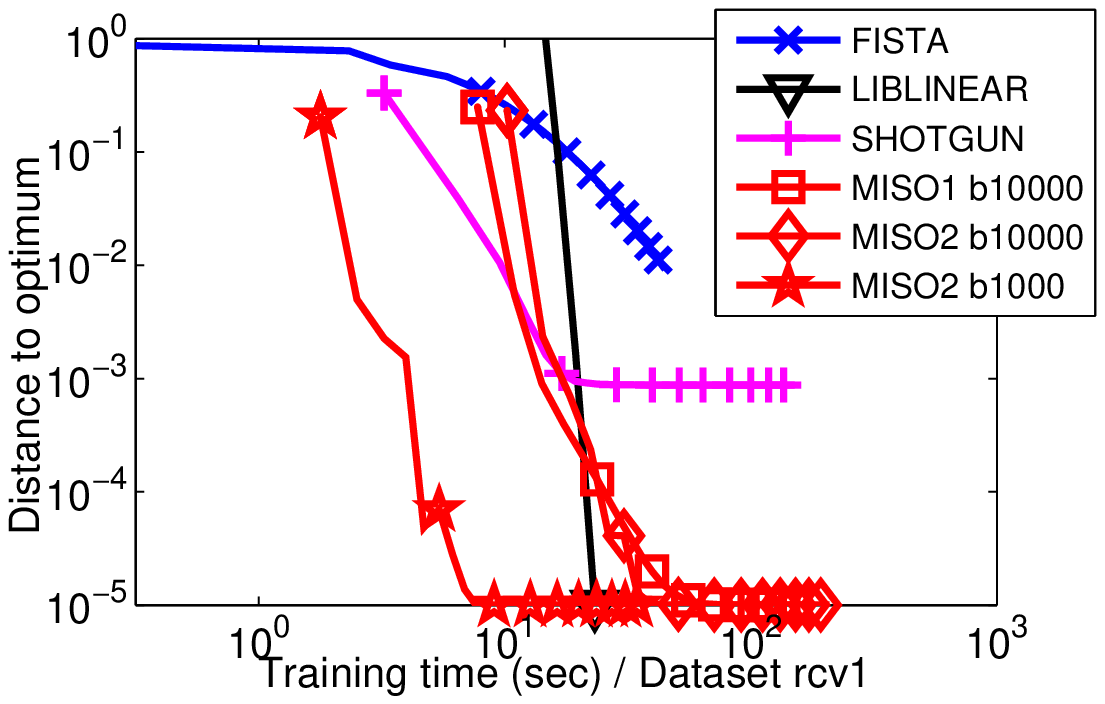}\\
\includegraphics[width=0.35\linewidth]{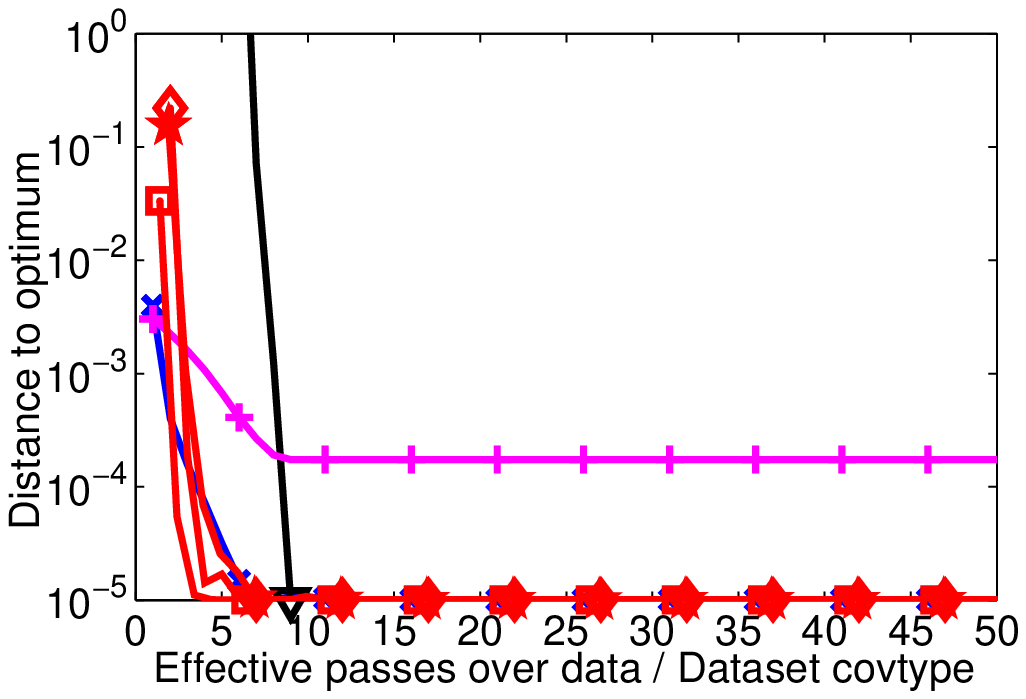}~
\includegraphics[width=0.375\linewidth]{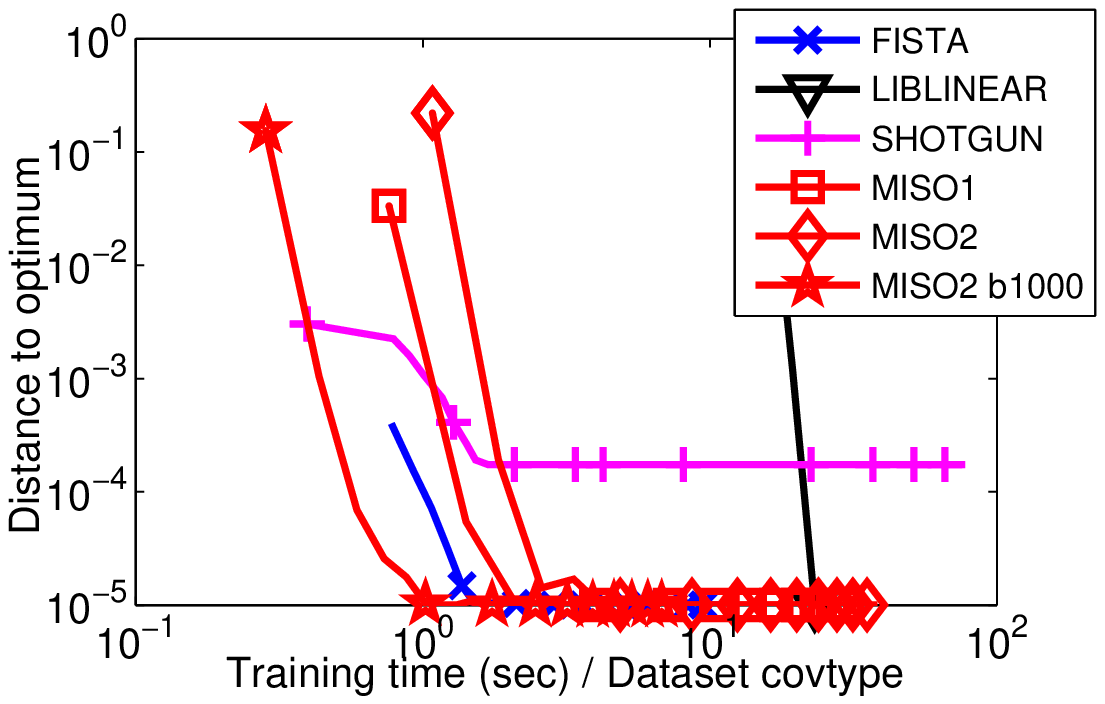}\\
\includegraphics[width=0.35\linewidth]{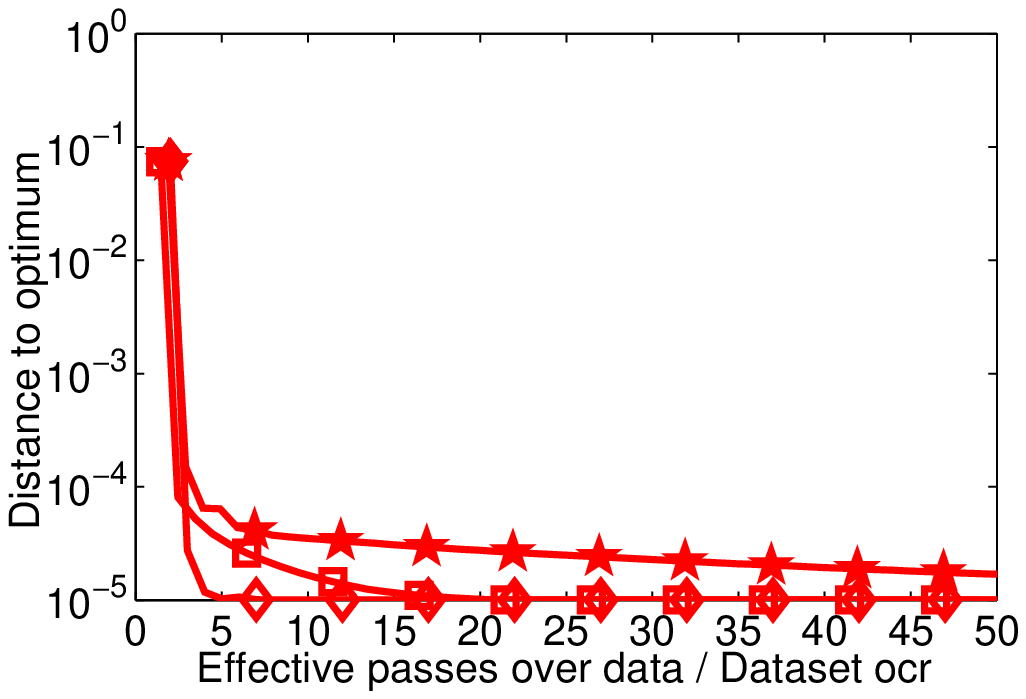}~
\includegraphics[width=0.375\linewidth]{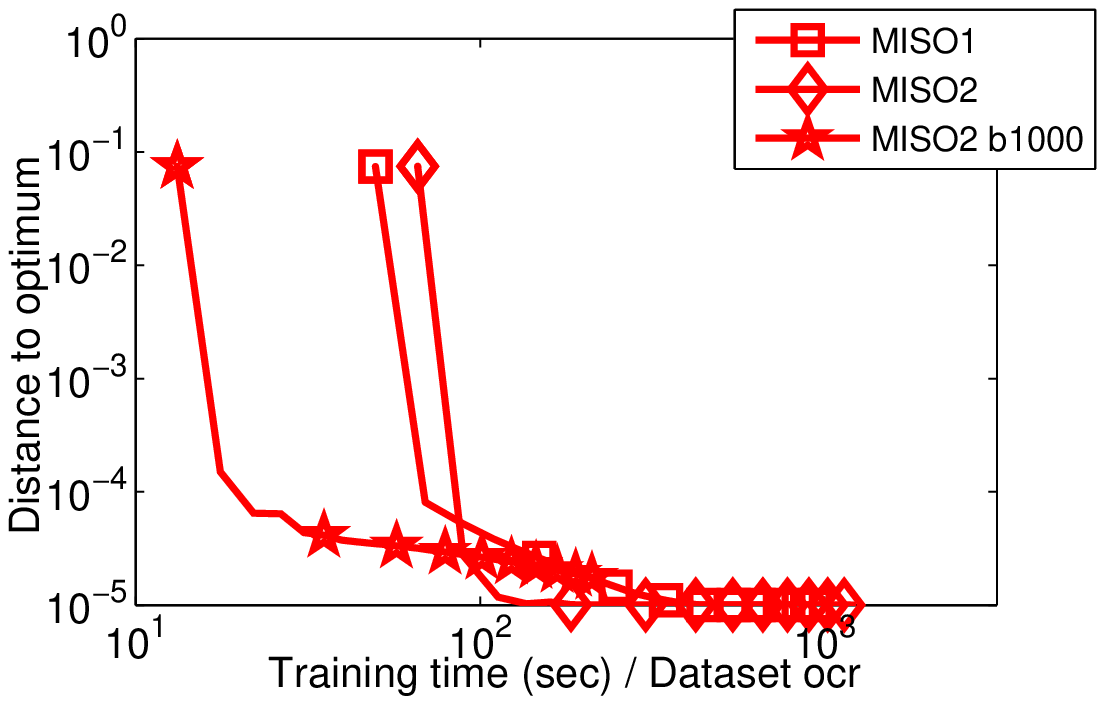}
\myvspace{0.2}
\caption{Benchmarks for $\ell_1$-logistic regression. The regularization parameter $\lambda$ was chosen to obtain a solution with about $3\%$ nonzero coefficients.}
\label{fig:exp3}
\end{figure}
\begin{figure}[hbtp]
\centering
\includegraphics[width=0.35\linewidth]{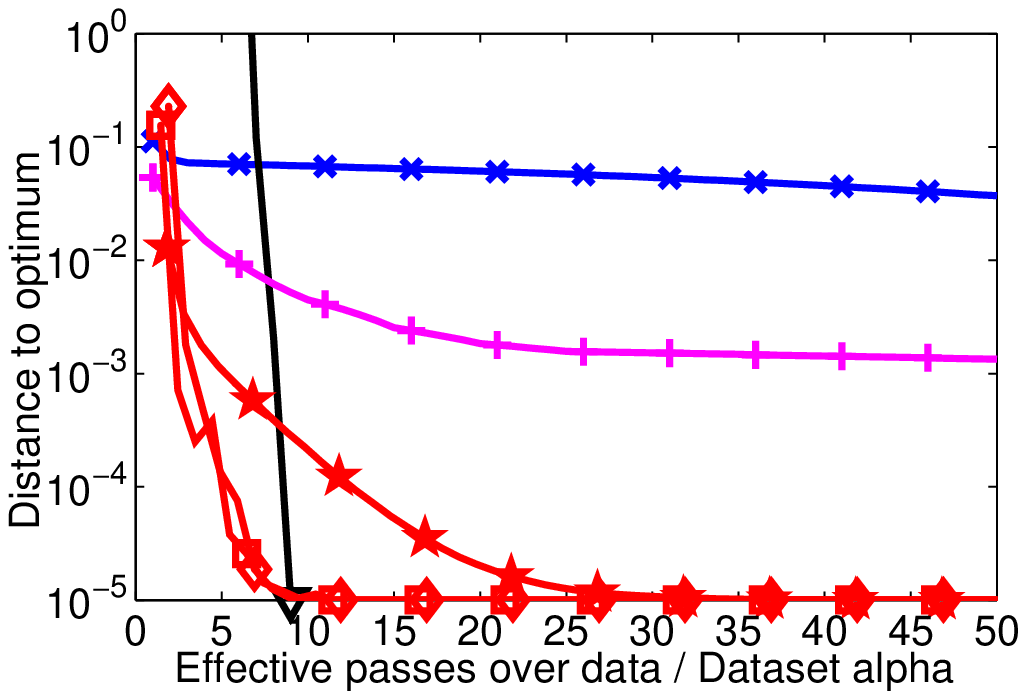}~
\includegraphics[width=0.375\linewidth]{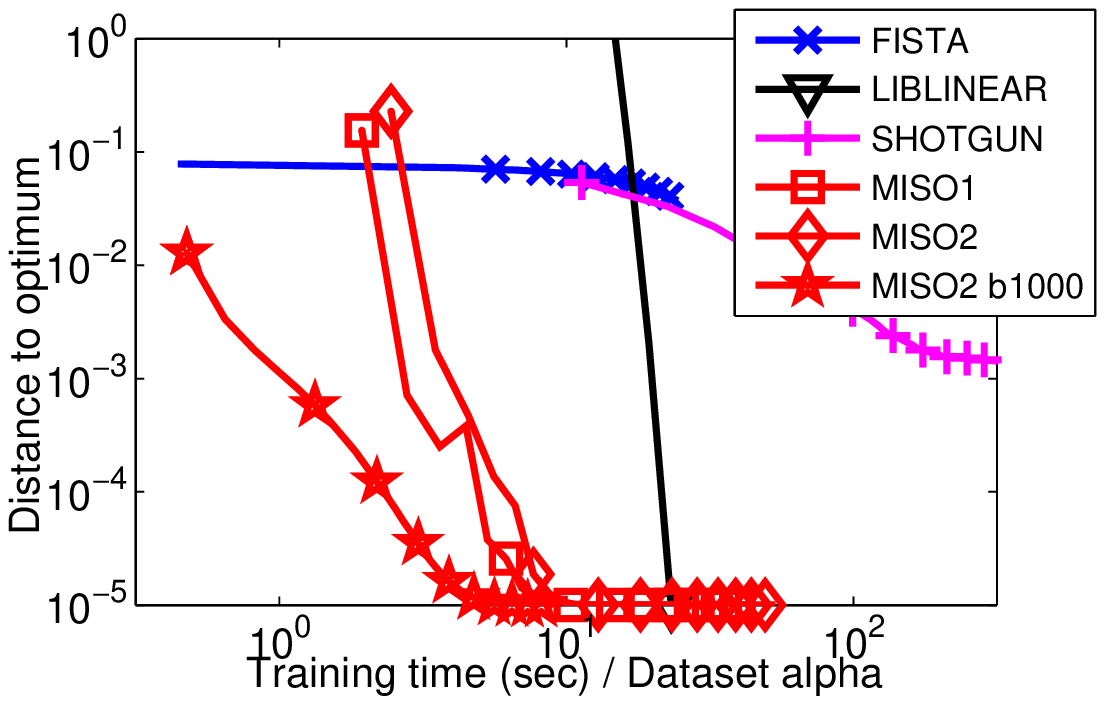}\\
\includegraphics[width=0.35\linewidth]{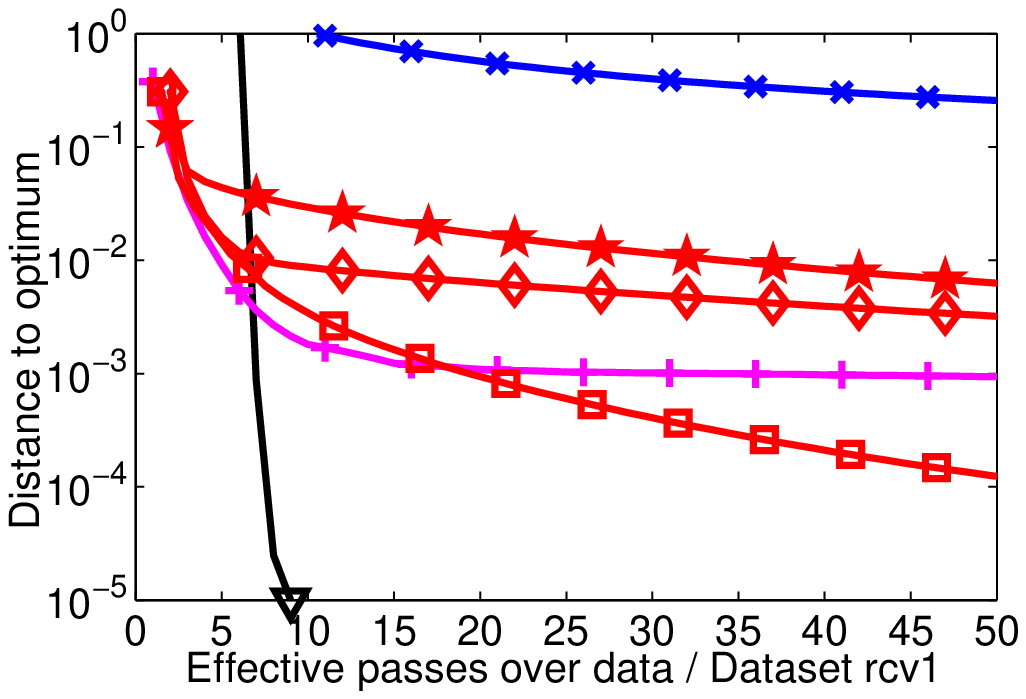}~
\includegraphics[width=0.375\linewidth]{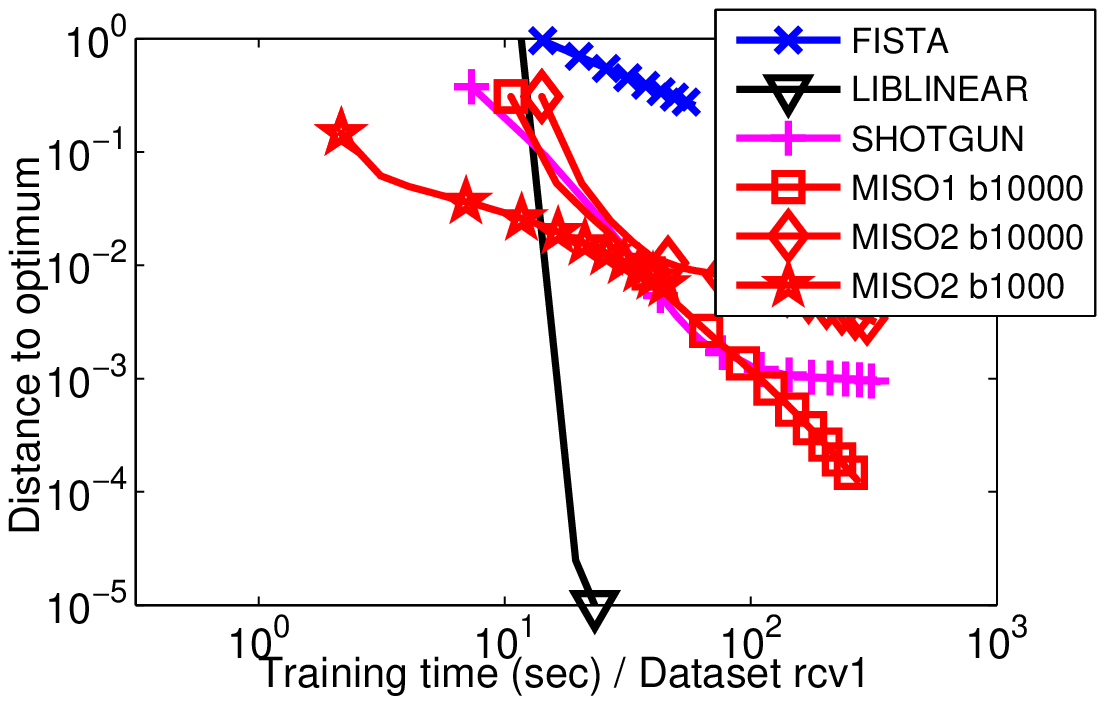}\\
\includegraphics[width=0.35\linewidth]{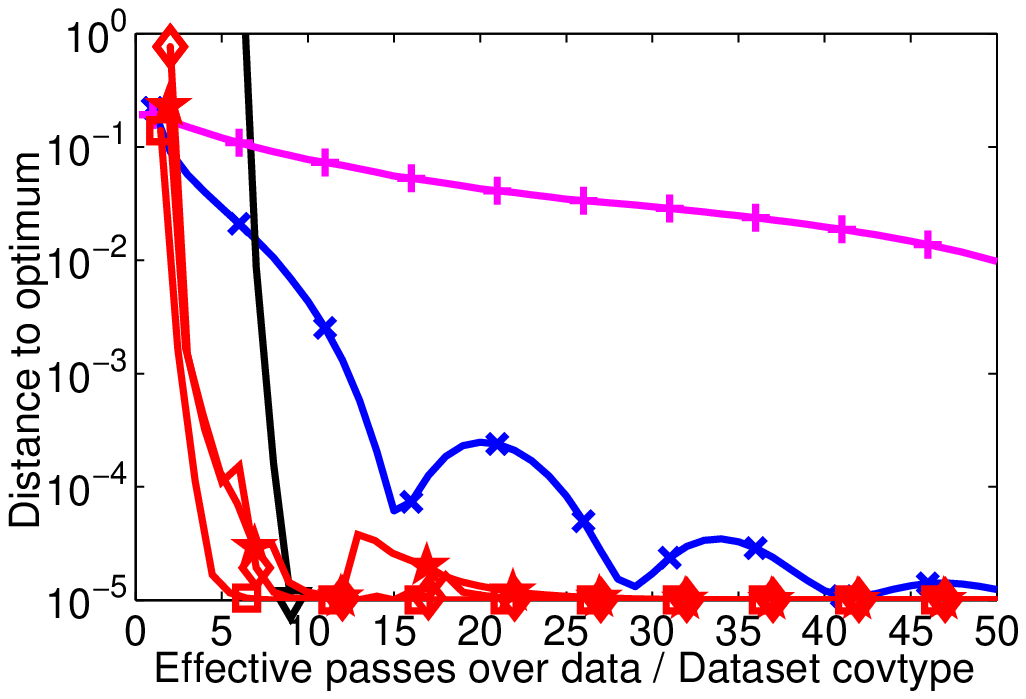}~
\includegraphics[width=0.375\linewidth]{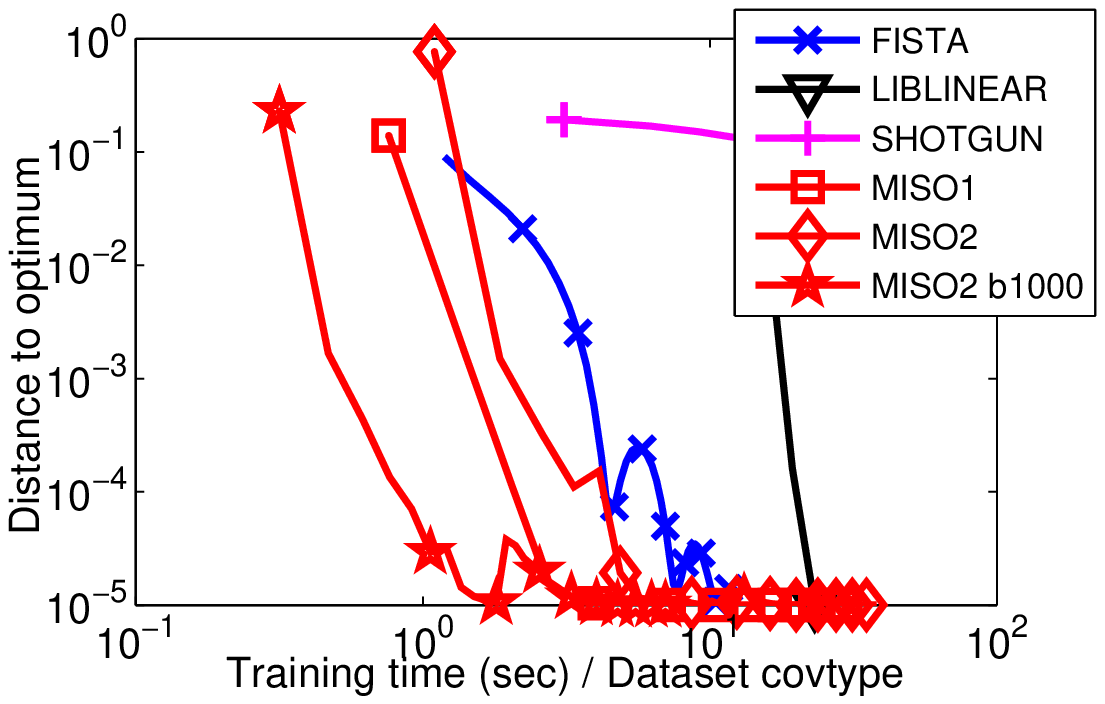}\\
\includegraphics[width=0.35\linewidth]{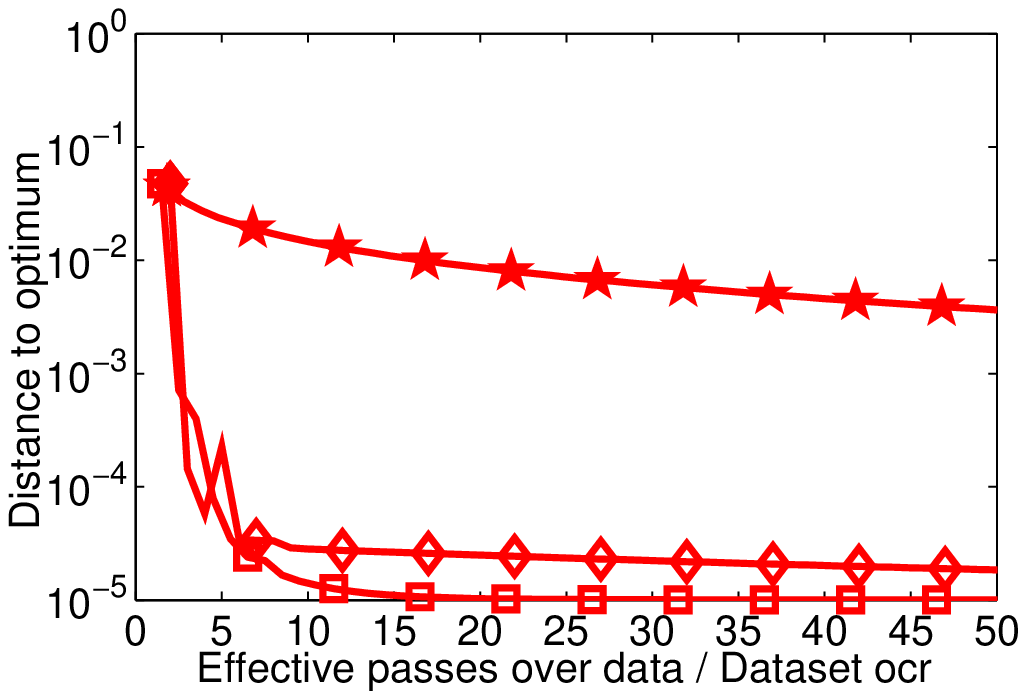}~
\includegraphics[width=0.375\linewidth]{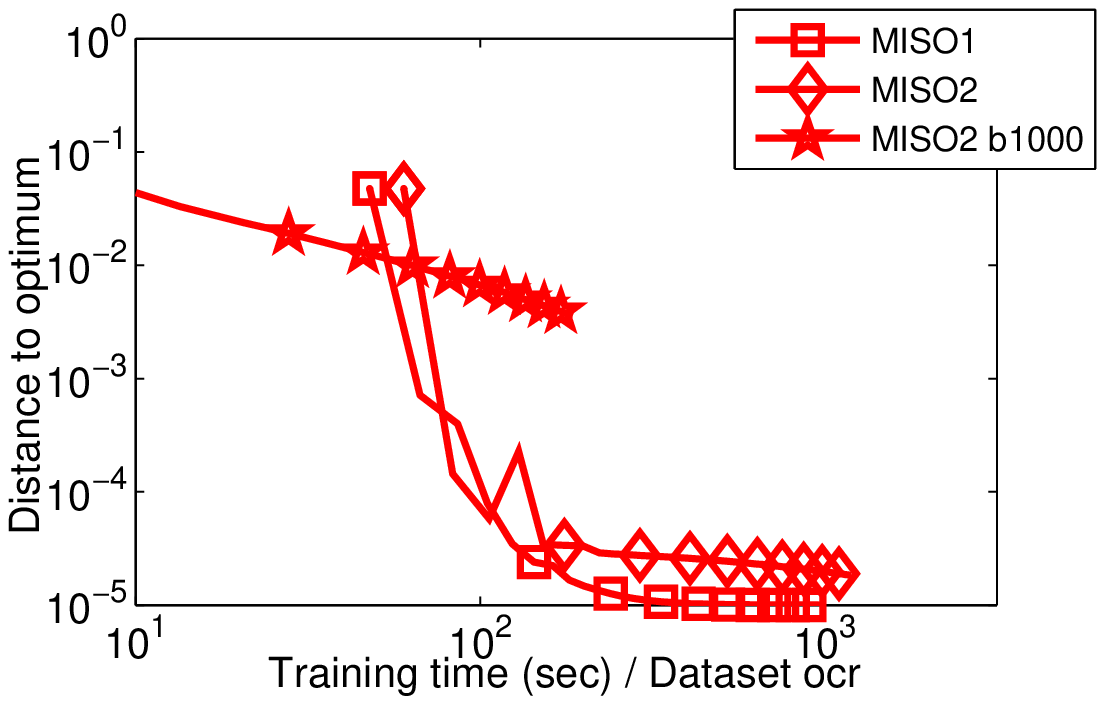}
\myvspace{0.2}
\caption{Benchmarks for $\ell_1$-logistic regression. The regularization parameter $\lambda$ was chosen to obtain a solution with about $50\%$ nonzero coefficients.}
\label{fig:exp4}
\end{figure}

\bigskip 

\bibliographysup{abbrev,main}
\bibliographystylesup{icml2013}

\end{document}